\documentclass[twoside,11pt]{article}
\usepackage{jmlr2e}
\usepackage{stackengine}
\usepackage{amsmath}
\usepackage{amssymb}
\usepackage{algorithm}
\usepackage{algorithmic}
\usepackage{xspace}
\usepackage{url}
\usepackage{subcaption}
\usepackage{multirow}
\usepackage{mathtools}
\usepackage{xparse}
\usepackage{tikz}
\usetikzlibrary{fit,matrix,backgrounds}
\pgfdeclarelayer{myback}
\pgfsetlayers{myback,background,main}
\tikzset{mycolor/.style = {line width=1bp,color=#1}}%
\tikzset{myfillcolor/.style = {draw,fill=#1}}%

\NewDocumentCommand{\highlight}{O{blue!40} m m}{%
\draw[mycolor=#1] (#2.north west)rectangle (#3.south east);
}
\NewDocumentCommand{\fhighlight}{O{blue!40} m m}{%
\draw[myfillcolor=#1] (#2.north west)rectangle (#3.south east);
}
\usepackage{footmisc}

\usepackage{hyperref}
\hypersetup{
	colorlinks,
	citecolor=pearDark,
	linkcolor=pearThree,
	urlcolor=pearDarker}
\title{Spectral bandits}

\usepackage{lastpage}
\jmlrheading{21}{2020}{1-\pageref{LastPage}}{10/16}{11/20}{16-529}{Tom\'a\v s Koc\'ak, R\'emi Munos, Branislav Kveton,  Shipra Agrawal, Michal Valko}

\ShortHeadings{Spectral bandits}{Koc\'ak,  Munos, Kveton, Agrawal, Valko}
\author{%
\name Tom\'a\v s Koc\'ak\thanks{\ also affiliated with Inria Lille -- Nord Europe, SequeL team}
\email tomas.kocak@gmail.com\\
\addr ENS de Lyon, 15 Parvis Ren\' e Descartes, 69342 Lyon, France
\AND 
\name R\'emi Munos\textcolor{pearThree}{$^*$}
\email munos@google.com\\
\addr DeepMind Paris, 14 Rue de Londres, 75009 Paris, France
\AND
\name Branislav Kveton \email bkveton@google.com\\
\addr Google Research, 1600 Amphitheatre Parkway, Mountain View, CA 94043, United States
\AND
\name Shipra Agrawal \email sa3305@columbia.edu\\
\addr Columbia University, West 120th Street, New York, NY, 10027 United States
\AND
\name Michal Valko\textcolor{pearThree}{$^*$}\email valkom@deepmind.com\\
\addr DeepMind Paris, 14 Rue de Londres, 75009 Paris, France
}

\definecolor{graphicbackground}{rgb}{0.96,0.96,0.8}
\definecolor{rouge1}{RGB}{226,0,38}  
\definecolor{orange1}{RGB}{243,154,38}  
\definecolor{jaune}{RGB}{254,205,27}  
\definecolor{blanc}{RGB}{255,255,255} 
\definecolor{rouge2}{RGB}{230,68,57}  
\definecolor{orange2}{RGB}{236,117,40}  
\definecolor{taupe}{RGB}{134,113,127} 
\definecolor{gris}{RGB}{91,94,111} 
\definecolor{bleu1}{RGB}{38,109,131} 
\definecolor{bleu2}{RGB}{28,50,114} 
\definecolor{vert1}{RGB}{133,146,66} 
\definecolor{vert3}{RGB}{20,200,66} 
\definecolor{vert2}{RGB}{157,193,7} 
\definecolor{darkyellow}{RGB}{233,165,0}  
\definecolor{lightgray}{rgb}{0.9,0.9,0.9}
\definecolor{darkgray}{rgb}{0.6,0.6,0.6}
\definecolor{babyblue}{rgb}{0.54, 0.81, 0.94}
\definecolor{citrine}{rgb}{0.89, 0.82, 0.04}
\definecolor{misogreen}{rgb}{0.25,0.6,0.0}
\definecolor{PalePurp}{rgb}{0.66,0.57,0.66}
\definecolor{todocolor}{rgb}{0.66,0.99,0.99}
\definecolor{pearOne}{HTML}{2C3E50}
\definecolor{pearTwo}{HTML}{A9CF54}
\definecolor{pearTwoT}{HTML}{C2895B}
\definecolor{pearThree}{HTML}{E74C3C}
\colorlet{titleTh}{pearOne}
\colorlet{bull}{pearTwo}
\definecolor{pearcomp}{HTML}{B97E29}
\definecolor{pearFour}{HTML}{588F27}
\definecolor{pearFith}{HTML}{ECF0F1}
\definecolor{pearDark}{HTML}{2980B9}
\definecolor{pearDarker}{HTML}{1D2DEC}

\DeclareMathOperator*{\argmax}{arg\,max}
\DeclareMathOperator*{\argmin}{arg\,min}

\newcommand{\R}{\mathbb{R}}

\newcommand{\NN}{{\mathbb N}}

\newcommand{\E}{\mathbb{E}}
\newcommand{\EE}[1]{\mathbb{E}\left[#1\right]}

\newcommand{\probability}{\mathbb{P}}
\renewcommand{\P}{\mathbb{P}}

\newcommand{\CommaBin}{\mathbin{\raisebox{0.5ex}{,}}}
\newcommand*{\eqdef}{\triangleq}
\newcommand{\transpose}{^\mathsf{\scriptscriptstyle T}}

\newcommand{\cD}{\mathcal{D}}

\newcommand{\cF}{\mathcal{F}}
\newcommand{\cG}{\mathcal{G}}
\newcommand{\cH}{\mathcal{H}}

\newcommand{\cL}{\mathcal{L}}

\newcommand{\cN}{\mathcal{N}}
\newcommand{\cO}{\mathcal{O}}
\newcommand{\tcO}{\widetilde{\cO}}

\newcommand{\cV}{\mathcal{V}}
\newcommand{\cW}{\mathcal{W}}

\newcommand{\bA}{{\bf A}}

\newcommand{\bD}{{\bf D}}

\newcommand{\bI}{{\bf I}}
\newcommand{\bM}{{\bf M}}

\newcommand{\bQ}{{\bf Q}}
\newcommand{\be}{{\bf e}}
\newcommand{\bff}{{\bf f}}

\newcommand{\bq}{{\bf q}}
\newcommand{\bu}{{\bf u}}
\newcommand{\bU}{{\bf U}}

\newcommand{\bV}{{\bf V}}
\newcommand{\bw}{{\bf w}}

\newcommand{\by}{{\bf y}}
\newcommand{\bx}{{\bf x}}
\newcommand{\bX}{{\bf X}}

\renewcommand{\epsilon}{\varepsilon}

\renewcommand{\tilde}{\widetilde}

\newcommand{\balpha}{{\boldsymbol \alpha}}

\newcommand{\bLambda}{{\boldsymbol \Lambda}}

\newcommand{\bxi}{{\boldsymbol \xi}}

\newcommand{\nothere}[1]{}

\usepackage{xspace}

\newcommand{\TS}{\normalfont \texttt{TS}\xspace}
\newcommand{\UCB}{\texttt{UCB}\xspace}

\newcommand{\ImprovedUCB}{\texttt{ImprovedUCB}\xspace}

\newcommand{\LinearTS}{\normalfont  \texttt{LinearTS}\xspace}

\newcommand{\SpectralEliminator}{\normalfont \texttt{\textcolor[rgb]{0.5,0.2,0}{SpectralEliminator}}\xspace}
\newcommand{\LinearEliminator}{\normalfont \texttt{\textcolor[rgb]{0.5,0.2,0}{LinearEliminator}}\xspace}
\newcommand{\LinUCB}{\normalfont \texttt{LinUCB}\xspace}

\newcommand{\OFUL}{\texttt{OFUL}\xspace}

\newcommand{\CLUB}{\texttt{CLUB}\xspace}

\newcommand{\SpectralUCB}{\normalfont \texttt{\textcolor[rgb]{0.5,0.2,0}{SpectralUCB}}\xspace}
\newcommand{\CheapUCB}{\texttt{CheapUCB}\xspace}
\newcommand{\SpectralTS}{\texttt{\textcolor[rgb]{0.5,0.2,0}{SpectralTS}}\xspace}
\newcommand{\SupLinRel}{\normalfont \texttt{SupLinRel}\xspace}
\newcommand{\SupLinUCB}{\normalfont \texttt{SupLinUCB}\xspace}

\newcommand{\NetBandits}{\texttt{NetBandits}\xspace}

\newcommand{\rounds}{{\textcolor[rgb]{0.25,0.0,0.6}{T}}}

\newcommand{\regret}{R_\rounds}

\newcommand{\dold}{d_{\scriptsize\mbox{old}}}

\newcommand{\node}[2]{(#1,#2)}

\begin{document}
\editor{Peter Auer}
\maketitle

\begin{abstract}%
Smooth functions on graphs have wide applications in manifold and
semi-supervised learning. In this work, we study a bandit problem where the
payoffs of arms are smooth on a graph. This framework is suitable for solving
online learning problems that involve graphs, such as content-based
recommendation. In this problem, each item we can recommend is a node of an undirected graph and its
expected rating is similar to the one of its neighbors. The goal is to recommend items that
have high expected ratings. We aim for the algorithms where the cumulative
regret with respect to the optimal policy would not scale poorly with the number
of nodes. In particular, we
introduce the notion of an \emph{effective dimension,} which is small in
real-world graphs, and propose three algorithms for solving our problem that scale
linearly and sublinearly in this dimension. Our experiments on content recommendation problem show 
that a good estimator of user preferences for thousands of items can be learned 
from just  tens of node evaluations.
\end{abstract}


\section{Introduction}
\label{sec:intro}

\emph{A smooth graph function} is a function on a graph that returns similar
values on neighboring nodes. This concept arises frequently in manifold and
semi-supervised learning \citep{zhu2008semi-supervised,valko2010online}, and reflects the fact
that the outcomes on the neighboring nodes tend to be similar. It is well-known
\citep{belkin2006manifold,belkin2004regularization} that a smooth graph function
can be expressed as a linear combination of the eigenvectors of the graph
Laplacian with smallest eigenvalues (see Figure~\ref{fig:flixster_eigenvectors} for an example). Therefore, the problem of learning such 
function can be cast as a regression problem on these eigenvectors. The present work 
brings this concept to bandits \citep{valko2016bandits}.
 In particular, we study a
bandit problem where the arms are the nodes of a graph and the expected payoff
of pulling an arm is a smooth function on this graph.

\begin{figure}[!b]
  \begin{center}
   \includegraphics[viewport = 130 400 490 538,clip,width=1\columnwidth]
 {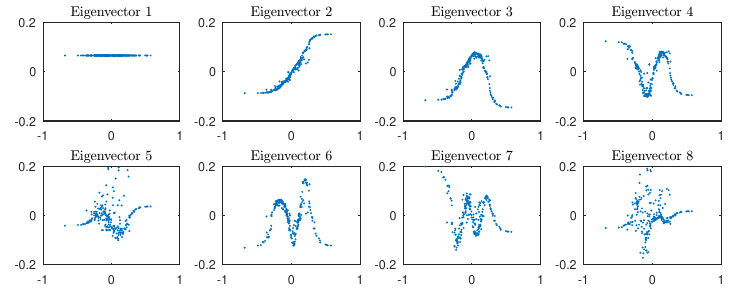}
 \caption{Eigenvectors from the Flixster dataset corresponding to the smallest
few eigenvalues projected onto the first principal component of the data; $x$-axis represents components of the eigenvector sorted according to the projection onto the first principal component of the data while $y$-axis represent the value of the corresponding component of the eigenvector. 
To produce the figure above, we performed the following steps. (1) \emph{Preprocessing:} We remove all the users that rated a small number of movies as well as the movies rated by only a few users. This leaves us with a $u\times m$ matrix $\bM$ where $u$ is the number of users and $m$ is the number of movies and entry $\bM_{i,\,j}$ of matrix $\bM$ is the rating of movie $j$ by user $i$, provided it exists. Note that matrix $\bM$ might be missing some of the entries. (2)
\emph{Filling in the missing entries:} For this step, we use low-rank matrix factorization \citep{keshavan2009matrix} to obtain $u\times r$ matrix $\bU$ and $m\times r$ matrix $\bV$, for some given rank $r$, such that $\bM\approx \bU\bV\transpose.$ (3)
\emph{Constructing a similarity graph:} We construct the graph by creating an edge between movies $i$ and $j$ if the movie $j$ is among 5 nearest neighbors of the movie $i$ in the latent space of movies $\bV.$ (4) \emph{Visualization:} 
Using  the computed matrix $\bV$, the matrix capturing the latent space of movies, we can find the direction of the highest variance of the data using PCA on~$\bV$. This gives us a way to visualize eigenvectors by projecting them on the first principal component. The above visualization shows that the eigenvectors corresponding to smaller eigenvalues tend to be smoother---the values corresponding to actions with their projections close to each other are similar. On the other hand, the eigenvectors corresponding to larger eigenvalues are more chaotic as values for nearby items can vary a lot. This gives a small insight into why a function created as a linear combination of the first few eigenvectors is a smooth reward function. We are more precise about the definition of smoothness and the connection of smooth functions and eigenvectors later in Section~\ref{sec:spectralbandits}.
}
  \label{fig:flixster_eigenvectors}
  \end{center}
\end{figure}

We are motivated by a range of practical problems that involve graphs. One
application is \textit{targeted advertisement} in social networks. Here, the
graph is a social network and our goal is to discover a part
of the network that is interested in a given product. Interests of people in a
social network tend to change smoothly \citep{mcpherson2001birds}, because
friends tend to have similar preferences. Therefore, we take advantage of
this structure and formulate this problem as learning a smooth preference
function on a graph.

Another application of our work are \textit{recommender systems}
\citep{jannach2010recommender}. In content-based recommendation
\citep{chau2011apolo}, 
the user is recommended items that are
similar to the
items that the user rated highly in the past. The assumption is that users
prefer similar items similarly. The similarity of the items is measured for
instance by a nearest-neighbor graph \citep{billsus2000learning}, where
each item is a node and its neighbors are the most similar items.

We consider the following learning setting. The graph is known in
advance and its edges represent the similarity of the nodes. At
round $t$, we choose a node and then observe its payoff. In targeted
advertisement, this may correspond to showing an ad and then observing whether
the person has clicked on it. In content-based recommendation, this may
correspond to recommending an item and then observing the assigned rating. 
Based
on the payoff, we update our model of the world and then the game proceeds into
round $t + 1$. 
In both  applications described above, the learner (advertiser)
has rarely the budget (time horizon $T$) to try all the options even once. 
Furthermore, imagine that the learner is a movie recommender system
and would ask the user to rate all the movies before it starts producing relevant 
recommendations. Such a recommender system would be of little value. 
Yet, many bandit algorithms start with pulling each arm once. 
This is something that we cannot afford  and therefore, contrary to standard  
 bandits, we consider the case $T \ll N$, where the number of nodes $N$ is huge. While we are mostly interested in
the regime when $t < N$, our results are beneficial also for $t > N$. 
This regime is especially challenging since traditional multi-arm bandit algorithms 
need to try every arm.

If the smooth graph function is expressed as a linear combination of $k$
eigenvectors of the graph Laplacian and $k$ is small and known, our learning
problem can be solved using ordinary linear bandits
\citep{auer2002using,dani2008stochastic,li2010contextual,agrawal2013thomson,abeille2017linear}. In practice,~$k$ is problem specific and
unknown. Moreover, the number of features $k$ may approach the number of nodes
$N$. Therefore, proper regularization is necessary, so that the regret of the
learning algorithm does not scale with $N$. We are interested in the setting
where the regret is independent of $N$ and this makes the problem we study 
non-trivial.

Early short versions of our work appeared at \emph{International Conference on Machine Learning} \citep{valko2014spectral} and \emph{AAAI Conference on Artificial Intelligence} \citep{kocak2014spectral}.
Compared to those, we give a new and improved definition of the effective dimension 
that is smaller than the old one, provide a
matching lower bound, improve the regret bounds for two of our algorithms,
and report a comprehensive empirical evaluation on 
artificial datasets as well as on the Movielens and Flixster datasets.



\section{Setting}

\label{sec:setting}
In this section, we formally define the spectral bandit setting.
Let $\cG$ be the given graph with the set of nodes $\cV$ and
denote $N \eqdef |\cV|$ the number of nodes.
Let $\cW$ be the symmetric $N \times N$ matrix of similarities $w_{ij}$ (edge
weights) and $\cD$ be the $N \times N$ diagonal matrix with entries $d_{ii}
\eqdef\sum_j w_{ij}$ (node degrees).
The graph Laplacian of $\cG$ is defined as $\cL \eqdef \cD - \cW$.
Let $\{\lambda^\cL_k, \bq_k\}_{k=1}^N$ be the eigenvalues and eigenvectors
of $\cL$ ordered such that $ 0 = \lambda^\cL_1 \le \lambda^\cL_2 \le \dots
\le \lambda^\cL_N$.
Equivalently, let $\cL \eqdef \bQ \bLambda_\cL \bQ \transpose$ be an eigendecomposition
of $\cL$, where $\bQ$ is an $N \times N$ \textit{orthogonal} matrix with
eigenvectors in columns.

\noindent The eigenvectors of the graph Laplacian form a basis. Therefore, we can represent the reward function as a linear combination of the eigenvectors. For any set of weights $\balpha$, let
$f_\balpha: \cV \to \R$ be the function defined on nodes,
linear in the basis of the eigenvectors of $\cL,$
\[
 f_\balpha(v) \eqdef \sum_{k=1}^{N} \alpha_k (\bq_{k})_v = \bx_v\transpose  \balpha,
\]
where $\bx_v$ is the $v$-th row of $\bQ$, i.e., $(\bx_{v})_i = (\bq_{i})_v$.
If the weight coefficients of the true $\balpha$
are such that the large coefficients correspond
to the eigenvectors with the small eigenvalues
and vice versa, then $f_{\balpha}$ would be a smooth function on $\cG$
\citep{belkin2006manifold}. 
For more details, see Section \ref{sec:smoothgraphfunctions}.
Figure~\ref{fig:flixster_eigenvectors}
displays the first few eigenvectors of the Laplacian
constructed from the data that we use in our experiments.
In the extreme case, the true $\balpha$ may be of the form
$[\alpha_1,\,\alpha_2,\, \dots,\, \alpha_k,\, 0,\, 0,\, 0]\transpose_N$
for some $k\ll N$. Had we known $k$ in such case, the
known linear bandit algorithms would work with the performance
scaling with $k$ instead of $D=N$. Unfortunately, first, we do not know $k$
and second, we do not want to assume such an extreme case (i.e.,~$\alpha_i =
0$ for $i>k$). Therefore, we opt for the more plausible assumption
that the coefficients with the high indexes are small. Consequently, we deliver
algorithms with the performance that scale with the
smoothness with respect to the graph.


%

We now define the learning setting. In each round $t\le T$,
the recommender chooses a node $a_t$ and obtains
a noisy reward such that
\[
  r_t \eqdef \bx_{a_t}\transpose \balpha  + \varepsilon_t,
\]
where the noise $\varepsilon_t$ is assumed to be zero mean and conditionally independent $R$-sub-Gaussian random variable for any $t$, that is, $\EE{\exp(s\varepsilon_t)}\le \exp(R^2s^2/2)$, for all $s\in \R$ and $\EE{\varepsilon_t} =0$.
In our setting, we have $\bx_v \in \R^D$ and $\|\bx_v\|_2\leq 1$
for all $\bx_v$.  The goal of the recommender is to minimize the cumulative
regret with respect to the strategy that always picks the best node
w.r.t.\,$\balpha$.
Let $a_t$ be the node picked (referred to as 
\textit{pulling an arm}) by an algorithm at round $t$.
The cumulative (pseudo-) regret of an algorithm is defined as
\[
R_T  \eqdef T \max_v  f_{\balpha}(v) -  \sum_{t=1}^T f_{\balpha}(a_t).
\]
We call this bandit setting \textit{spectral}
since it is built on the spectral properties of a graph.
Compared to the linear and multi-arm bandits,
the number of arms
 $K$ is equal to the number of nodes $N$ and to the
 dimension of the basis $D$  (the eigenvectors are of dimension~$N$).
However, a regret that scales with $N$ or $D$
that can be obtained using those approaches is not acceptable because the number
of nodes can be large.
While we are mostly interested in the setting with $K=N$,
our algorithms and analyses are valid for any \emph{finite}~$K$.



\section{Related work}
\label{sec:related}

The most related settings to our work are that of the linear
and contextual linear bandits. 
For these settings, \citet{auer2002using} proposed 
\SupLinRel and showed that it obtains $\sqrt{DT}$
regret which matches the lower bound by \citet{dani2008stochastic}.
However, the first practical and empirically successful
algorithm was \LinUCB \citep{li2010contextual}.
Later, \citet{chu2011contextual} analyzed 
\SupLinUCB, which is a \LinUCB equivalent of \SupLinRel, to show that
it also obtains $\sqrt{DT}$ regret.
\citet{abbasi2011improved}
proposed \OFUL 
for linear bandits
which obtains $D\sqrt{T}$ regret.
Using their analysis, it is possible
to show that \LinUCB obtains  $D\sqrt{T}$
regret as well (Remark~\ref{rem:linucb}). Whether \LinUCB matches the
$\sqrt{DT}$ lower bound
for this setting is still an open problem.

Apart from the above optimistic approaches,  an older approach to the problem is Thompson sampling (\TS, \citealp{thompson1933likelihood}). It solves the exploration-exploitation
dilemma by a simple and intuitive rule: when choosing the next action to play, choose it according to the probability that it is the best one; that is the one that maximizes the expected payoff.  \citet{chapelle2011empirical} showed its practical relevance to the computational advertising. 
This motivated the researchers to explain the success of \TS  \citep{agrawal2011analysis,kaufmann2012thompson,may2012optimistic,agrawal2013further,abeille2017linear}.
The most relevant results for our work are by~\citet{agrawal2013thomson}, who bring a new martingale technique, enabling
us to analyze cases where the payoffs of the actions are linear in some basis.

\citet{abernethy2008competing} and \citet{bubeck2012towards} studied a more
difficult \textit{adversarial} setting of linear bandits where the reward function is time-dependent. It is an open problem if this
approaches would work in our setting and have
an upper bound on the regret that scales better than with~$D$.

\citet{kleinberg2008multi}, \citet{slivkins2009contextual}, and \citet{bubeck2011x} use
similarity
information
between the context of arms, assuming a Lipschitz or more general
properties.
While such settings are indeed more general, the regret bounds
scale worse with the relevant dimensions.
\citet{srinivas2009gaussian} and \citet{valko2013finite} also perform maximization
over the smooth functions that are either sampled from a Gaussian process
prior or have a small RKHS norm.  Their setting is also more general
than ours since it already generalizes linear bandits. However, 
their regret bound in the linear case also scales with $D$.
Moreover, the regret of these algorithms also
depends on a quantity for which data-independent
bounds exist only for some kernels, while
our effective dimension is always computable
given the graph.

Another bandit graph setting called the \textit{gang of bandits}
was studied by \citet{cesa-bianchi2013gang},
where each node is a linear bandit with its own
weight vector. These weight vectors are assumed to be smooth on the graph.
\citet{gentile2014online} take a different approach to similarities in social networks 
by assuming that the actions are clustered into several unknown clusters and the actions within 
one cluster have the same expected reward. This approach can be applied also to
the setting presented in our paper. 
The biggest advantage of the \CLUB algorithm by \citet{gentile2014online} is that it constructs graph
iteratively, starting with complete graph and removing edges which are not likely to be presented
in the underlying clustering. Therefore, the algorithm does not need to know the similarity graph unlike in our setting.
However, theoretical improvement of \CLUB compared to the basic bandit algorithm comes from the small number 
of clusters. Therefore, if the number of clusters is close to the number of actions the algorithm does not bring any improvement
while the algorithms in our setting still can leverage the similarity structure.
\citet{li2015online} and \citet{gentile2017context} later extended the approach to \emph{double-clustering} where both the users and the items
are assumed to appear in clusters (with the underlying clustering unknown to the learner)
and \citet{korda2016distributed} considers a distributed extension. 
Yet another assumption of a special graph reward structure is exploited by unimodal bandits~\citep{yu2011unimodal,combes2014unimodal}.
One of the settings considered by \citet{yu2011unimodal} is a graph bandit setting where every path in the graph has unimodal rewards and
therefore also imposes a specific  kind of smoothness with respect to the graph topology.
In networked bandits~\citep{fang2014networked}, the learner picks a node, but besides receiving the reward from that node, its reward is the sum of the rewards
of the picked node and its neighborhood. The algorithm of \citet{fang2014networked}, \NetBandits, can also deal with changing topology, however, this has to be always 
revealed to the learner before it makes its decision.
 
 Furthermore, bandits with side observations treat a different graph bandit setting where the learner obtains
not only the reward from the selected action but also the rewards from
the neighbors of the selected action. This setting was 
studied in both the stochastic case \citep{caron2012leveraging,buccapatnam2014stochastic} and the adversarial
one \citep{mannor2011from,alon2013from,kocak2014efficient,Alon2014nonstochastic,alon2015online,kocak2016online,kocak2016onlinea}.
For a comprehensive discussion, we refer to survey on graph bandits 
 \citep{valko2016bandits}.

 \paragraph{Spectral bandits with different objectives}
In the follow-up work on spectral bandits, there have been algorithms
optimizing other objective function than the cumulative regret.
First, in some sensor networks, sensing a node (pulling an arm)
has an associated cost \citep{narang2013signal}. In a 
particular, \emph{cheap bandit} setting \citep{hanawal2015cheap}, 
it is cheaper to  get an average of rewards of a set of nodes than a specific reward of a single one. 
More precisely, the learner pays the cost for the action
which depends on the spectral properties of the graph while relying on the property that getting the average reward of many nodes is less costly than getting a reward of a single node. 
For this setting, \citet{hanawal2015cheap}  proposed 
\CheapUCB  that reduces the cost of sampling by 1/4
as compared to \SpectralUCB, while maintaining 
$\tilde\cO(d\sqrt{T})$ cumulative regret.
Next, \citet{gu2014online} study the online classification setting 
on graphs with bandit feedback, very similar to spectral bandits; after predicting the class the oracle returns a single bit indicating whether the
prediction is correct or not.
The analysis of their algorithm delivers essentially the same bound
on the regret, however, they need to know the number of relevant eigenvectors~$d$.
Moreover, \citet{ma2015active} consider several variants of \emph{$\Sigma$-optimality}
that favors specific exploration when selecting the nodes, for example, the learner is not allowed to play
one arm multiple times.  The authors were able to show a regret bound which scales with 
the effective dimension that we defined in our prior work \citep{valko2014spectral}.


\section{Spectral bandits}
\label{sec:spectralbandits}
In this section, we show how to leverage the smoothness of the rewards on a given graph.
In our setting, the features of the arms (contexts) form a basis and therefore are \textit{orthogonal} to each other.
Thinking that the reward observed
 for an arm does not provide any information for other arms
would not be correct because of the assumption that under another
basis, the unknown parameter has a low norm.
This provides additional information across the arms through the estimation
of the parameter $\balpha$. 

\subsection{Smooth graph functions}\label{sec:smoothgraphfunctions}

There are several possible ways to define the \textit{smoothness} of the function $f$ 
with respect to the undirected graph $G$. We are using the one which is standard in the spectral 
clustering \citep{luxburg2007tutorial} and semi-supervised learning \citep{belkin2006manifold}, defined as
\[
S_G(f) \eqdef \frac{1}{2}\sum_{i,\,j \in [N]} w_{i,j}\left(f(i)-f(j)\right)^2\!.
\]
Therefore, whenever the function values of the nodes connected by an edge with large weight 
are close, the smoothness of the function with respect to the graph is small and the 
function is smoother with respect to the graph. This definition has several useful properties. We 
are mainly interested in the following one,
\[
S_G(f) = \bff\transpose\cL\bff = \bff\transpose\bQ\bLambda\bQ\transpose\bff = \balpha\transpose\bLambda\balpha = \|\balpha\|^2_{\bLambda} = \sum_{i = 1}^N \lambda_i \alpha_i^2,
\]
where $\bff = (f(1),\,\dots,\,f(N))\transpose$ is the vector of the function values, 
$\bQ\bLambda\bQ\transpose$ is an eigendecomposition of the graph Laplacian $\cL$, and 
$\balpha = \bQ\transpose\bff$ is the representation of the vector $\bff$ in the eigenbasis. 
The assumption on the smoothness of the reward function with respect to the 
underlying graph is reflected by the small value of $S_\cG(f)$ and therefore, 
the components of~$\balpha$ corresponding to the large eigenvalues 
should be small as well.


As a result, we can think of our setting as an $N$-arm bandit
problem where $N$ is possibly larger than the time horizon $T$  and the
mean reward $f(k)$ for each arm $k$ satisfies the property that under a
change of coordinates, the vector $\bff$ of mean rewards has small components, i.e.,~there exists a known
orthogonal matrix $\bU$ such that $\balpha = \bU \bff$ has a low norm.
As a consequence, we can estimate $\balpha$ using penalization corresponding to the large eigenvalues
and to recover $\bff$.
Given a vector of weights
$\balpha$,
we define its $\bLambda$-norm as
\begin{align}
\label{eq:spectalpenalty}
  \|\balpha\|_{\bLambda} \eqdef \sqrt{\sum_{i=1}^{N}  \lambda_i \alpha_i^2} =
\sqrt{\balpha\transpose \bLambda \balpha}.
\end{align}
This norm is closely related to the smoothness of the function and we use it later in our algorithms by regularization which enforces small $\bLambda$-norm of $\balpha$.

\subsection{Effective dimension}
\label{ssec:introeffd}

In order to present and analyze our algorithms,
we use a notion of \textit{effective dimension} denoted by (lower case) $d$. While we introduced a slightly different version of the effective dimension for spectral bandits previously \citep{valko2014spectral}, we now present an improved definition. This new definition of effective 
dimension enables us to prove tighter regret bounds for our algorithms. In the rest 
of the paper, we refer to the old definition of the effective dimension, introduced by \cite{valko2014spectral},
as $\dold$. We keep using capital $D$ to denote the ambient dimension (the number of features). Intuitively, the effective dimension is a proxy for the number of relevant dimensions. We first provide a formal definition and then 
discuss its properties, including $d<\dold\ll D$.

In general, we assume there exists a diagonal matrix $\bLambda$ with the
entries $0<\lambda=\lambda_1\leq\lambda_2\leq \dots\leq \lambda_N$ and
a set of $N$ vectors $\bx_1,\dots, \bx_N\in\R^N$ such that $\|\bx_i\|_2\leq 1$
for all $i$. Moreover, since $\bQ$ is an orthonormal matrix, $\|\bx_i\|_2 = 1$.
Finally, since the first eigenvalue of a graph Laplacian is always zero,
 $\lambda^\cL_1 = 0$, we use $\bLambda = \bLambda_{\cL} + \lambda \bI$,
in order to have $\lambda_1 = \lambda > 0$.

\begin{definition}\label{def:effectived}
The \textbf{effective dimension} $d$ is defined as
\[
d \eqdef \left\lceil \frac{\max\log\prod_{i=1}^N\left(1+\frac{t_i}{\lambda_i}\right)}{\log\left(1+\frac{T}{K\lambda}\right)}\right\rceil\!\CommaBin
\]
where the maximum is taken over all possible non-negative real numbers $\{t_1,\,\dots,\,t_N\}$, such that $\sum_{i = 1}^N t_i = T$ and $K$ is the number of zero eigenvalues of $\bLambda_{\cL}$. $K$ is also the number of components of $G$.
\end{definition}

\begin{remark}
Note that if we first upper bound every $1/\lambda_i$ in the numerator by $1/\lambda$ then the maximum is acquired for $t_i$ equal to $T/N$. Therefore, the right-hand side of the definition is bounded from above by $D=N$. This means that $d$ is upper bounded by $D$. Later we show that in many practical situations, $d$ is much smaller than $D$.
\end{remark}
For the comparison, we show the previous definition of the effective dimension \citep{valko2014spectral} and from now we call it \textit{old effective dimension} denoted by $\dold$.

\begin{definition}[old effective dimension, \citealp{valko2014spectral}]
Let the \textbf{old effective dimension} $\dold$ be the largest $\dold\in[N]$ such that
\[
(\dold -1)\lambda_{\dold} \le \frac{T}{\log(1 + T/\lambda)}\cdot
\]
\end{definition}

\begin{remark}
\label{rem:newvsold}
Note that from Lemma~5 and Lemma~6 by \cite{valko2014spectral}, we see that the relation between the old and new definition of the effective dimension is: $d\leq2\dold$. As we show later, the bounds using the effective dimension scale either with $d$ or with $2\dold$. Moreover, we show that $d$ is usually much smaller than $2\dold$ and therefore using the new definition of the effective dimension brings an improvement to the bound.
\end{remark}

The effective dimension $d$ is small when the coefficients $\lambda_i$ grow
rapidly above $T$. This is the case when the dimension of the space $D$ is much larger than $T$, such as in graphs from social networks with a very
large number of nodes $N$. In contrast, when the coefficients $\lambda_i$ are all small
(if the graph is sparse, all eigenvalues of Laplacian are small) then $d$ may be of the order of $T$. That would make the regret bounds useless.

The actual form of Definition~\ref{def:effectived} comes from
Lemma~\ref{lem:logdetratio} and becomes apparent in 
Section~\ref{sec:analysis}.
The dependence of the effective dimension on $T$ comes from the fact
that $d$ is related to the number of ``non-negligible'' dimensions
characterizing the space where the solution to the penalized least-squares may
lie, since this solution is basically constrained to an ellipsoid defined
by the inverse of the eigenvalues. This ellipsoid is wide in the directions corresponding 
to the small eigenvalues and narrow in the directions corresponding to the large ones.
After playing an action, the confidence ellipsoid shrinks in the directions of the action.  
Therefore, exploring in a direction where the ellipsoid is wide can reduce the volume of the ellipsoid
much more than exploring in a direction where the ellipsoid is narrow. In fact, for a small $T$, 
the axes of the ellipsoid corresponding to the large eigenvalues of $\cL$ are negligible.
Consequently,~$d$ is related to the metric dimension of this ellipsoid.
Therefore, when $T$ tends to infinity, then all directions matter, thus the solution can be
anywhere in a (bounded) space of dimension $N$. On the contrary, for a smaller
$T$, the ellipsoid possesses a smaller number of ``non-negligible'' dimensions.


\subsubsection{The computation of the effective dimension}\label{sec:effd_alg}

All of the algorithms that we propose need to know the value of the effective dimension
in order to leverage the structure of the problem. Therefore, it is necessary to compute
it beforehand. when computing the effective dimension, we proceed in two steps:

\begin{enumerate}
\item Finding an $N$-tuple $(t_1,\,\dots,\,t_N)$ which maximizes the expression from the definition of the effective dimension.
\item Plugging the $N$-tuple to the definition of the effective dimension.
\end{enumerate}
We now focus on the first step. The following lemma gives us an efficient way to determine the $N$-tuple
\begin{lemma}
Let $\omega\in[N]$ be the largest integer such that
\[
\frac{\sum_{i=1}^\omega \lambda_i}{\omega} + \frac{T}{\omega} - \lambda_\omega >0,
\]
then $t_1,\,\dots,\,t_N$ that maximize the expression in the definition of the effective dimension are in the following form,
\begin{alignat*}{2}
t_i &= \frac{\sum_{i=1}^\omega \lambda_i}{\omega} + \frac{T}{\omega} - \lambda_i \qquad\qquad&&\textrm{for }i=1,\,\dots,\,\omega,\\
t_i &= 0  \qquad\qquad&&\textrm{for } i=\omega+1,\,\dots,\,N.
\end{alignat*}
\end{lemma}
\begin{proof}
First of all, we use the fact that logarithm is an increasing function and that 
the $N$-tuple which maximizes the expression is invariant to a multiplication of the expression by a constant,
\[
\argmax\log\prod_{i=1}^N\left(1+\frac{t_i}{\lambda_i}\right) = \argmax\prod_{i=1}^N\left(1+\frac{t_i}{\lambda_i}\right) = \argmax\prod_{i=1}^N\left(\lambda_i+t_i\right).
\]
The last expression is easy to maximize since we know that for any $\Delta\ge\delta\ge0$ and for any real number $a$ we have
\begin{align*}
0 &\le \Delta^2 - \delta^2									\\
a^2-\Delta^2 &\le a^2-\delta^2							\\
(a-\Delta)(a+\Delta) &\le (a-\delta)(a+\delta).
\end{align*}
Therefore, if we take any two terms $(\lambda_i+t_i)$ and $(\lambda_j + t_j)$ from the expression which we are maximizing, we can potentially increase their product simply by balancing them,
\begin{align*}
t_i^\textrm{new} &\eqdef \frac{\lambda_i + \lambda_j + t_i + t_j}{2} - \lambda_i   \\
t_j^\textrm{new} &\eqdef  \frac{\lambda_i + \lambda_j + t_i + t_j}{2} - \lambda_j.
\end{align*}  
However, we still have to take into consideration that every $t_i$ has to be positive. Therefore, if, for example, $t_j^\textrm{new}$ is negative, we can simply set 
\begin{align*}
t_i^\textrm{new} &\eqdef  t_i + t_j   \\
t_j^\textrm{new} &\eqdef  0.
\end{align*}
We apply this argument to the expression we are trying to maximize to obtain the statement of the lemma.
\end{proof}
The second part is straightforward. To avoid computational difficulties of multiplying $N$ numbers, we use properties of logarithm to get
\[
d = \left\lceil\frac{\max\log\prod_{i=1}^N\left(1+\frac{t_i}{\lambda_i}\right)}{\log\left(1+\frac{T}{K\lambda}\right)}\right\rceil = \left\lceil\frac{\max\sum_{i=1}^N\log\left(1+\frac{t_i}{\lambda_i}\right)}{\log\left(1+\frac{T}{K\lambda}\right)}\right\rceil\!\cdot
\]
Knowing an $N$-tuple which maximizes the expression, we simply plug it in and obtain the value of the effective dimension.


\subsubsection{The old vs.\,new definition of the effective dimension}

As we mentioned in Remark \ref{rem:newvsold}, our new effective dimension is always 
upper bounded by~$2\dold$. In this section, we show that the gap between $d$ and $2\dold$
can be significant. We demonstrate on the graphs constructed for several real-world datasets
and also on several random graphs.
\begin{figure}[ht]
 \begin{center}
  \includegraphics[width=0.32\columnwidth]{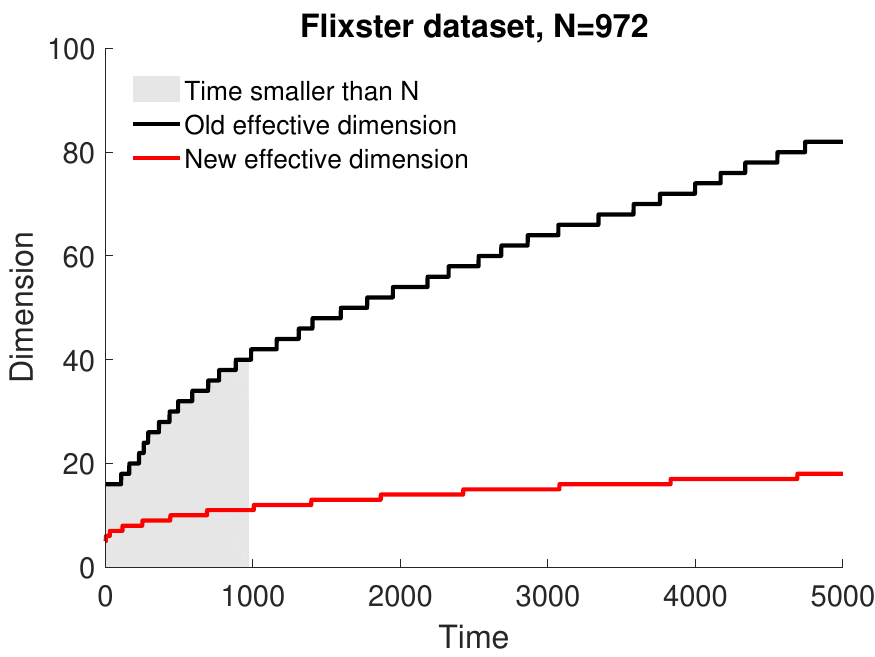}
  \includegraphics[width=0.32\columnwidth]{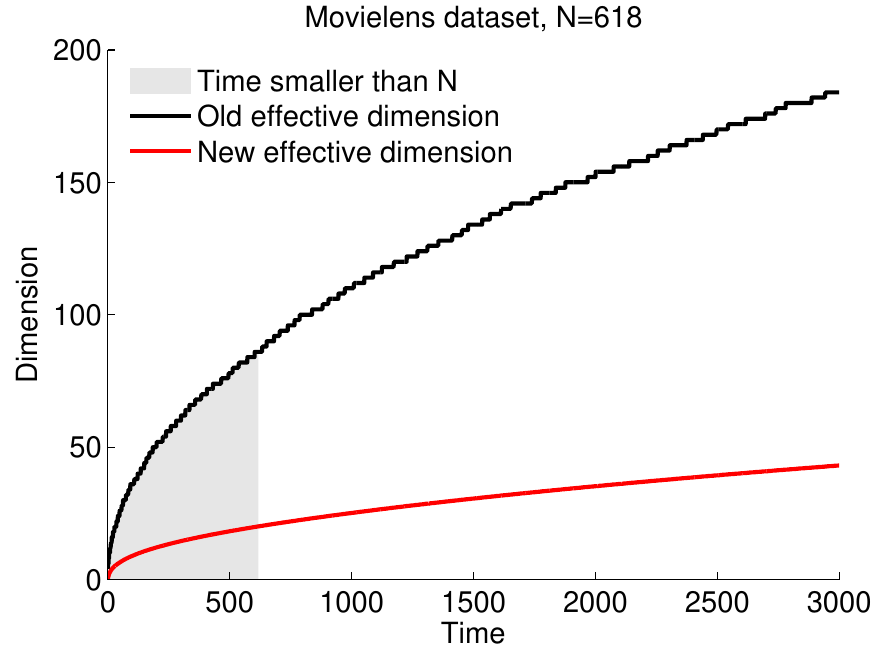}
  \includegraphics[width=0.32\columnwidth]{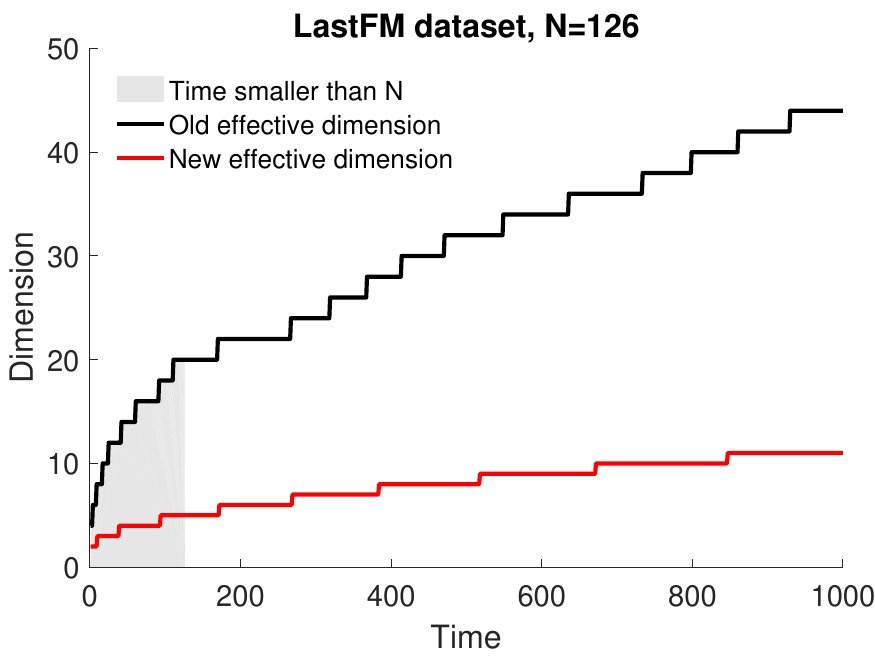}
  \caption{Difference between $d$ and $2\dold$ for real world datasets. From left to right: Flixster dataset with $N=972$, Movielens dataset with $N=618$, and LastFM dataset with $N=804$.}
 \label{fig:effd_real}
 \end{center}
\end{figure}
 \begin{figure}[ht]
 \begin{center}
\includegraphics[width=0.32\columnwidth]{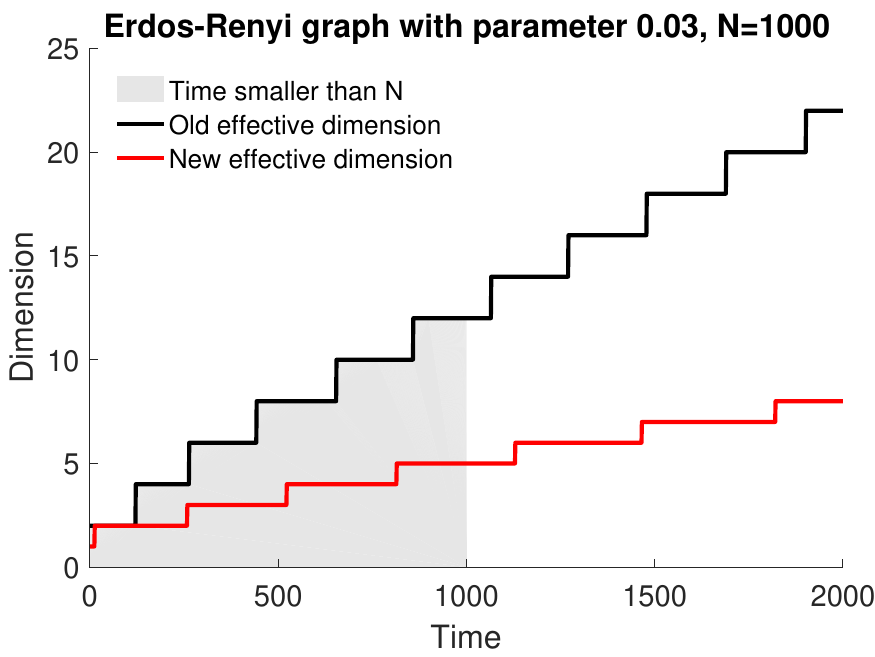}
\includegraphics[width=0.32\columnwidth]{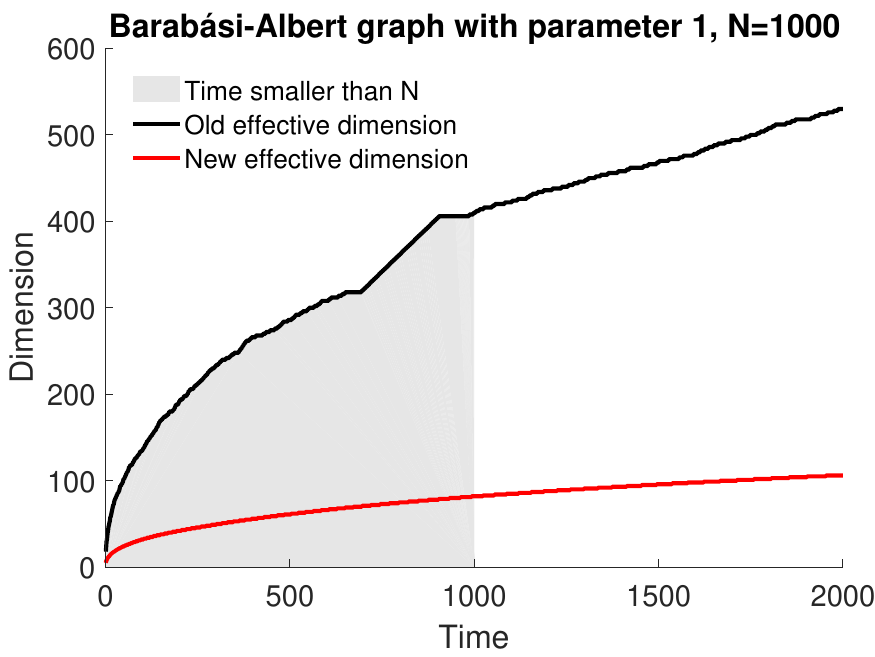}
\includegraphics[width=0.32\columnwidth]{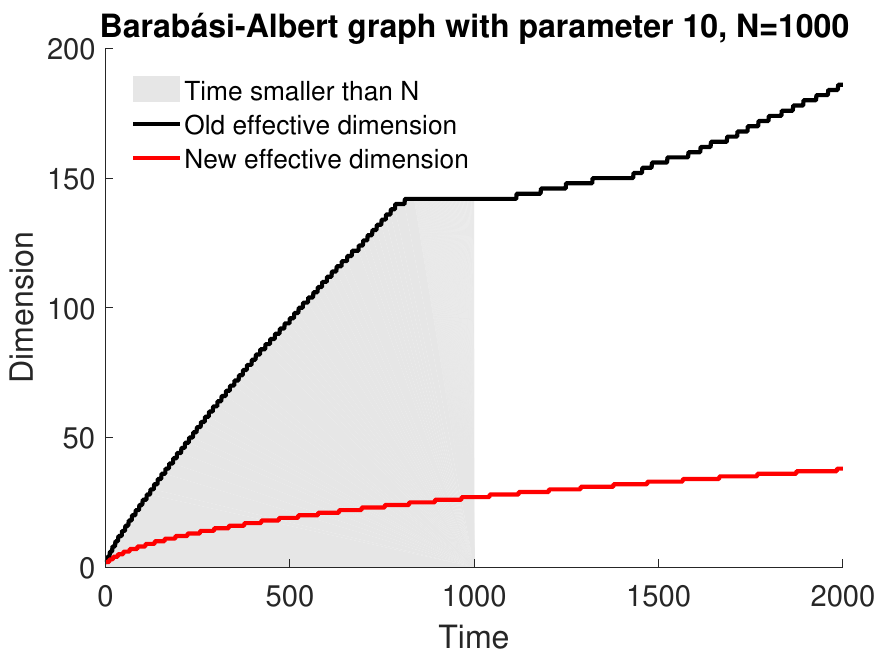}
 \caption{Difference between $d$ and $2\dold$ for random graphs on $N=1000$ nodes. From left to right: Erd\H{o}s-Renyi graph with the probability $0.03$ of an edge, Barab\'asi-Albert graph with one edge per added node, Barab\'asi-Albert graph with ten edges per added node.}
 \label{fig:effd_artificial}
 \end{center}
 \end{figure}
 
 \noindent
Figures~\ref{fig:effd_real} and~\ref{fig:effd_artificial} show how $d$ behaves compared to $2\dold$ on the
generated and the real Flixster, Movielens, and LastFM network graphs.\footnote{We set $\bLambda$ to
$\bLambda_{\cL} + \lambda \bI$ with $\lambda = 0.1$,
where $\bLambda_{\cL}$ is the graph Laplacian of the respective graph.} We
use some of them for the experiments in Section~\ref{sec:exp}. The figures clearly demonstrate the gap between $d$ and $2\dold$ while both of the quantities 
are much smaller then $D$. In fact, effective dimension $d$ is much smaller than~$D$ even for $T > N$ (Figures~\ref{fig:effd_real} and~\ref{fig:effd_artificial}). Therefore, spectral bandits can be used even for $T > N$ while maintaining the advantage of better regret bounds compared to the linear bandit algorithms.

\subsection{Lower bound}
\label{ssec:lowerbound}

In this section, we show a lower bound for the spectral setting.
More precisely, for each possible value of effective dimension $d$ and time horizon $T$,  we show the existence of a ``hard'' problem with a lower bound of $\Omega(\sqrt{dT})$.
We prove the theorem by reducing a carefully selected problem to a multi-arm bandit problem with $d$ arms and using the following lower bound for it.
\begin{theorem}[\citealp{auer2002nonstochastic}]\label{karmlowerbound}
For any number of actions $K \ge 2$ and for any time horizon $T$,
there exists a distribution over the assignment of Bernoulli rewards such that the expected regret of any algorithm (where the expectation is taken with respect to both the randomization
over rewards and the algorithms internal randomization) is at least
\[
R_T \geq \frac{1}{20}\min\left\{\sqrt{KT},\, T\right\}\!.
\]
\end{theorem} 
Theorem~\ref{karmlowerbound} can also be proved without the randomization device.
The constant 1/20 in the lower bound above can be improved into 1/8 
\citep[Theorem~6.11]{cesa-bianchi2006prediction}. 
We now state a lower bound for spectral bandits, featuring the effective dimension~$d$.

\begin{theorem}
For any $T$ and $d$, there exists a problem with effective dimension $d$ and time horizon $T$ such that the expected regret of any algorithm is of $\Omega(\sqrt{dT})$.
\end{theorem}

\begin{proof}
We define a problem with the set of actions consisting of $K = d$ blocks. Each block is a complete graph $K_{M_T}$ on $M_T$ vertices. Moreover, all weights of the edges inside a component are equal to one. We define $M_T$ as a $T$-dependent constant such that the effective dimension of the problem $d$ is exactly $K$. We specify the precise value of $M_T$ later. 

On top of the structure described above, we choose a reward function with smoothness~0, i.e., a constant on each of the components of the graph. In fact, even knowing that the reward function is constant on individual components, this problem is as difficult as the multi-arm bandit problem with $K$ arms. Therefore, the lower bound of $\Omega(\sqrt{KT})$ of the $K$-arm bandit problem applies to our setting too.  Consequently, we have the lower bound of $\Omega(\sqrt{dT})$, since $d= K$. 

The last part of the proof is to show that $d=K$ and therefore, we have to specify value of $M_T$. The graph consists of $K$ blocks consisting $M_T$ vertices each. Therefore, the graph Laplacian is the following matrix
\[ 
\tiny
L=
\begin{tikzpicture}[baseline=-\the\dimexpr\fontdimen22\textfont2\relax ]
\matrix (m)[matrix of math nodes,nodes in empty cells,left delimiter=(,right delimiter=)]
{
\,(M_T-1) & -1 & \dots & -1 &  & & & & \\
-1 & &  & \vdots &  & & & & \\
\vdots &  &  & -1 &  & & & &  \\
-1 & \dots & -1 & \,(M_T-1)\, &  & & & & \\
 &  &  &  & \ddots &  &  &  &  & \\
 & & & & & \,(M_T-1)\, & -1 & \dots & -1\\
 & & & &  & -1 &  &  & \vdots\\
 & & & & & \vdots &  &  & -1\\
 & & & & & -1 & \dots & -1 & \,(M_T-1)\,\\
};

\begin{pgfonlayer}{myback}
\node[fit=(m-6-1)(m-9-4)]{\Huge$0$};
\node[fit=(m-1-6)(m-4-9)]{\Huge$0$};
\draw[loosely dotted,line width=1pt] (m-1-1)-- (m-4-4);
\draw[loosely dotted,line width=1pt] (m-6-6)-- (m-9-9);
\highlight[red]{m-1-1}{m-4-4}
\highlight[red]{m-6-6}{m-9-9}
\end{pgfonlayer}
\end{tikzpicture}
.\]
Now we compute eigenvalues of $L$ to obtain
\[
L\  \xrightarrow{\rm eigenvalues} \ \overbrace{0,\,\dots,\, 0}^{K},\, \overbrace{M_T,\,\dots,\,M_T}^{(M_T-1)K}.
\]
We plug the above eigenvalues to the definition of the effective dimension and to set the value of $M_T$ to obtain 
\[
d = \left\lceil\frac{\max\log\prod_{i=1}^K\left(1+\frac{t_i}{\lambda}\right)\prod_{i=K+1}^{KM_T}\left(1+\frac{t_i}{\lambda + M_T}\right)}{\log\left(1+\frac{T}{K\lambda}\right)}\right\rceil\!\cdot
\]
By setting $M_T \ge T/K$, we have that the maximum in the definition is obtained for $t_1 = \dots = t_K = T/K$ and $t_{K+1} = \dots = t_{KM_T} = 0$. Therefore,
\[
d = \left\lceil\frac{\log\prod_{i=1}^K\left(1+\frac{T}{K\lambda}\right)}{\log\left(1+\frac{T}{K\lambda}\right)}\right\rceil = \left\lceil\frac{\sum_{i=1}^K\log\left(1+\frac{T}{K\lambda}\right)}{\log\left(1+\frac{T}{K\lambda}\right)}\right\rceil = K.
\]
This means that our problem is at least as difficult as the multi-arm bandit problem with $K = d$ arms and therefore, the lower bound for $K$-arm bandits (Theorem~\ref{karmlowerbound}) applies.\end{proof}


\section{Algorithms}

In this section, we introduce the algorithms for spectral bandits: \SpectralUCB, \SpectralTS, and \SpectralEliminator{}.
 For each algorithm, we state the regret bound and later we discuss the computational advantages and compare the theoretical regret bounds of the algorithms with the lower bound provided in the previous section. Full proofs are given in Section \ref{sec:analysis}.


\subsection{\SpectralUCB}
\label{ssec:algospectralucb}
\begin{algorithm}[t]
  \caption{\SpectralUCB}
  \label{alg:TUCB}
\begin{algorithmic}[1]
  \STATE {\bfseries Input:}
  \STATE \quad  $N$:  number of actions
   \STATE \quad   $T$: number of rounds
  \STATE \quad  $\{\bLambda_{\cL}, \bQ\}$: spectral basis of a graph Laplacian $\cL$
  \STATE \quad  $\lambda,\, \delta$: regularization and confidence parameters
  \STATE \quad  $R,\,C$: upper bounds on the noise and $\|\balpha\|_\bLambda$
  \STATE {\bfseries Initialization:}
  \STATE  \quad $\bV_1 \gets   \bLambda \gets  \bLambda_{\cL} + \lambda \bI$ \label{alg:regularization}
  \STATE  \quad $\widehat\balpha_1 \gets  0_N$
  \STATE  \quad $d \gets  \lceil(\max\log\prod_{i=1}^N(1+t_i/\lambda_i))/\log(1 + T /(K\lambda))\rceil $ \quad(Definition \ref{def:effectived})
  \STATE  \quad $c \gets  R\sqrt{ 2d \log(1 + T/(K\lambda)) + 8\log(1/\delta)} + C $
  \STATE {\bfseries Run:}
\FOR{$t = 1$ {\bfseries to} $T$}
  \STATE Choose the node $a_{t}$ (${a_{t}}$-th row of $\bQ$):
  $a_{t} \gets  \argmax_a \left( \bx\transpose_{a}\widehat\balpha_{t} + c \|\bx_a \|_{\bV_t^{-1}} \right)$
  \STATE Observe a noisy reward $r_t \gets  \bx_{a_t}\transpose\balpha + \varepsilon_t$
  \STATE Update the basis coefficients $\widehat\balpha$:
  \STATE \quad $\bV_{t+1} \gets \bV_t + \bx_{a_t}\bx_{a_t}\transpose$
  \STATE \quad $\widehat\balpha_{t+1} \gets  \bV_{t+1}^{-1}\sum_{s = 1}^t\bx_{a_s}r_s$

\ENDFOR
\end{algorithmic}
\end{algorithm}
\noindent
We first present  \SpectralUCB\ (Algorithm~\ref{alg:TUCB})
which is based on \LinUCB{} \citep{li2010contextual} and uses the \textit{spectral
penalty}~\eqref{eq:spectalpenalty} in its least-square estimate.
Here, we consider a \emph{regularized}
least-squares estimate $\widehat\balpha_t$ of the form
\[
\widehat\balpha_t \eqdef \argmin_{\bw\in\R^N} \left( \sum_{s=1}^{t}\left[\bx_{a_s}\transpose \bw  - r_{a_s}\right]^2 + \|\bw\|_{\bLambda}^2  \right)\!.
\]
A key part of the algorithm is to define the $c_t\|\bx\|_{\bV_t^{-1}}$
confidence widths for the prediction of the rewards
and consequently the upper confidence bounds (UCBs).
We take advantage of our analysis (Section~\ref{sec:effd}) to define $c_t$ based on the effective dimension $d$ which is  tailored to our setting.
This way we also avoid the computation of the determinant (see Section~\ref{sec:analysis}).
The following theorem characterizes the performance
of \SpectralUCB\/ and bounds the regret as a function of
effective dimension~$d$.

\begin{theorem}
\label{thm:spectralucb}
Let $d$ be the effective dimension and $\lambda$ be the minimum
eigenvalue of $\bLambda$. If~$\| \balpha\|_{\bLambda} \leq C$
and for all $\bx_a$, $ \bx_a\transpose\balpha \in [-1, 1]$,  then the
cumulative regret of \SpectralUCB\/  is with probability at
least $1-\delta$ bounded as
\begin{align*}
R_T &\leq  \left(2 R\sqrt{ 2d \log\left(1 + \frac{T}{K\lambda}\right) + 8\log\left(\frac{1}{\delta}\right)} + 2C +2\right)  \sqrt{2d T \log\left(1 + \frac{T}{K\lambda}\right)}\\
&\leq \tcO\left(d \sqrt{T} \right).
\end{align*}
\end{theorem}

\begin{remark}
The constant $C$ needs to be such that $\| \balpha \|_{\bLambda} \leq C$.
If we set $C$ too small, the true $\balpha$ will lie outside of the region and far from $\widehat\balpha_t$, causing the algorithm to underperform.
Alternatively, $C$ can be time-dependent, e.g.,~$C_t \eqdef \log t$. In such case, we
do not need to know an upper bound on $\| \balpha \|_{\bLambda}$
in advance, but our regret bound would only hold after some~$t$,  in particular when $C_t \geq \| \balpha \|_{\bLambda}$.
\end{remark}
We provide the proof of Theorem~\ref{thm:spectralucb} in Section~\ref{sec:analysis}
and examine the performance of our \SpectralUCB experimentally
in Section~\ref{sec:exp}. The $d\sqrt{T}$ result of Theorem~\ref{thm:spectralucb}
is to be compared with the standard linear bandits,
where \LinUCB is the algorithm often used in practice \citep{li2010contextual},
achieving $D\sqrt{T}$ cumulative regret.
As mentioned above and demonstrated in Figures~\ref{fig:effd_real} and~\ref{fig:effd_artificial}, in the $T < N$ regime we can expect $d \ll D = N$
and obtain an improved performance.



\subsection{\SpectralTS}
The second algorithm presented in this paper is \SpectralTS which is based on 
\LinearTS, analyzed by \cite{agrawal2013thomson}, and uses Thompson sampling to decide which arm to play.
Specifically, we  represent our
current knowledge about $\balpha$ as a normal distribution 
$\cN(\widehat\balpha_t,v^2\bV_t^{-1})$, where
$\widehat\balpha_t$ is our actual approximation of the unknown vector $\balpha$ and
$v^2\bV_t^{-1}$ reflects our uncertainty about it.
As mentioned before, we assume that the reward function is a linear
combination of eigenvectors of graph Laplacian $\cL$ with large coefficients corresponding to
the eigenvectors with small eigenvalues. We encode this assumption into our
initial confidence ellipsoid by setting $\bV_1 \eqdef \bLambda \eqdef  \bLambda_\cL +
\lambda\bI$, where $\lambda$ is again a regularization parameter.

\noindent In every round~$t$, we generate a sample
$\widetilde\balpha_t$ from the distribution $\cN(\widehat\balpha_t,v^2\bV_t^{-1})$,
choose an arm $a_t$ which maximizes $\bx_i\transpose\widetilde\balpha_t$, and receive a reward. 
Afterwards, we update our estimate of $\balpha$ and the confidence of
it, i.e.,~we compute $\widehat{\balpha}_{t+1}$ and $\bV_{t+1}$,
\[
\bV_{t+1} = \bV_t + \bx_{a_t}\bx_{a_t}\transpose \qquad \text{and} \qquad \widehat{\balpha}_{t+1} = \bV_{t+1}^{-1}\left(\sum_{s=1}^{t}\bx_{a_s}r_s\right)\!.
\]

\begin{algorithm}[t]
   \caption{\SpectralTS}
   \label{alg}
\begin{algorithmic}[1]
   \STATE {\bfseries Input:}
   \STATE \quad $N$: number of actions
   \STATE \quad $T$:  number of rounds
   \STATE \quad $\{\bLambda_\cL,\bQ\}$: spectral basis of a graph Laplacian $\cL$
   \STATE \quad $\lambda$, $\delta$: regularization and confidence parameters
   \STATE \quad $R$, $C$: upper bounds on the noise and $\|\balpha\|_\bLambda$
   \STATE {\bfseries Initialization:}
   \STATE \quad $\bV_1 \gets \bLambda \gets \bLambda_\cL + \lambda\bI_N$
   \STATE \quad $\widehat{\balpha}_1 \gets 0_N$
   \STATE \quad $d \gets \lceil(\max\log\prod_{i=1}^N(1+t_i/\lambda_i))/\log(1 + T /(K\lambda))\rceil $ \quad(Definition \ref{def:effectived})
   \STATE \quad $v \gets R\sqrt{3d\log(1/\delta  + T/(\delta\lambda K))}+C$
   \STATE {\bfseries Run:}
   \FOR{$t=1$ {\bfseries to} $T$}
   \STATE Sample $\widetilde{\balpha}_t \sim \cN(\widehat{\balpha}_t,v^2\bV_t^{-1})$
   \STATE Choose the node $a_{t}$ (${a_{t}}$-th row of $\bQ$):
  $a_t \gets \argmax_{a}\bx_a\transpose\widetilde{\balpha}$
   \STATE Observe a noisy reward $r_t \gets \bx_{a_t}\transpose\balpha +
\varepsilon_t$
   \STATE Update the basis coefficients $\widehat\balpha$:
   \STATE \quad $\bV_{t+1} \gets \bV_t + \bx_{a_t}\bx_{a_t}\transpose$
   \STATE \quad $\widehat\balpha_{t+1} \gets \bV_{t+1}^{-1}\sum_{s = 1}^t\bx_{a_s}r_s$
   \ENDFOR
\end{algorithmic}
\end{algorithm}

\begin{remark}
Since \TS is a Bayesian approach, it requires a prior to run
and we choose it here to be a Gaussian. However, this does not pose any
assumption whatsoever about the actual data
both for the algorithm and the analysis.
The only assumptions we make about the data are: (a) that the mean payoff is
linear in the features, (b) that the noise is sub-Gaussian, and (c) that we know
a bound on the Laplacian norm of the mean reward function.
We provide a \textbf{frequentist} bound on the regret (and not an
average over the prior) which is a much stronger worst-case result.
\end{remark}
The following theorem upper bounds the cumulative regret of \SpectralTS in terms of the effective dimension.

\begin{theorem}\label{thm:spectralts}
Let $d$ be the effective dimension and $\lambda$ be the minimum eigenvalue
of $\bLambda$. If~$\|\balpha\|_\bLambda\leq C$ and for all $\bx_a$,
$\bx_a\transpose \balpha \in[-1,1]$, then the cumulative regret of \SpectralTS\
 is with probability at least $1-\delta$ bounded as
\begin{align*}
R_T\leq\,&\frac{11g}{p}\sqrt{\frac{2+2\lambda}{\lambda}dT\log\left(1+\frac{T}{K\lambda}\right)}+\frac{1}{T}+ \frac{g}{p}\left(\frac{11}{\sqrt{\lambda}}+2\right)\sqrt{2T\log\left(\frac{2}{\delta}\right)}\CommaBin
\end{align*}
where $p = 1/(4e\sqrt{\pi})$ and
\begin{align*}
g =\, &\sqrt{4\log (TN)}\left(R\sqrt{3d\log\left(\frac{1}{\delta}+\frac{T}{\delta\lambda K}\right)}+C\right)	+R\sqrt{d\log\left(\frac{T^2}{\delta}+\frac{T^3}{\delta\lambda K}\right)}+C.
\end{align*}

\end{theorem}

\begin{remark}
Substituting $g$ and $p$, we see that the regret bound scales as
$d\sqrt{T\log{N}}$. Note that $N\!=\!D$ could be exponential in $d$ and therefore we need to
consider factor $\sqrt{\log{N}}$ in our bound. On the other hand, if $N$ is
indeed exponential in $d$, then our algorithm scales with
$\log{D}\sqrt{T\log{D}}=\log(D)^{3/2}\sqrt{T}$ which is even better.
\end{remark}


\subsection{\SpectralEliminator}

\begin{algorithm}[t]
\caption{\SpectralEliminator}
 \label{alg:TransductiveElimination}
 \begin{algorithmic}[1]
\STATE {\bfseries Input:}
\STATE \quad  $N$: number of nodes
\STATE \quad   $T$: number of pulls
\STATE \quad  $\{\bLambda_{\cL}, \bQ\}$: spectral basis of a graph Laplacian $\cL$
\STATE \quad  $\lambda$: regularization  parameter
\STATE \quad  $\beta, \{t_j\}_j^J$: parameters of the elimination and phases where $J = \lfloor\log_2 T\rfloor + 1$
 \STATE \quad $A_1 \gets \{\bx_1,\dots,\bx_K\}$
 \FOR {$j=1$ {\bfseries to} $J$}
 \STATE $\bV_{t_j}\gets  \bLambda_{\cL} + \lambda \bI$
 \FOR{ $t=t_{j}$ {\bfseries to} $\min(t_{j+1}-1,T)$}
   \STATE Play available arm $a_t$ ($\bx_{a_t}\in A_j$) with the largest width and observe $r_t$:
   \STATE \quad $a_t \gets \arg\max_{a |\bx_a \in A_j} \|\bx_a\|_{\bV_{t}^{-1}}$
   \STATE $\bV_{t+1} \gets \bV_{t}+\bx_{a_t}^{\phantom{\mathsf{\scriptscriptstyle T}}} \bx\transpose_{a_t}$  
 \ENDFOR
\STATE Eliminate the arms that are not promising:
\STATE \quad $\widehat\balpha_{j+1} \gets \bV_{t+1}^{-1}
[\bx_{t_{j}},\dots,\bx_{t}]
[r_{t_{j}},\dots,r_{t}]\transpose$
 \STATE \quad  $A_{j+1} \gets  \Big\{\bx\in A_{j}, \langle\widehat\balpha_{j+1}, \bx
\rangle + \|\bx\|_{\bV_{t+1}^{-1}}\beta \geq \max_{\bx\in A_{j}} \Big[ \langle \widehat\balpha_{j+1}, \bx
\rangle - \|\bx\|_{\bV_{t+1}^{-1}}\beta \Big]  \Big\}$
 \ENDFOR
 \end{algorithmic}
\end{algorithm}

\noindent
It is known that the available upper bound for \LinUCB, \LinearTS, or \OFUL is not tight for
 linear bandits with a \emph{finite} number of arms in terms of dimension $D$.  On the other
hand, the algorithms \SupLinRel or \SupLinUCB
achieve the optimal $\sqrt{DT}$ regret.
In the following, we similarly provide an
algorithm that also scales better with $d$
and achieves a $\sqrt{dT}$ regret.
The algorithm is called \SpectralEliminator\/
(Algorithm~\ref{alg:TransductiveElimination})
and works in phases, eliminating the arms
that are not promising. The  phases are defined by the time indexes $t_1=1\leq t_2\leq \dots$
and depend on some parameter $\beta$.
The algorithm is in a spirit similar to  \ImprovedUCB by~\citet{auer2010ucb}.
As a special case and as a side result of independent interest, we  give 
also the \LinearEliminator algorithm. \LinearEliminator achieves the optimal $\sqrt{DT}$ regret
and uses \emph{adaptive confidence intervals}, unlike \SupLinRel or \SupLinUCB, 
that use data-agnostic confidence intervals of the form $2^{-u}$ for $u \in \NN_0$.

In the following theorem, we characterize the performance
of \SpectralEliminator\/ and show
that the upper bound on its regret has a  $\sqrt{d}$ improvement 
over \SpectralUCB{} and \SpectralTS.

\begin{theorem}
\label{thm:eliminator}
Choose the phases' starts as  $t_j \eqdef 2^{j-1}$. Assume all rewards are in $[0,1]$
and $\|\balpha\|_\bLambda\leq C$. For any $\delta>0$, with probability at least
$1-\delta$, the cumulative regret of \SpectralEliminator\/ run with
parameter $\beta \eqdef R \sqrt{\log(2K\!(1+\log_2 T)/\delta)}\!+ C$ is bounded as
\[
R_T \leq 2 + 8\left(R \sqrt{2 \log\left(\frac{2K(1+\log_2
T)}{\delta}\right)}+C + \frac{1}{2}\right) \sqrt{2dT\log_2(T) \log\left(1 + \frac{T}{\lambda K}\right)}\cdot
\]
\end{theorem}


\subsection{Scalability and computational complexity}
\label{ssec:scalability}

There are three main computational issues to address in
order to make proposed algorithms scalable:
the computation of $N$ UCBs (applies to \SpectralUCB), the matrix inversion, and obtaining the eigenbasis
which serves as an input to any of the algorithms.
First, to speed up the computation of $N$ UCBs (that in general takes  $N^3$ time) in each round, we use lazy updates~\citep{desautels12parallelizing} which maintains a sorted
queue of UCBs and using the fact that the UCB for every arm can only decrease after the update. Therefore, the algorithm does not need to update all UCBs in each round. In practice, lazy updates lead to light-speed gains. This issue does not apply to \SpectralTS since we only need to sample~$\widetilde\balpha$
which can be done in $N^2$ time and find a maximum of
$\bx_i\transpose\widetilde\balpha$ which can be also done in $N^2$ time. In general, the computational 
complexity of sampling in \SpectralTS is better than the complexity of computing the $N$ UCBs in \SpectralUCB. 
However, using lazy updates can significantly speed up \SpectralUCB up to the point that \SpectralUCB is comparable 
to  \SpectralTS. 

Second, all of the proposed algorithms need to compute the inverse of an $N\times N$ matrix in each round which is costly. However, we can use Sherman-Morrison formula to  invert the matrix iteratively and thus  speed up the inversion since the matrix changes only by adding a rank-one matrix from one round to the next one.

Finally, while an eigendecomposition of a general matrix is 
computationally difficult, Laplacians are symmetric diagonally
dominant (SDD). This enables us to use fast SDD solvers as CMG by
\cite{koutis2011combinatorial}. Furthermore, using CMG we can find good
approximations to the first
$L$ eigenvectors in $\cO(L m \log m)$ time, where $m$ is the number of edges in
the graph (e.g., $m=10N$ for Flixter data, Section~\ref{sec:flixster}).
CMG can easily work with~$N$ in millions.
In general, we have $L = N$ but from our experience, a smooth reward function
can be often approximated by \emph{dozens} of eigenvectors. In fact, $L$ can be considered
as an upper bound on the number of eigenvectors we actually need.
Furthermore, by choosing small $L$ we not only reduce the complexity of
an eigendecomposition but also the complexity of the least-square problem that is
solved in each round.

Choosing a small $L$ can significantly reduce the computation
but it is important to choose $L$ large enough so that still less than
$L$ eigenvectors are enough. This way, the
problem that we solve remains relevant and our analysis applies.
In short, the problem cannot be solved trivially by choosing first $k$
relevant eigenvectors because $k$ is unknown. Therefore, in practice, we choose
the largest $L$ such that our method is able to
run. In Section~\ref{ssec:computationImprovements}, we  demonstrate
that we can obtain good results with relatively small~$L$.


\section{Analysis}
\label{sec:analysis}
We are now ready to prove regret bounds for all algorithms. First, we show some general preliminary results (Section \ref{sec:preliminaries}). Then, we present several auxiliary lemmas concerning confidence ellipsoid of the estimate (Section \ref{sec:conf}) and effective dimension (Section \ref{sec:effd}). Using these results we upperbound the regrets of \SpectralUCB (Section \ref{sec:regretspectralucb}), \SpectralTS (Section~\ref{sec:regretspectralts}), and \SpectralEliminator (Section \ref{sec:regretspectraleliminator}).
\subsection{Preliminaries}
\label{sec:preliminaries}
The first lemma is a standard anti-concentration inequality for a Gaussian random variable.
\begin{lemma}
\label{gaussianConcentration}
For a Gaussian distributed random variable $Z$ with mean $m$ and variance
$\sigma^2$, for any $z\geq1$,
\[\frac{1}{2\sqrt{\pi}z}e^{-\frac{z^2}{2}}\leq \probability\left(|Z-m|>\sigma
z\right)\leq\frac{1}{\sqrt{\pi}z}e^{-\frac{z^2}{2}}.\]
\end{lemma}

\noindent
Multiple applications of Sylvester's determinant theorem gives our second preliminary lemma.
\begin{lemma}
\label{lemma:logdetassum}
Let $\bV_t = \bLambda + \sum_{s =
1}^{t-1}\bx_{s}\bx_s\transpose$, then
\[\log\frac{|\bV_t|}{|\bLambda|} = \sum_{s = 1}^{t-1}\log\left(1 +
\|\bx_s\|_{\bV_{s}^{-1}}^2\right)\!.\]
\end{lemma}
\noindent
Third lemma says that adding a rank-one matrix to a symmetric positive semi-definite matrix implies the following L{\"o}wner ordering 
for their inverses.

\begin{lemma}\label{lem:updateInverse}
For any symmetric, positive semi-definite matrix $\bX,$ and any vectors \bu{} and \by,
\[
\by\transpose(\bX + \bu\bu\transpose)^{-1}\by \le \by\transpose\bX^{-1}\by.
\]
\end{lemma}

\begin{proof}
Using Sherman-Morrison formula and the fact that inverse of a symmetric matrix is again symmetric, we have
\begin{align*}
-\frac{\left(\bu\transpose\bX^{-1}\by\right)\transpose\left(\bu\transpose\bX^{-1}\by\right)}{1+\bu\transpose\bX^{-1}\bu}&\leq 0 \\
\by\transpose\left(\bX^{-1}-\frac{\bX^{-1}\bu\bu\transpose\bX^{-1}}{1+\bu\transpose\bX^{-1}\bu}\right)\by&\leq \by\transpose\bX^{-1}\by \\
\by\transpose\left(\bX+\bu\bu\transpose\right)^{-1}\by&\leq \by\transpose\bX^{-1}\by.\\
\end{align*}\end{proof}
\begin{corollary} \label{cor:decreasingInverseNorm}
Let $\bV_t \eqdef \bLambda + \sum_{s=1}^{t-1}\bx_{s} \bx_{s}\transpose $. Then for any vector $\bx$ 
and for any positive integers~$t_1$ and $t_2$ satisfying $t_1\leq t_2$,
\[
||\bx||_{\bV_{t_1}^{-1}} \geq ||\bx||_{\bV_{t_2}^{-1}}.
\]
\end{corollary}
\subsection{Confidence ellipsoid}
\label{sec:conf}
We restate the  two lemmas  by \citet{abbasi2011improved} for convenience.
\begin{lemma}[\citealp{abbasi2011improved}, Lemma~9]
\label{lem:selfnorm}
 Let $\bV_t \eqdef \bLambda + \sum_{s=1}^{t-1}\bx_{s} \bx_{s}\transpose $
and define $\bxi_t  \eqdef \sum_{s=1}^{t-1} \varepsilon_s \bx_s$.
With probability at least $1-\delta$, $\forall t\geq 1$,
\begin{align*}
\|\bxi_{t}\|^2_{\bV_t^{-1}}
  \leq 2 R^2 \log \left(\frac{|\bV_t|^{1/2}}
  {\delta\textbf{}|\bLambda|^{1/2}}\right)\!\cdot
\end{align*}
\end{lemma}
\begin{lemma}[\citealp{abbasi2011improved}, Lemma~11]\label{lemma:sumx2}  
For any round $t$, let us define $\bV_t  \eqdef \bLambda + \sum_{s=1}^{t-1}\bx_{s} \bx_{s}\transpose$. Then,
\[\sum_{s = 1}^t\min\left(1,\,\|\bx_s\|^2_{\bV_{s}^{-1}}\right)\leq 2\log\frac{|\bV_{t+1}|}{|\bLambda|}\cdot\]
\end{lemma}
%
The next lemma is a generalization of Theorem~2 by~\citet{abbasi2011improved}
to the regularization with~$\bLambda$.

\begin{lemma}
\label{lem:confinterval}
 Let $\bV_t  \eqdef \bLambda + \sum_{s=1}^{t-1}\bx_{s} \bx_{s}\transpose$
and $\| \balpha \|_{\bLambda} \leq C$. With probability at least $1-\delta$,
for any $\bx$  and $t\geq 1$,
\begin{align*}
|\bx\transpose \widehat \balpha_t-\bx\transpose\balpha|\leq
\|\bx  \|_{\bV_t^{-1}} \left(R
\sqrt{2\log \left(\frac{|\bV_t|^{1/2}}
  {\delta\textbf{}|\bLambda|^{1/2}}\right)} + C \right)\!.
\end{align*}
\end{lemma}

\begin{proof}
We have that
\begin{align*}
|\bx\transpose \widehat \balpha_t-\bx\transpose\balpha| &= |\bx\transpose (-\bV_t^{-1}
\bLambda \balpha + \bV_t^{-1} \bxi_{t} )| 
\leq |\bx\transpose \bV_t^{-1} \bLambda \balpha | + | \bx\transpose
\bV_t^{-1} \bxi_{t} | \\
&\leq |\bx\transpose\bV_t^{-\frac{1}{2}}\bV_t^{-\frac{1}{2}}\bLambda\balpha|  + |\bx\transpose\bV_t^{-\frac{1}{2}}\bV_t^{-\frac{1}{2}}\bxi_t|
\leq \|\bx  \|_{\bV_t^{-1}} \left(\|\bxi_{t} \|_{\bV_t^{-1}} + \|\bLambda
\balpha \|_{\bV_t^{-1}} \right)\!,
\end{align*}
where we use Cauchy-Schwarz inequality in the last step.
Now, we bound $\|\bxi_{t} \|_{\bV_t^{-1}}$ by Lemma~\ref{lem:selfnorm} and using Corollary \ref{cor:decreasingInverseNorm} we bound $||\bLambda\balpha||_{V_t^{-1}}$ as
\begin{align*}
\|\bLambda \balpha \|_{\bV_t^{-1}} \leq \|\bLambda \balpha \|_{\bV_1^{-1}} = \|\bLambda \balpha \|_{\bLambda^{-1}} = \| \balpha \|_{\bLambda} \leq C.
 \end{align*}\end{proof}
\subsection{Effective dimension}
\label{sec:effd}
In Section~\ref{sec:conf}, we show that several quantities scale with
 $\log(|\bV_t|/|\bLambda|)$, which
can be of the order of $D$. 
Therefore, in this part, we present the key ingredient
of our analysis, based on the geometrical
properties of determinants (Lemmas~\ref{lemma:2} and~\ref{lem:detgeo}),
to upperbound $\log(|\bV_t|/|\bLambda|)$ by a term
that scales with $d$ (Lemma~\ref{lem:logdetratio}). Not only this allows us
to show that the regret bound
scales with $d$, but it also helps us to avoid the
computation of the determinants in Algorithm~\ref{alg:TUCB}.

\begin{lemma}\label{lemma:1}
For any real positive-definite matrix $A$ with only simple eigenvalue
multiplicities
and any vector $\bx$ such that $\|\bx \|_2\leq 1$, we have that the determinant
$|\bA+\bx \bx\transpose |$ is maximized by a vector $\bx$ which is aligned with
an eigenvector of $\bA$.
\end{lemma}

\begin{proof}
Using Sylvester's determinant theorem, we have that
\[|\bA+\bx \bx\transpose | = |\bA| |\bI+\bA^{-1}\bx\bx\transpose| = |\bA|
(1+\bx\transpose \bA^{-1} \bx).\]
From the spectral theorem, there exists an orthonormal matrix $\bU$, the columns
of which are the eigenvectors of $\bA$, such that $\bA=\bU \bD \bU\transpose$
with
$\bD$  being a diagonal matrix with the positive eigenvalues of $\bA$ on the
diagonal. Thus,
\begin{align*}
\max_{\|\bx\|_2\leq 1} \bx\transpose \bA^{-1} \bx
&= \max_{\|\bx\|_2\leq 1} \bx\transpose \bU \bD^{-1} \bU\transpose \bx = \max_{\|\by\|_2\leq 1} \by\transpose \bD^{-1} \by,
\end{align*}
since $\bU$ is a bijection from $\{\bx, \|\bx\|_2\leq 1\}$ to itself.

As there are no multiplicities, it is easy to see that the quadratic mapping
$\by\mapsto \by\transpose \bD^{-1} \by$  is maximized
(under the constraint $\|\by \|_2\leq 1$) by a canonical vector
$\be_I$ corresponding to the lowest diagonal entry $I$ of $\bD$.
Thus the maximum of $\bx\mapsto \bx\transpose \bA^{-1} \bx$ is reached for $\bU
\be_I$, which is the eigenvector of $\bA$ corresponding to its lowest
eigenvalue.
\end{proof}

\begin{lemma}\label{lemma:2}
Let $\bLambda\eqdef\mbox{\normalfont diag}(\lambda_1,\dots,\lambda_N)$
be any diagonal matrix with strictly positive entries. For any vectors
$(\bx_s)_{1\leq s< t}$
 such that $\|\bx_s \|_2\leq 1$ for all $1\leq s< t$, we have that the
 determinant $|\bV_{t}|$ of $\bV_{t}\eqdef\bLambda + \sum_{s=1}^{t-1} \bx_s \bx_s\transpose$
 is maximized when all $\bx_s$ are  aligned with the axes.
\end{lemma}

\begin{proof}
Let us write $d(\bx_1,\dots,\bx_{t-1}) \eqdef  |\bV_t|$ the determinant of $\bV_t$. We want
to characterize
\[\max_{\bx_1,\dots,\bx_{t-1}: \|\bx_s\|_2\leq 1, \forall 1\leq s< t}
d\left(\bx_1,\dots,\bx_{t-1}\right)\!.\]
%
For any $1\leq i< t$, let us define
\[\bV_{-i} \eqdef \bLambda + \sum\limits_{{\begin{array}{c} \\[-1.5em]
\scriptstyle{s=1}\\[-0.5em] \scriptstyle{s\neq i}\end{array}}}^{t-1}
\bx_s\bx_s\transpose.\]
We have that
$\bV_t = \bV_{-i} + \bx_i\bx_i\transpose$. Consider the case with only simple
eigenvalue
multiplicities.
In this case, Lemma~\ref{lemma:1} implies that $\bx_i \mapsto
d(\bx_1,\dots,\bx_i,\dots,\bx_{t-1})$ is
 maximized when $\bx_i$ is aligned with an eigenvector of $\bV_{-i}$.
Thus all $\bx_s$, for $1\leq s< t$, are aligned with an eigenvector
of $\bV_{-i}$ and therefore also with an eigenvector of $\bV_t$.
Consequently, the eigenvectors of $\sum_{s=1}^{t-1} \bx_s \bx_s\transpose$
are also aligned with $\bV_t$. Since $\bLambda = \bV_t - \sum_{s=1}^{t-1} \bx_s
\bx_s\transpose$
and $\bLambda$ is diagonal, we conclude that $\bV_t$ is diagonal and
all $\bx_s$ are  aligned with the canonical axes.


Now in the case of eigenvalue multiplicities, the maximum of $|\bV_t|$
may be reached by several sets of vectors $\{(\bx^m_s)_{1\leq s< t}\}_m$
but for some $m^\star$, the set $(\bx^{m^\star}_s)_{1\leq s< t}$
will be aligned with the axes.
In order to see that, consider a perturbed matrix $\bV_{-i}^\varepsilon$
by a random perturbation of amplitude at most $\varepsilon$,
i.e.~such that $\bV_{-i}^\varepsilon\rightarrow \bV_{-i}$ when
$\varepsilon\rightarrow 0$.
Since the perturbation is random, then the probability that
$\bLambda^\varepsilon$,
as well as all other $\bV_{-i}^\varepsilon$ possess an eigenvalue of
multiplicity bigger than 1 is zero.
Since the mapping $\varepsilon\mapsto \bV_{-i}^\varepsilon$  is continuous,
we deduce that any adherent point $\overline{\bx}_i$ of the sequence
$(\bx_i^\varepsilon)_{\varepsilon}$
(there exists at least one since the sequence is bounded in $\ell_2$-norm)
is aligned with the limit $\bV_{-i}$ and we  apply the previous reasoning.\end{proof}

\begin{lemma}\label{lem:detgeo}
For any $t$, let $\bV_{t}\eqdef\sum_{s=1}^{t-1} \bx_s \bx_s\transpose + \bLambda$. Then,
\[
\log\frac{|\bV_{t}|}{| \bLambda  |} \leq \max \sum_{i=1}^N\log\left(1+\frac{t_i}{\lambda_i}\right)\!\CommaBin
\]
where the maximum is taken over all possible positive real
numbers $\{t_1,\dots,t_N\}$, such that $\sum_{i=1}^N t_i = t-1$.
\end{lemma}
\begin{proof}
We want to bound the determinant $|\bV_{t}|$ under the coordinate
constraints $\|\bx_s\|_2\leq 1$. Let
\[M(\bx_1,\dots,\bx_{t-1})\eqdef \left| \bLambda + \sum_{s=1}^{t-1} \bx_s
\bx_s\transpose\right|\!.\]
From Lemma~\ref{lemma:2}, we deduce that the maximum of $M$ is reached when all
$\bx_t$ are aligned with the axes,
\begin{eqnarray*}
M&=&\max_{\bx_1,\dots,\bx_{t-1}; \bx_s\in \{\be_1,\dots, \be_N\}} \left| \bLambda +
\sum_{s=1}^{t-1} \bx_s \bx_s\transpose\right| \\
&=& \max_{t_1,\dots,t_N \mbox{\scriptsize \ positive integers}, \sum_{i=1}^N t_i= t-1} \left|
\mbox{diag}\left(\lambda_i + t_i \right)\right| \\
&\leq& \max_{t_1,\dots,t_N \mbox{\scriptsize \ positive reals}, \sum_{i=1}^N t_i= t-1}
\prod_{i=1}^N \left(\lambda_i + t_i \right)\!,
\end{eqnarray*}
from which we obtain the result.\end{proof}
\begin{lemma}\label{lem:logdetratio}
Let $d$ be the effective dimension and $t\leq T+1$. Then,
\[
\log\frac{|\bV_{t}|}{| \bLambda  |} \leq d \log\left(1+\frac{T}{K\lambda}\right)\!\cdot
\]
\end{lemma}

\begin{proof}
Using Lemma~\ref{lem:detgeo} and Definition \ref{def:effectived} we have
\begin{align*}
\log\frac{|\bV_{t}|}{| \bLambda  |} &\le \max \sum_{i=1}^N\log\Big(1+\frac{t_i}{\lambda_i}\Big) \cr
&= \frac{\max \sum_{i=1}^N\log\Big(1+\frac{t_i}{\lambda_i}\Big)}{\log(1+T/(K\lambda))}\log\left(1 + \frac{T}{K\lambda}\right) \cr
&\le \left\lceil\frac{\max \sum_{i=1}^N\log\Big(1+\frac{t_i}{\lambda_i}\Big)}{\log(1+T/(K\lambda))}\right\rceil\log\left(1 + \frac{T}{K\lambda}\right) \cr
&= d\log\left(1 + \frac{T}{K\lambda}\right)\!\cdot
\end{align*}
\end{proof}
\subsection{Regret bound of \SpectralUCB}
\label{sec:regretspectralucb}

The analysis of \SpectralUCB  has
two main ingredients. The first one
is the derivation of the confidence ellipsoid
for $\widehat\balpha$, which is a straightforward update of 
the analysis of \OFUL \citep{abbasi2011improved} using
the self-normalized martingale inequality from Section~\ref{sec:conf}.
The second part is crucial
for showing
that the final regret bound
scales only with the effective dimension~$d$
and not with the ambient dimension $D$.
We achieve this by considering
the geometrical properties
of the determinant that hold in our setting
(Section~\ref{sec:effd}).

\begin{proof}[\textbf{Theorem~\ref{thm:spectralucb}}]
Let $\bx_{\star} \eqdef \argmax_{\bx_v} \bx_v\transpose\balpha$ and let $\regret(t)$ denote the instantaneous regret at round $t$. With probability at least $1-\delta$, for all $t$:
\begin{align}
\regret(t)  &= \bx_{\star}\transpose\balpha - \bx_{a_t}\transpose\balpha  \nonumber \\
 &\leq  \bx_{a_t}\transpose\widehat\balpha_t + c\|\bx_{a_t}\|_{\bV_t^{-1}} - \bx_{a_t}\transpose\balpha \label{ThRuseOFU} \\
        &\leq  \bx_{a_t}\transpose\widehat\balpha_t + c\|\bx_{a_t}\|_{\bV_t^{-1}} - \bx_{a_t}\transpose\widehat\balpha_t + c\|\bx_{a_t}\|_{\bV_t^{-1}}   \label{ThRuseConf} \\
 & = 2 c \|\bx_{a_t}\|_{\bV_t^{-1}}, \nonumber
\end{align}
where  \eqref{ThRuseOFU} is by algorithm design and
reflects the optimistic principle of \SpectralUCB\/. Specifically,
$ \bx_{\	}\transpose\widehat{\balpha}_t + c\|\bx_{\star}\|_{\bV_t^{-1}} \leq \bx_{a_t}\transpose\widehat\balpha_t + c\|\bx_{a_t}\|_{\bV_t^{-1}},
$
from which
\[
 \bx_{\star}\transpose\balpha  \leq \bx_{\star}\transpose\widehat{\balpha}_t + c\|\bx_{\star}\|_{\bV_t^{-1}}\leq \bx_{a_t}\transpose\widehat\balpha_t + c\|\bx_{a_t}\|_{\bV_t^{-1}}.
\]
In~\eqref{ThRuseConf}, we apply Lemma~\ref{lem:confinterval},
$
 \bx_{a_t}\transpose\widehat\balpha_t \leq \bx_{a_t}\transpose\balpha + c\|\bx_{a_t}\|_{\bV_t^{-1}}.
$
Now, by Lemmas~\ref{lemma:sumx2} and~\ref{lem:logdetratio},
\begin{align*}
R_T &= \sum_{t=1}^T \regret(t) \leq \sum_{t=1}^T \min\left(2,\,2c\|\bx_{a_t}\|_{\bV_t^{-1}}\right) 
\leq (2+2c)\sum_{t=1}^T \min\left(1,\,\|\bx_{a_t}\|_{\bV_t^{-1}}\right)\\
&\leq(2+2c)\sqrt{T\sum_{t=1}^T \min\left(1,\,\|\bx_{a_t}\|_{\bV_t^{-1}}^2\right)}
\leq(2+2c)\sqrt{2T\log\frac{|\bV_{T+1}|}{|\bLambda|}} \\
&\leq(2+2c)\sqrt{2dT\log\left(1 + \frac{T}{K\lambda}\right)}\cdot
\end{align*}
By plugging $c$, we get that  the theorem holds with probability at least $1-\delta$.
\end{proof}

\begin{remark}
\label{rem:linucb}
Notice that if we set $\bLambda \eqdef
\bI$ in Algorithm~\ref{alg:TUCB}, we recover the  \LinUCB algorithm. Since
$\log(|\bV_{T+1}|/|\bLambda|)$is upper bounded by $D \log T$
\citep{abbasi2011improved}, we obtain $\widetilde\cO(D\sqrt{T})$ upper bound
of regret of \LinUCB as a corollary of Theorem~\ref{thm:spectralucb}.
The known $\widetilde\cO(\sqrt{DT})$ upper bound of \cite{chu2011contextual} applies to a related but 
\emph{different} \SupLinUCB, which is not efficient.
\end{remark}


\subsection{Regret bound of \SpectralTS}
\label{sec:regretspectralts}
The regret bound of \SpectralTS is based on the proof technique of~\citet{agrawal2013thomson}.
Before applying the technique, we first give an intuitive explanation.
 Each round an arm is played, our
algorithm improves the confidence about our actual estimate of 
$\balpha$ via an update of~$\bV_t$ and thus the update of the confidence ellipsoid.
However, when we play a
suboptimal arm, the regret we obtain can be
much
higher than the improvement of our knowledge. To overcome this difficulty, the
arms are divided into two groups of \textit{saturated} and \textit{unsaturated}
arms, based on whether the standard deviation for an arm is smaller than the
 standard deviation of the optimal arm
(Definition~\ref{saturation}) or not. Consequently, the optimal arm is in the group
of unsaturated arms. The idea is to bound the regret of playing
an unsaturated arm in terms of standard deviation and to show that
the probability that the saturated arm is played is small enough.
This way, we overcome the difficulty of high regret and small knowledge
 obtained by playing an arm.

\begin{definition}
We define $E^{\widehat\balpha}(t)$ as the event when for all $i$,
\[
|\bx_i\transpose\widehat\balpha_t-\bx_i\transpose\balpha|\leq \ell\|\bx_i\|_{\bV_t^{-1}},
\]
where
\[
\ell \eqdef R\sqrt{d\log\left(\frac{T^2}{\delta} + \frac{T^3}{\delta\lambda K}\right)} + C,
\]
and $E^{\widetilde\balpha}(t)$ as the event when for all $i$,
\[
|\bx_i\transpose\widetilde\balpha_t-\bx_i\transpose\widehat\balpha_t|\leq v\|\bx_i\|_{\bV_t^{-1}}\sqrt{4\ln(TN)},
\]
where
\[
v \eqdef R\sqrt{3d\log\left(\frac{1}{\delta} + \frac{T}{\delta\lambda K}\right)}+C.
\]
\end{definition}

\begin{definition}\label{saturation}
Let $\Delta_i \eqdef \bx\transpose _{a^\star}\balpha - \bx\transpose_i\balpha$. We say that an arm $i$ is \textbf{saturated} at round $t$ if
$\Delta_i>g\|\bx_i\|_{\bV_t^{-1}}$ and \textbf{unsaturated}
otherwise, including the optimal arm $a_\star$. Let $C(t)$ denote the \textbf{set of saturated arms} at 
round $t$.
\end{definition}

\begin{definition}
We define the filtration $\cF_{t-1}$ as the union of the history until round $t-1$
and
features, 
\[\cF_{t-1} \eqdef \{\cH_{t-1}\} \cup \{\bx_i,i=1,\dots,N\}.\]
By definition, $\cF_1\subseteq\cF_2\subseteq\dots\subseteq\cF_{T-1}$.
\end{definition}

\begin{lemma}\label{concentrationOfMus}
For all $t$, $0<\delta<1$, $\probability(E^{\widehat\balpha}(t))\geq1-\delta/T^2$, and for all
possible filtrations $\cF_{t-1}$,
\[
\P\left(E^{\widetilde\balpha}(t)\,|\,\cF_{t-1}\right)\geq1-\frac{1}{T^2}\cdot
\]
\end{lemma}

\begin{proof}
\textbf{Bounding the probability of event $E^{\widehat{\balpha}}(t)$:}
Using Lemma \ref{lem:confinterval}, where $C$ is such that
$\|\balpha\|_\bLambda\leq
C$, for all $i$ with probability at least $1-\delta'$ we have that
\begin{align*}
|\bx_i\transpose (\widehat \balpha_t-\balpha)| &\leq \|\bx_i  \|_{\bV_t^{-1}} \left(R
\sqrt{2\log \left(\frac{|\bV_t|^{1/2}} {\delta'\textbf{}|\bLambda|^{1/2}}\right)}
+ C \right) \\
&= \|\bx_i  \|_{\bV_t^{-1}} \left(R \sqrt{\log\frac{|\bV_t|} {|\bLambda|}+
2\log\left(\frac{1}{\delta'}\right)} + C \right)\!.
\end{align*}
Therefore, using Lemma \ref{lem:logdetratio} and substituting $\delta' =
\delta/T^2$, we get that with probability at least $1-\delta/T^2$, for all $i$,
\begin{align*}
|\bx_i\transpose  (\widehat \balpha_t-\balpha)| \leq\, &\|\bx_i  \|_{\bV_t^{-1}} \left(R \sqrt{d\log\left(1 + \frac{T} {K\lambda}\right)+ d\log\left(\frac{T^2}{\delta}\right)} + C \right)	\\
=\, &\|\bx_i  \|_{\bV_t^{-1}} \left(R\sqrt{d\log\left(\frac{T^2}{\delta}+\frac{T^3}{\delta\lambda K}\right)}+C\right) =\, \ell\|\bx_i  \|_{\bV_t^{-1}}.
\end{align*}
\textbf{Bounding the probability of event $E^{\widetilde{\balpha}}(t)$:}
The probability of each individual term
$|\bx_i\transpose(\widetilde{\balpha}_t-\widehat{\balpha}_t)|<\sqrt{4\log(TN)}$ can be
bounded using Lemma \ref{gaussianConcentration} to get
\[
\probability\left(|\bx_i\transpose(\widetilde{\balpha}_t-\widehat{\balpha}_t)|\geq
v\|\bx_i\|_{\bV^{-1}_t}\sqrt{4\log(TN)} \right)\leq \frac{e^{-2\log{TN}}}{\sqrt{\pi4\log(TN)}}\leq\frac{1}{T^2N}\cdot
\]
We complete the proof by taking a union bound over all $N$ vectors~$\bx_i$.
Notice that we took a \emph{different} approach than~\cite{agrawal2013thomson} to
avoid the dependence on the ambient dimension~$D$.
\end{proof}
\begin{lemma}\label{lemmaWithLambda}
For any filtration $\cF_{t-1}$ such that $E^{\widehat\balpha}(t)$ is true,
\[\probability\left(\bx_{a_\star}\transpose\widetilde\balpha_t>\bx_{a_\star}\transpose\balpha \,|\,\cF_{t-1}\right)\geq \frac{1}{4e\sqrt{\pi}}\cdot\]
\end{lemma}
\begin{proof}
Since $\bx_{a_\star}\transpose\widetilde\balpha_t$ is a Gaussian random variable with
the mean $\bx_{a_\star}\transpose\widehat\balpha_t$ and the standard deviation
$v\|\bx_{a_\star}\|_{\bV_t^{-1}}$, we can use the anti-concentration inequality from
Lemma \ref{gaussianConcentration},
\begin{align*}
\probability\left(  \bx_{a_\star}\transpose\widetilde\balpha_t\geq \bx_{a_\star}\transpose\balpha \,|\,
\cF_{t-1}\right)  &=\probability\left( \frac{\bx_{a_\star}\transpose\widetilde\balpha_t-\bx_{a_\star}\transpose\widehat\balpha_t}{v\|\bx_{a_\star}\|_{\bV_t^{-1}}}\geq \frac{\bx_{a_\star}\transpose\balpha-\bx_{a_\star}\transpose\widehat\balpha_t}{v\|\bx_{a_\star}\|_{\bV_t^{-1}}} \,|\, \cF_{t-1}\right) 
\geq\frac{1}{4\sqrt{\pi}Z_t}e^{-Z_t^2}, 
\end{align*}
where
\[ 
|Z_t| \eqdef \left|\frac{\bx_{a_\star}\transpose\balpha-\bx_{a_\star}\transpose\widehat\balpha_t}{v\|\bx_{a_\star}\|_{\bV_t^{-1}}}\right|\!.
\]
Since we consider filtration $\cF_{t-1}$ such that $E^{\widehat\balpha}(t)$ is true,
we can upperbound the numerator to get
\begin{align*}
\left|Z_t\right| \leq \frac{\ell\|\bx_{a_\star}\|_{\bV_t^{-1}}}{v\|\bx_{a_\star}\|_{\bV_t^{-1}}}=\frac{\ell}{v}\leq1.
\end{align*}
Finally,
\[\probability\left(\bx_{a_\star}\transpose\widetilde\balpha_t>\bx_{a_\star}\transpose\balpha\,|\,\cF_{t-1} \right)\geq \frac{1}{4e\sqrt{\pi}}\cdot\]\end{proof}
\begin{lemma}\label{atIsUnsaturated}
For any filtration $\cF_{t-1}$ such that $E^{\widehat\balpha}(t)$ is true,
\[\probability\left(a_t\not\in C(t)\,|\,\cF_{t-1}\right) \geq \frac{1}{4e\sqrt{\pi}}-\frac{1}{T^2}\cdot\]
\end{lemma}

\begin{proof}
The algorithm chooses the arm with the highest value of
$\bx_i\transpose\widetilde{\balpha}_t$ to be played at round $t$. Therefore, if
$\bx_{a_\star}\transpose\widetilde{\balpha}_t$ is greater than
$\bx_{j}\transpose\widetilde{\balpha}_t$ for all saturated arms,
i.e.,~$\bx_{a_\star}\transpose\widetilde{\balpha}_t>\bx_{j}\transpose\widetilde{\balpha}_t,\,
\forall j\in C(t)$, then one of the unsaturated arms (that include the optimal
arm and other suboptimal unsaturated arms) must be played. Therefore,
\[
\probability\left(a_t\not\in C(t)\,|\,\cF_{t-1}\right) \geq\, \probability\left(\bx_{a_\star}\transpose\widetilde{\balpha}_t>\bx_{j}\transpose\widetilde{\balpha}_t,\,\forall j\in C(t)\,|\,\cF_{t-1}\right)\!.
\]
By definition, for all saturated arms $j\in C(t)$,
$\Delta_j>g\|\bx_{j}\|_{\bV_t^{-1}}$. Now if both of the events
$E^{\widehat{\balpha}(t)}$ and $E^{\widetilde{\balpha}(t)}$ are true, then, by definition of
these events, for all $j\in C(t)$, $\bx_{j}\transpose\widetilde{\balpha}_t\leq
\bx_{j}\transpose{\balpha}_t+g\|\bx_{j}\|_{\bV_t^{-1}}$. Therefore, given
filtration $\cF_{t-1}$, such that $E^{\widehat{\balpha}}(t)$ is true, either
$E^{\widetilde{\balpha}}(t)$ is false, otherwise for all $j\in C(t)$, 
\[\bx_{j}\transpose\widetilde{\balpha}_t \leq \bx_{j}\transpose{\balpha}+g\|\bx_{j}\|_{\bV_t^{-1}} \leq \bx_{a_\star}\transpose{\balpha}.\]
Hence, for any $\cF_{t-1}$ such that $E^{\widehat{\balpha}}(t)$ is true, 
\begin{align*}
\P\left(\bx_{a_\star}\transpose\widetilde{\balpha}_t>\bx_{j}\transpose\widetilde{\balpha}_t,\,
\forall j\in C(t)\,|\,\cF_{t-1}\right)  &\geq\P\left(\bx_{a_\star}\transpose\widetilde{\balpha}_t>\bx_{a_\star}\transpose{\balpha}\,|\,\cF_{t-1}\right)-\probability\left(\overline{E^{\widehat{\balpha}}(t)}\,|\,\cF_{t-1}\right) \\
&\geq\frac{1}{4e\sqrt{\pi}}-\frac{1}{T^2}\CommaBin
\end{align*}
where in the last inequality we used Lemma~\ref{concentrationOfMus} and
Lemma~\ref{lemmaWithLambda}.\end{proof}
\begin{lemma}\label{lem:regretBound}
For any filtration $\cF_{t-1}$ such that $E^{\widehat\balpha}(t)$ is true,
\[\E\left[\Delta_{a_t}\,|\,\cF_{t-1}\right]\leq \frac{11g}{p}\E\left[\|\bx_{a_t}\|_{\bV_t^{-1}}\,|\,\cF_{t-1}\right]+\frac{1}{T^2}\cdot\]
\end{lemma}
\begin{proof}
Let $\overline{a}_t$ denote the unsaturated arm with the smallest norm
$\|\bx_i\|_{\bV_t^{-1}}$, 
\[\overline{a}_t \eqdef \argmin_{i\not\in C(t)}\|\bx_i\|_{\bV_t^{-1}}.\]
Notice that since $C(t)$ and $\|\bx_i\|_{\bV_t^{-1}}$ for all $i$ are fixed on
fixing $\cF_{t-1}$, so is $\overline{a}_t$. Now, using Lemma
\ref{atIsUnsaturated}, for any $\cF_{t-1}$ such that $E^{\widehat\balpha}(t)$ is true,
\begin{align*}
\E\left[\|\bx_{a_t}\|_{\bV_t^{-1}}\,|\,\cF_{t-1}\right] &\geq\E\left[\|\bx_{a_t}\|_{\bV_t^{-1}}\,|\,\cF_{t-1},a_t\not\in C(t)\right]\, \P\left(a_t\not\in C(t)\,|\,\cF_{t-1}\right)  \\
&\geq\|\bx_{\overline{a}_t}\|_{\bV_t^{-1}}\left(\frac{1}{4e\sqrt{\pi}}-\frac{1}{T^2}\right)\!\cdot
\end{align*}
Now, if events $E^{\widehat\balpha}(t)$ and $E^{\widetilde\balpha}(t)$ are true, then for
all $i$, by definition,
$\bx_i\transpose\widetilde\balpha_t\leq\bx_i\transpose\balpha+g\|\bx_i\|_{\bV_t^{-1}}$.
Using this observation along with 
$\bx_{a_t}\transpose\widetilde\balpha_t\geq\bx_i\transpose\widetilde\balpha_t$ for all~$i$,
\begin{align*}
\Delta_{a_t}=\,&\Delta_{\overline{a}_t} + (\bx_{\overline{a}_t}\transpose\balpha-\bx_{a_t}\transpose\balpha) \\
\leq\, &\Delta_{\overline{a}_t} + (\bx_{\overline{a}_t}\transpose\widetilde\balpha_t-\bx_{a_t}\transpose\widetilde\balpha_t) 
+g\|\bx_{\overline{a}_t}\|_{\bV_t^{-1}}+g\|\bx_{a_t}\|_{\bV_t^{-1}} \\
\leq\,&\Delta_{\overline{a}_t} +g\|\bx_{\overline{a}_t}\|_{\bV_t^{-1}}+g\|\bx_{a_t}\|_{\bV_t^{-1}} \\
\leq\,&
g\|\bx_{\overline{a}_t}\|_{\bV_t^{-1}}+g\|\bx_{\overline{a}_t}\|_{\bV_t^{-1}}
+g\|\bx_{a_t}\|_{\bV_t^{-1}}.
\end{align*}
Therefore, for any $\cF_{t-1}$ such that $E^{\widehat\balpha}(t)$ is true, either
$\Delta_{a_t}\leq2g\|\bx_{\overline{a}_t}\|_{\bV_t^{-1}}+g\|\bx_{a_t}\|_{
\bV_t^{-1}}$, or $E^{\widetilde\balpha}(t)$ is false. 
We can deduce that
\begin{align*}
\E\left[\Delta_{a_t}\,|\,\cF_{t-1}\right]&\leq\E\left[2g\|\bx_{\overline{a}_t}\|_{\bV_t^{-1}}+g\|\bx_{a_t}\|_{\bV_t^{-1}}\,|\,\cF_{t-1}\right] + \P\left(\overline{E^{\widetilde\balpha}(t)}\right) \\
&\leq\frac{2g}{p-\frac{1}{T^2}}\E\left[\|\bx_{a_t}\|_{\bV_t^{-1}}\,|\,\cF_{t-1}\right] +g\E\left[\|\bx_{a_t}\|_{\bV_t^{-1}}\,|\,\cF_{t-1}\right] + \frac{1}{T^2} \\
&\leq\frac{11g}{p}\E\left[\|\bx_{a_t}\|_{\bV_t^{-1}}\,|\,\cF_{t-1}\right]+\frac{1}{T^2}\CommaBin
\end{align*}
where in the last inequality we used that
$1/(p-1/T^2)\leq 5/p,$ 
which holds trivially for $T\leq 4$. For $T\geq 5$, we know that $p = 1/(4e\sqrt{\pi})$ and therefore 
$T^2\geq5e\sqrt{\pi}$, from which we get  that $1/(p-1/T^2)\leq 5/p$ as well.

\end{proof}

\begin{definition}
We define
$\regret'(t) \eqdef \regret(t)\cdot I\left(E^{\widehat{\balpha}}(t)\right)$.
\end{definition}

\begin{definition}
A sequence of random variables $(Y_t;\, t\geq0)$ is called a
\textbf{super-martingale}
corresponding to a filtration $\cF_t$, if for all $t$, $Y_t$ is
$\cF_t$-measurable, and for $t\geq1$,
\[\E\left[Y_t-Y_{t-1}\,|\,\cF_{t-1}\right] \leq 0.\]
\end{definition}
Next, following~\citet{agrawal2013thomson}, we establish a super-martingale
process that forms the basis of our proof of the high-probability regret
bound.

\begin{definition}
Let
\begin{align*}
X_t &\eqdef \regret'(t) - \frac{11g}{p}\|\bx_{a_t}\|_{\bV_t^{-1}}-\frac{1}{T^2}\quad \text{and}\\
Y_t &\eqdef \sum_{w = 1}^t X_w.
\end{align*}
\end{definition}

\begin{lemma}
$(Y_t;\, t = 0,\dots,T)$ is a super-martingale process with respect to filtration~$\cF_t$.
\end{lemma}

\begin{proof}
We need to prove that for all $t \in \{1,\dots,T\}$ and any possible filtration
$\cF_{t-1}$, $\E[Y_t-Y_{t-1}\,|\,\cF_{t-1}] \leq 0$, i.e.,
\[\E\left[\regret'(t)\,|\,\cF_{t-1}\right]\leq\frac{11g}{p}\|\bx_{a_t}\|_{\bV_t^{-1}}+\frac
{1}{T^2}\cdot\]
Note that whether $E^{\widehat\balpha}(t)$ is true or not, is completely determined by
$\cF_{t-1}$. If $\cF_{t-1}$ is such that $E^{\widehat\balpha}(t)$ is not true, then
$\regret'(t)=\regret(t)\cdot I\left(E^{\widehat{\balpha}}(t)\right)=0$, and the above inequality
holds trivially. Moreover, for $\cF_{t-1}$ such that $E^{\widehat\balpha}(t)$ holds, the
inequality follows from Lemma~\ref{lem:regretBound}.\end{proof}
Note that unlike \citet{agrawal2013thomson} and~\cite{abbasi2011improved}, we
do not want to require $\lambda\geq 1$. Therefore, we provide the following
lemma that features the dependence of $\|\bx_{a_t}\|_{\bV_t^{-1}}^2$
on $\lambda$.

\begin{lemma}\label{lem:bbound}
For all $t$,
\[\|\bx_{a_t}\|_{\bV_t^{-1}}^2\leq\left(2+\frac{2}{\lambda}
\right)\log\left(1+\|\bx_{a_t}\|_{\bV_t^{-1}}^2\right)\!.\]
\end{lemma}
\begin{proof}
Note, that
$\|\bx_{a_t}\|_{\bV_t^{-1}}\leq(1/\sqrt{\lambda})\|\bx_{a_t}\|\leq
(1/\sqrt{\lambda})$ and for all $0\leq x\leq1$, we have
\begin{align}
x\leq2\log(1 + x). \label{ineq}
\end{align}
 We now consider two cases depending on $\lambda$. If
$\lambda\geq1$, we know that $0\leq\|\bx_{a_t}\|_{\bV_t^{-1}}\leq1$ and
therefore by \eqref{ineq},
\[\|\bx_{a_t}\|_{\bV_t^{-1}}^2\leq2\log\left(1+\|\bx_{a_t}\|_{\bV_t^{-1}}^2\right)\!.\]
Similarly, if $\lambda<1$, then $0\leq\lambda\|\bx_{a_t}\|^2_{\bV_t^{-1}}\leq 1$
and we get
\[
\|\bx_{a_t}\|_{\bV_t^{-1}}^2\leq\frac{2}{\lambda}\log\left(1+\lambda\|\bx_{a_t}\|_{\bV_t^{-1}}^2\right) \leq\frac{2}{\lambda}\log\left(1+\|\bx_{a_t}\|_{\bV_t^{-1}}^2\right)\!.
\]
Combining the two, we get that for all $\lambda\geq0$,
\[
\|\bx_{a_t}\|_{\bV_t^{-1}}^2\leq\max\left(2,\ \frac{2}{\lambda}\right)\log\left(1+\|\bx_{a_t}\|_{\bV_t^{-1}}^2\right) \leq\left(2+\frac{2}{\lambda}\right)\log\left(1+\|\bx_{a_t}\|_{\bV_t^{-1}}^2\right)\!.
\]
\end{proof}
\begin{proof}[\textbf{Theorem~\ref{thm:spectralts}}]
First, notice that $X_t$ is bounded as  \[|X_t|\leq
1+\frac{11g}{p\sqrt\lambda}+\frac{1}{T^2}\leq \frac{g}{p}\left(\frac{11}{\sqrt{\lambda}}+2\right)\!\cdot\] Thus, we can apply the
Azuma-Hoeffding inequality to obtain that with probability
 at least $1-\delta/2$,
\[
\sum_{t=1}^T\regret'(t)\leq\sum_{t=1}^T\frac{11g}{p}\|\bx_{a_t}\|_{\bV_t^{-1}}+\sum_{t=1}^T\frac{1}{T^2} +\sqrt{2\left(\sum_{t=1}^T\frac{g^2}{p^2}\left(\frac{11}{\sqrt{\lambda}}+2\right)^2\right)\log\left(\frac{2}{\delta}\right)}\cdot
\]
Since $p$ and $g$ are constants, then with
probability $1-\delta/2$,
\begin{align*}
\sum_{t=1}^T\regret'(t)\leq&\frac{11g}{p}\sum_{t=1}^T\|\bx_{a_t}\|_{\bV_t^{-1}}+\frac{1}{T} +\frac{g}{p}\left(\frac{11}{\sqrt{\lambda}}+2\right)\sqrt{2T\log\left(\frac{2}{\delta}\right)}\cdot
\end{align*}
The last step is to upperbound $\sum_{t=1}^T\|\bx_{a_t}\|_{\bV_t^{-1}}$. For
this purpose, \citet{agrawal2013thomson} rely on the analysis of \citet{auer2002using} and the assumption that $\lambda\geq1$. We provide an alternative approach
using Cauchy-Schwartz inequality, Lemma~\ref{lemma:logdetassum}, and
Lemma~\ref{lem:bbound} to get
\begin{align*}
\sum_{t=1}^T\|\bx_{a_t}\|_{\bV_t^{-1}}\leq\sqrt{T\sum_{t=1}^T\|\bx_{a_t}\|^2_
{ \bV_t^{-1}}} 
\leq\sqrt{T\left(2+\frac{2}{\lambda}\right)\log\frac{|\bV_T|}{|\bLambda|}}
\leq\sqrt{\frac{2+2\lambda}{\lambda}dT\log\left(1+\frac{T}{K\lambda}\right)}\cdot
\end{align*}
Finally, we know that $E^{\widehat\balpha}(t)$ holds for all $t$ with probability at
least $1-\delta/2$ and $\regret'(t) = \regret(t)$ for all $t$ with
probability at least $1-\delta/2$. Hence, with probability $1-\delta$,
\begin{align*}
R_T\leq\,&\frac{11g}{p}\sqrt{\frac{2+2\lambda}{\lambda}dT\log\left(1+\frac{T}{K\lambda}\right)}+\frac{1}{T} + \frac{g}{p}\left(\frac{11}{\sqrt{\lambda}}+2\right)\sqrt{2T\log\left(\frac{2}{\delta}\right)}\cdot
\end{align*}\end{proof}
\subsection{Regret bound of \SpectralEliminator}
\label{sec:regretspectraleliminator}
The probability space induced by the rewards $r_1,r_2,\dots$ can be decomposed
as a product of independent probability spaces induces by rewards in each phase
$[t_j, t_{j+1}-1]$. Denote by~$\cF'_j$ the $\sigma$-algebra generated by
the rewards $r_1,\dots,r_{t_{j+1}-1}$, i.e.,~received before and during the
phase $j$.
We have the following two lemmas for any phase $j$. Let $\overline{\bV}_j \eqdef \bLambda + \sum_{s = t_{j-1}}^{t_j-1}\bx_{a_s}^{\phantom{\mathsf{\scriptscriptstyle T}}}\bx_{a_s}\transpose$ and let $\widehat\balpha_j$ stand for $\widehat\balpha_{t_j}$ for simplicity.

\begin{lemma}
\label{lem:3}
For any fixed $\bx\in\R^N$, any $\delta>0$, and
{\color{black} $\beta(\delta) \eqdef R \sqrt{2\log(2/\delta)} + \|\balpha\|_\bLambda  $}, we have 
for all~$j$,
\[\P\left( |\bx\transpose (\widehat\balpha_j-\balpha)| \leq \|\bx\|_{\overline{\bV}_j^{-1}}
\beta(\delta) \right) \geq 1-\delta.\]
\end{lemma}
\begin{proof}
Defining $\bxi_j\eqdef \sum_{s=t_{j-1}}^{t_{j}-1}\bx_{a_s}^{\phantom{\mathsf{\scriptscriptstyle T}}} \varepsilon_s$, we have
\begin{align}
\label{eq:1}
|\bx\transpose (\widehat \balpha_j-\balpha)| &= |\bx\transpose (-\overline{\bV}_j^{-1}
\bLambda \balpha + \overline{\bV}_j^{-1} \bxi_j )| \leq |\bx\transpose \overline{\bV}_j^{-1} \bLambda \balpha| + | \bx\transpose
\overline{\bV}_j^{-1} \bxi_j |.
\end{align}
The first term in the right hand side~of \eqref{eq:1} is bounded as
\begin{eqnarray*}
|\bx\transpose \overline{\bV}_j^{-1} \bLambda \balpha|
&\leq& \|\bx\transpose \overline{\bV}_j^{-1} \bLambda^{1/2}\|  \|\bLambda^{1/2}\balpha\| \\
&=& \|\balpha\|_\bLambda  \sqrt{\bx\transpose \overline{\bV}_j^{-1} \bLambda \overline{\bV}_j^{-1}
\bx} \\
&\leq & \|\balpha\|_\bLambda   \sqrt{\bx\transpose \overline{\bV}_j^{-1} \bx} =
\|\balpha\|_\bLambda \|\bx\|_{\overline{\bV}_j^{-1}}.
\end{eqnarray*}
Now, consider the second term in the right hand side~of \eqref{eq:1}. We have
\[ \left|\bx\transpose \overline{\bV}_j^{-1} \bxi_j\right| = \left|
\sum_{s=t_{j-1}}^{t_{j}-1} (\bx\transpose \overline{\bV}_j^{-1}
\bx_{a_s})\varepsilon_s\right|.\]
Let us notice that the context vectors $(\bx_{a_s})$ selected by the algorithm during phase
$j-1$ only depend on their width $\|\bx\|_{\bV_{s}^{-1}}$, which does not depend
on the rewards received during the phase $j-1$. Thus, given $\cF'_{j-2}$, the
values $\bx\transpose\overline{\bV}_j^{-1}\bx_{a_s}$ are deterministic for all rounds $t_{j-1}\leq s< t_{j}$. Consequently, we can use a variant of Hoeffding bound for \emph{scaled} sub-Gaussians \citep{wainwright2015stat210b}, in particular for
$\bx\transpose \overline{\bV}_j^{-1}\bxi_j = \sum _{s = t_{j-1}}^{t_j-1}\bx\transpose\overline{\bV}_j^{-1}\bx_{a_s}\varepsilon_s $, to get
\[
\P\left(\left|\bx\transpose \overline{\bV}_j^{-1}\bxi_j\right| \leq R\sqrt{2\log\left(\frac{2}{\delta}\right)\sum_{s=t_{j-1}}^{t_{j}-1}\left(\bx\transpose \overline{\bV}_j^{-1}\bx_{a_s}\right)^2}\right)\ge 1-\delta,
\]
where $\varepsilon_s$ is  $R$-sub-Gaussian  and $\bx\transpose\overline{\bV}_j^{-1}\bx_{a_s}$ is deterministic given $\cF'_{j-2}$. We further deduce
\begin{align*}
\P\left(\left|\bx\transpose \overline{\bV}_j^{-1}\bxi_j\right| \leq R\sqrt{2\log\left(\frac{2}{\delta}\right)\sum_{s=t_{j-1}}^{t_{j}-1}\left(\bx\transpose \overline{\bV}_j^{-1}\bx_{a_s}\bx_{a_s}\transpose\overline{\bV}_j^{-1}\bx \right)}\right)&\ge 1-\delta \\
\P\left(\left|\bx\transpose \overline{\bV}_j^{-1}\bxi_j\right| \leq R\sqrt{2\log\left(\frac{2}{\delta}\right)\bx\transpose \overline{\bV}_j^{-1}\left(\sum_{s=t_{j-1}}^{t_{j}-1}\bx_{a_s}\bx_{a_s}\transpose\right)\overline{\bV}_j^{-1}\bx }\right)&\ge 1-\delta
\\
\P\left(\left|\bx\transpose \overline{\bV}_j^{-1}\bxi_j\right| \leq R\sqrt{2\log\left(\frac{2}{\delta}\right)\bx\transpose \overline{\bV}_j^{-1}\bx }\right)&\ge 1-\delta,
\end{align*}
since $\overline{\bV}_j^{-1}$ is symmetric and $\sum_{s=t_{j-1}}^{t_{j}-1} \bx_s \bx_s\transpose \prec \overline{\bV}_j$ (Lemma~\ref{lem:updateInverse}). We conclude that
\[
\P\left(\left|\bx\transpose \overline{\bV}_j^{-1}\bxi_j\right| \leq R\|\bx\|_{\overline{\bV}_j^{-1}}\sqrt{2\log\left(\frac{2}{\delta}\right)}\right)\ge 1-\delta.\]\end{proof}
\begin{lemma}For all $\bx\in A_j$, $j>1$, we have that
\[\min\left(1,\,\|\bx\|_{\overline{\bV}_{j}^{-1}}\right) \leq \frac{1}{t_{j}-t_{j-1}} \sum_{s=t_{j-1}}^{t_{j}-1}
\min\left(1,\,\|\bx_{a_s}\|_{\bV_{s}^{-1}}\right)\!.\]
\end{lemma}

\begin{proof}
Using Lemma \ref{lem:updateInverse}, we have that
\begin{align*}
(t_{j}-t_{j-1}) \min\left(1,\,\|\bx\|_{\overline{\bV}_{j}^{-1}}\right) & \leq  \max_{\bx\in A_j} \sum_{s=t_{j-1}}^{t_{j}-1} \min\left(1,\,\|\bx\|_{\bV_{s}^{-1}}\right)\\
& \leq  \max_{\bx\in A_{j-1}} \sum_{s=t_{j-1}}^{t_{j}-1} \min\left(1,\,\|\bx\|_{\bV_{s}^{-1}}\right)\\
&\leq \sum_{s=t_{j-1}}^{t_{j}-1} \min\left(1,\,\max_{\bx\in A_{j-1}}  \|\bx\|_{\bV_{s}^{-1}}\right)\\
&= \sum_{s=t_{j-1}}^{t_j-1}  \min\left(1,\,\|\bx_{a_s}\|_{\bV_{s}^{-1}}\right)\!,
\end{align*}
since the algorithm selects (during phase $j-1$) the arms with the largest width.
\end{proof}
We now are ready to upperbound the cumulative regret
of \SpectralEliminator.

\begin{proof}[\textbf{Theorem~\ref{thm:eliminator}}]
Let $J \eqdef \lfloor \log_2 T\rfloor + 1$ and $t_j \eqdef 2^{j-1}$. We have that
\begin{align*}
R_T &= \sum_{t=1}^T  \bx_{a^\star}\transpose\balpha-\bx_{a_t}\transpose\balpha\leq 2 +  \sum_{j=2}^{J } \sum_{t=t_j}^{t_{j+1}-1} \min(2,\,\bx_{a^\star}\transpose\balpha-\bx_{a_t}\transpose\balpha ) \\
&\leq 2+\sum_{j=2}^{J } \sum_{t=t_j}^{t_{j+1}-1}\min\left(2,\,\bx_{a^\star}\transpose\widehat\balpha_j-\bx_{a_t}\transpose \widehat\balpha_j  +\left(\|\bx_\star\|_{\overline{\bV}_{j}^{-1}}+\|\bx_t\|_{\overline{\bV}_{j}^{-1}}\right)\beta(\delta')\right)\!,
\end{align*}
in an event $\omega$ of probability $1-\delta$, where we used Lemma~\ref{lem:3} and the union bound in the last inequality for $\delta' \eqdef \delta/(KJ)$. By definition of the action subset $A_j$ at phase $j>1$, under~$\omega$, we have that
\[
\bx_{a^\star}\transpose\widehat\balpha_j-\bx_{a_t}\widehat\balpha_j  \leq \left(\|\bx_{a^\star}\|_{\overline{\bV}_j^{-1}}+\|\bx_{a_t}\|_{\overline{\bV}_j^{-1}}\right)\beta(\delta'),
\] 
since $\bx_{a^\star}\in A_j$ for all $j\leq J$. By previous two lemmas and the Cauchy-Schwarz inequality,

\begin{align*}
R_T&\leq  2 + \sum_{j = 2}^J\sum_{t = t_j}^{t_{j+1}-1}\min\left(2,\, 4 \beta(\delta')\|\bx_{a_t}\|_{\overline{\bV}_{j}^{-1}}\right)  \\
&\leq 2 + (4\beta(\delta')+2)\sum_{j = 2}^J \sum_{t = t_j}^{t_{j+1}-1}\min\left(1,\,\|\bx_{a_t}\|_{\bV_t^{-1}}\right)  \\
&\leq 2 + (4\beta(\delta')+2)\sum_{j = 2}^J \frac{t_{j+1}-t_j}{t_j-t_{j-1}}\sum_{t = t_{j-1}}^{t_{j}-1}\min\left(1,\,\|\bx_{a_t}\|_{\bV_t^{-1}}\right)  \\
&\leq 2 + (8\beta(\delta')+4)\sum_{j = 2}^J \sum_{t = t_{j-1}}^{t_{j}-1}\min\left(1,\,\|\bx_{a_t}\|_{\bV_t^{-1}}\right)  \\
&\leq 2 + (8\beta(\delta')+4)\sqrt{T\sum_{j = 2}^J \sum_{t = t_{j-1}}^{t_{j}-1}\min\left(1,\,\|\bx_{a_t}\|_{\bV_t^{-1}}^2\right)}\\
&\leq 2 + (8\beta(\delta')+4)\sqrt{T\sum_{j = 2}^J 2\log\frac{|\overline{\bV}_j|}{|\bLambda|}}\\
&\leq 2 + (8\beta(\delta')+4)\sqrt{2dT\log_2(T) \log\left(1 + \frac{T}{K\lambda}\right)}\cdot
\end{align*}
Finally, using $J = 1 + \lfloor \log_2T\rfloor$, $\delta'=\delta/(KJ)$, and $\beta(\delta')\leq\beta(\delta/(K(1+\log_2T)))$, we obtain the result of Theorem~\ref{thm:eliminator}.
\end{proof}
\begin{remark}
If we set $\bLambda = \bI$ in~Algorithm~\ref{alg:TransductiveElimination} 
 as in Remark~\ref{rem:linucb}, we get a new algorithm,
\LinearEliminator\/,
which is a competitor to \SupLinRel~\citep{auer2002using} and \SupLinUCB \citep{chu2011contextual}
and as a corollary to Theorem~\ref{thm:eliminator}
also enjoys an $\widetilde\cO(\sqrt{DT})$ upper bound on the cumulative regret.
Compared to \SupLinRel and \SupLinUCB, \LinearEliminator and its analysis
are significantly  simpler and more elegant. Furthermore, \LinearEliminator is more data-adaptive
since it uses self-normalized concentration bounds rather than
data-agnostic confidence intervals of the form $2^{-u}$ for $u \in \NN_0$,
which are used in \SupLinRel and \SupLinUCB. Therefore, 
\LinearEliminator narrows the gap between the practical algorithms
and the algorithms with the optimal cumulative regret of $\widetilde\cO(\sqrt{DT}).$

\end{remark}
\renewcommand{\arraystretch}{1.2}
\section{Experiments}
\begin{figure}[t]
\centering
\begin{subfigure}{0.45\textwidth}
\includegraphics[width = \textwidth]{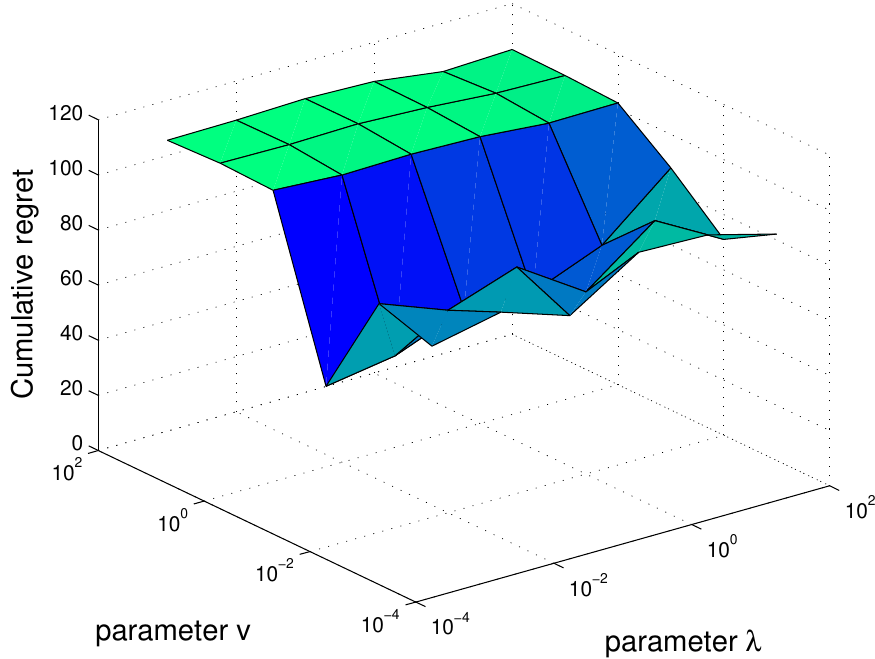}
\caption{\SpectralTS}
\end{subfigure}
\begin{subfigure}{0.45\textwidth}
\includegraphics[width = \textwidth]{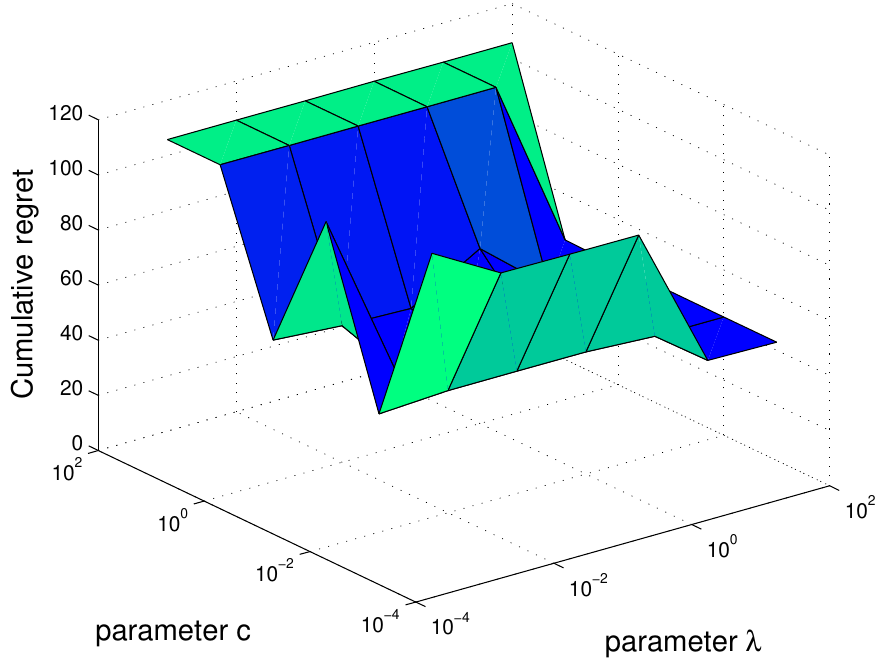}
\caption{\SpectralUCB}
\end{subfigure}
\begin{subfigure}{0.45\textwidth}
\includegraphics[width = \textwidth]{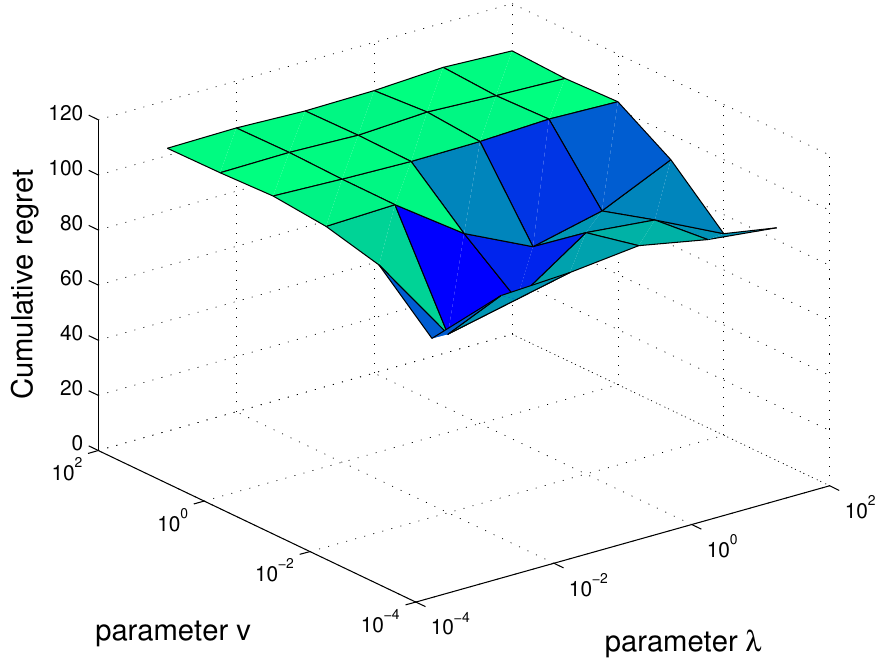}
\caption{\LinearTS}
\end{subfigure}
\begin{subfigure}{0.45\textwidth}
\includegraphics[width = \textwidth]{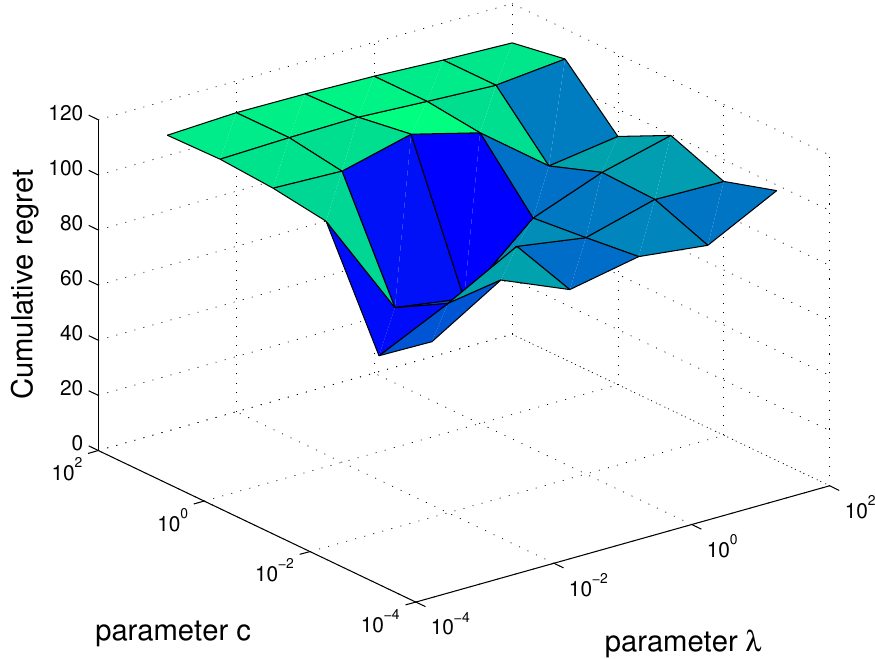}
\caption{\LinUCB}
\end{subfigure}
\caption{The dependence of cumulative regret on confidence and regularization parameters.}
\label{fig:parameters_regret}
\end{figure}

\label{sec:exp}
In this section, we evaluate the empirical regret as well the empirical computational complexity of \SpectralTS, \SpectralUCB, \LinearTS, and \LinUCB  on artificial datasets with different types of underlying graph structure as well as on MovieLens and Flixster datasets. We do not include \SpectralEliminator in our experiments due to its impracticality for small time horizons.\footnote{since the algorithm updates confidence ellipsoid only at the end of the phase} 
We study the sensitivity of the algorithms to the important parameters and comment on practical issues.
Moreover, we study the effects of different speed-up techniques. In particular, we show the effect of the reduced basis on both the computational complexity and performance of the algorithms and the effect of Sherman-Morrison (the computation of matrix inversions) together with lazy updates (the computation of UCBs) on the running time. In all experiments, we set both the confidence parameter $\delta$, use the uniformly distributed noise satisfying  $R \leq 0.05$, and average over 5 runs. We performed but do not include the results for different values of $\delta$ and $R$ since the results of the experiments are not sensitive to the values of these parameters and follow the same trend.

\subsection{Artificial datasets}
To demonstrate the benefit of spectral algorithms, we perform exhaustive experiments on artificial datasets with various underlying graphs. More precisely, we focus on problems where underlying graphs form a lattice or they are sampled either from the Barab\'asi-Albert (BA) or Erd\H os-R\'enyi (ER) graph model. For all experiments on artificial datasets, we set the number of arms $N$ to 500 and the time horizon $T$ to 100. We sample a random vector~$\balpha$~such that reward function $\bff \eqdef \bQ\balpha$ is smooth on the graph. We do it by settings only the first 20 elements of $\balpha$ to be nonzero. For a more useful empirical comparison, we set the regularization parameter $\lambda$ and confidence ellipsoid parameters $v$ (\TS) and $c$ (\UCB) respectively to the best empirical value over a grid search. We run the algorithms with several different values and select the values which minimized average cumulative regret after a few runs of algorithms. Figure \ref{fig:parameters_regret} shows the dependence of cumulative regret on parameters with strong indication that \SpectralTS and \SpectralUCB can leverage smoothness of the reward function and outperform \LinearTS and \LinUCB.

\subsubsection{Erd\H os-R\'enyi graphs}
For this experiment, we construct the underlying graph as an Erd\H os-R\'enyi graph on~$500$ nodes with parameter $0.005$ (the probability of edge appearance). The values of the parameters used for the experiment are listed in Table \ref{tab:er_parameters}, which are the values where the algorithms perform the best.

Figure \ref{fig:artifitial_er_regrets} shows the cumulative regrets of the algorithms with selected parameters. The regret of spectral algorithms tends to be sublinear while regret of linear algorithms appears to be linear for small $T$. Moreover, spectral algorithms reach much smaller empirical regrets than their linear counterparts.
\begin{table}[h]
\begin{center}
\begin{tabular}{| l | l | l | l | l | l | l | l |}
\hline
\multicolumn{2}{ |c| }{\SpectralTS} & \multicolumn{2}{ |c| }{\SpectralUCB} & \multicolumn{2}{ |c| }{\LinearTS} & \multicolumn{2}{ |c| }{\LinUCB} \\[.1em]
\hline
$\lambda = 0.1$ & $v = 0.1$ & $\lambda = 1$ & $c = 1$ & $\lambda = 1$ & $v = 0.1$ & $\lambda = 0.1$ & $c = 0.1$ \\
\hline
\end{tabular}
\caption{The best-performing empirical parameters for the Erd\H os-R\'enyi graph model.}
\label{tab:er_parameters}
\end{center}
\end{table}
\begin{figure}
\centering
\begin{subfigure}{0.45\textwidth}
\centering
\includegraphics[width = \textwidth]{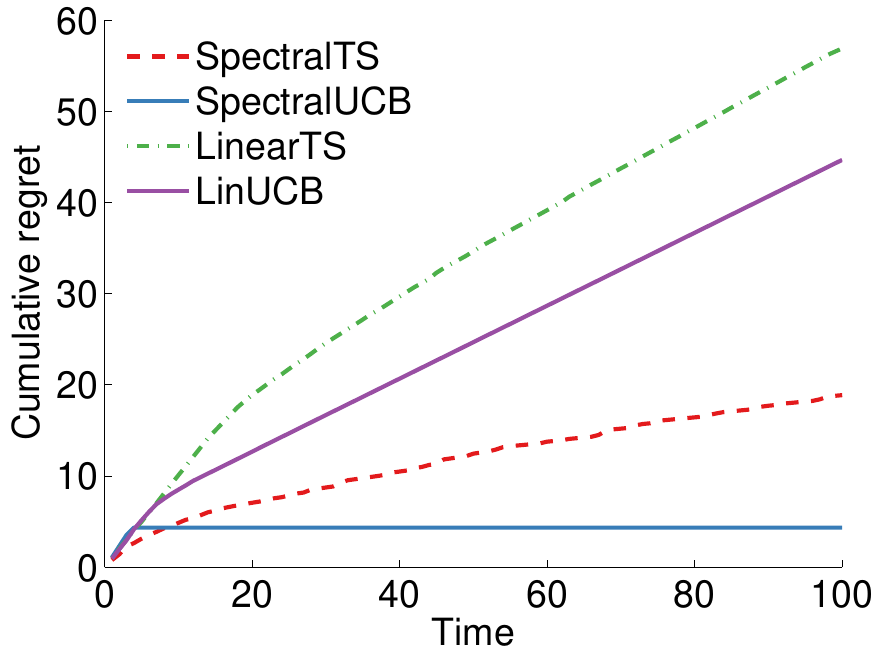}
\caption{Erd\H os-R\'enyi graph }
\label{fig:artifitial_er_regrets}
\end{subfigure}
\quad
\begin{subfigure}{0.45\textwidth}
\centering
\includegraphics[width = \textwidth]{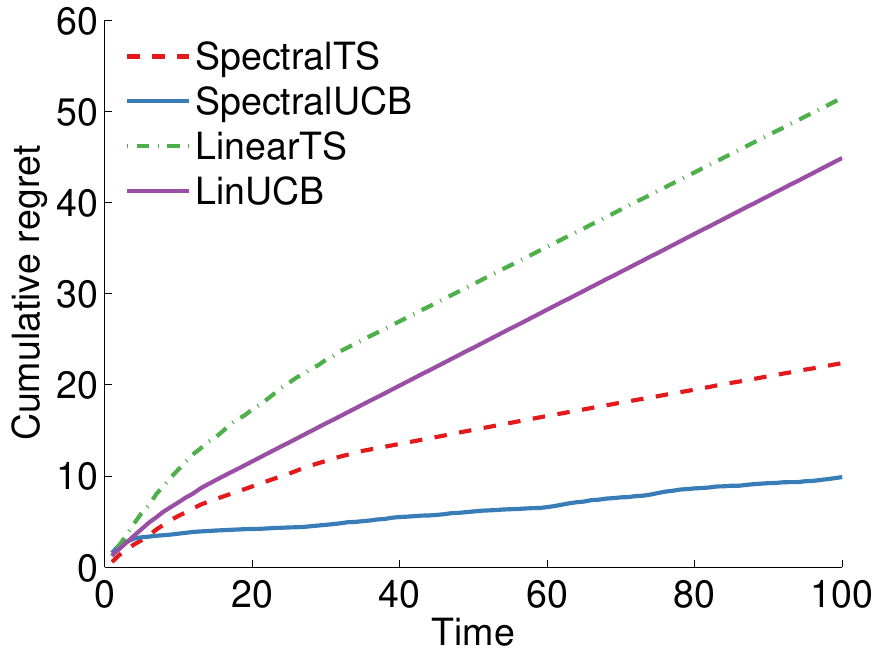}
\caption{lattice}
\label{fig:artifitial_lattice_regrets}
\end{subfigure}
\begin{subfigure}{0.45\textwidth}
\centering
\includegraphics[width = \textwidth]{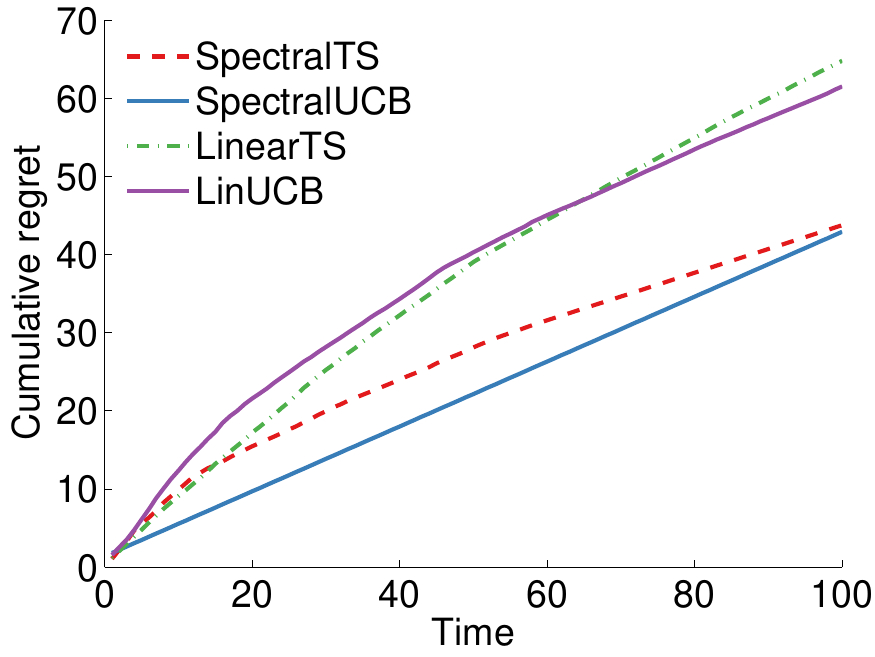}
\caption{Barab\'asi-Albert graph}
\label{fig:artifitial_ba_regrets}
\medskip
\end{subfigure}
\quad
\begin{subfigure}{0.45\textwidth}
\centering
\includegraphics[width = \textwidth]{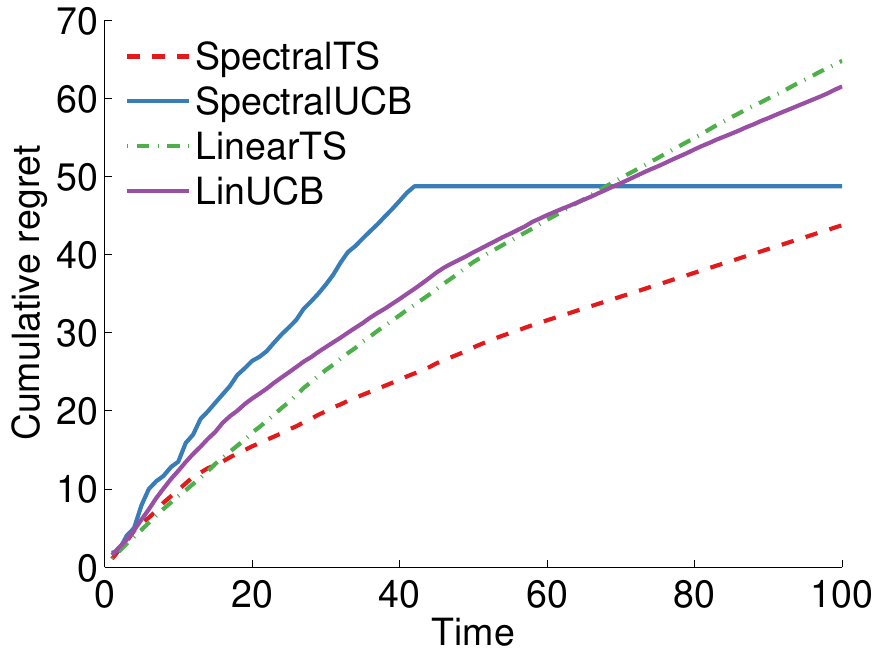}
\caption{Barab\'asi-Albert graph with suboptimal parameters for \SpectralUCB}
\label{fig:artifitial_ba_regrets_suboptimal}
\medskip
\end{subfigure}
\caption{Cumulative regret comparison of algorithms for different underlying graphs.}
\end{figure}
\begin{figure}
\centering
\begin{subfigure}{0.45\textwidth}
\includegraphics[width = \textwidth]{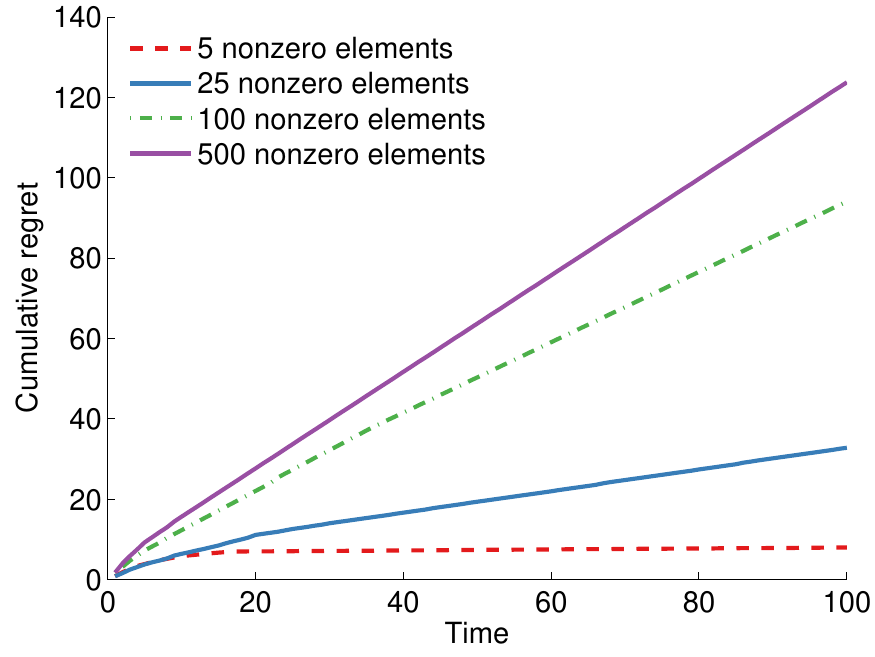}
\caption{\SpectralTS}
\label{fig:increased_smoothness_ts}
\end{subfigure}
\begin{subfigure}{0.45\textwidth}
\includegraphics[width = \textwidth]{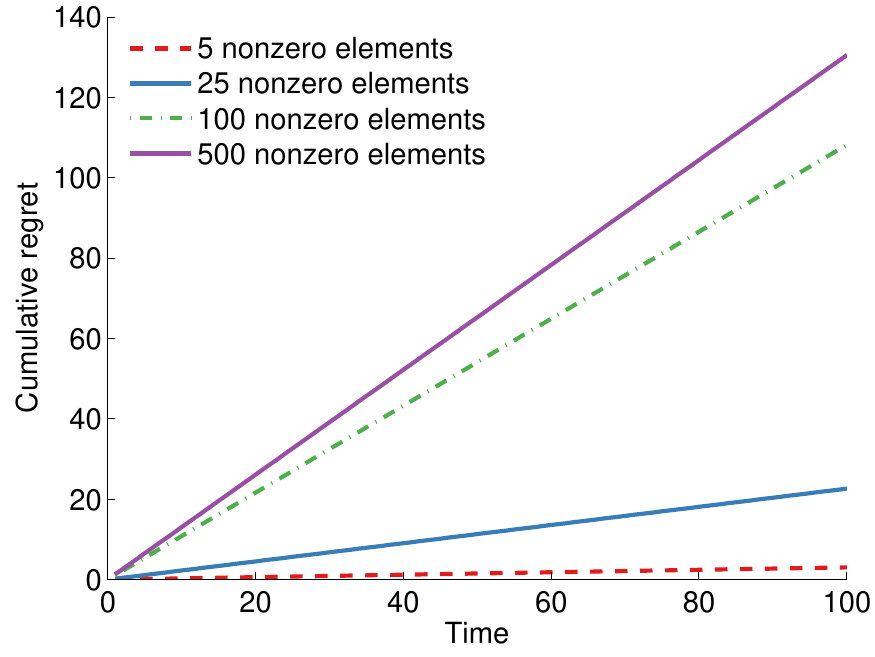}
\caption{\SpectralUCB}
\label{fig:increased_smoothness_ucb}
\end{subfigure}
\caption{Cumulative regret of \SpectralTS and \SpectralUCB for reward functions with different smoothness.}
\end{figure}
\subsubsection{Lattice graphs}
For lattices, we arrange 500 nodes to form a lattice and connect every pair of nodes by an edge if they are neighbors in the lattice. As  for the other experiments, we  empirically select the best set of parameters (Table \ref{tab:lat_parameters}) and use them to plot the cumulative regret of algorithms (Figure \ref{fig:artifitial_lattice_regrets}). Even in this case, spectral algorithms perform well compared to the linear ones.
\begin{table}[h]
\begin{center}
\begin{tabular}{| l | l | l | l | l | l | l | l |}
\hline
\multicolumn{2}{ |c| }{\SpectralTS} & \multicolumn{2}{ |c| }{\SpectralUCB} & \multicolumn{2}{ |c| }{\LinearTS} & \multicolumn{2}{ |c| }{\LinUCB} \\[.1em]
\hline
$\lambda = 0.01$ & $v = 0.1$ & $\lambda = 0.1$ & $c = 1$ & $\lambda = 1$ & $v = 0.1$ & $\lambda = 0.1$ & $c = 0.1$ \\
\hline
\end{tabular}
\caption{The best-performing empirical parameters for lattices.}
\label{tab:lat_parameters}
\end{center}
\end{table}
\subsubsection{Barab\'asi-Albert graphs}
We construct the BA graph for our experiments in the following way. We start with~$k$ nodes ($k=3$ in our case) without any connections between them. Then, we sequentially add one node  at a time. Each new node is connected to $m\leq k$ previously added nodes and we sampled the connections according to the degrees of existing nodes: the higher the degree, the bigger the chance of the connection.
Table~\ref{tab:ba_parameters} summarizes the best empirical values of the parameters for the algorithms and 

Figure \ref{fig:artifitial_ba_regrets}
shows the performance of algorithms for the parameters in Table \ref{tab:ba_parameters}. Here we can clearly see that the spectral algorithms outperform the linear ones after just a few rounds. Note that the empirically optimal parameters can sometimes be too aggressive and force an algorithm to exploit more than it should. This is likely the case of \SpectralUCB  in Figure~\ref{fig:artifitial_ba_regrets} since the curve of the cumulative regret of \SpectralUCB appears to be linear for the time horizon used in our experiment. Therefore, we include Figure \ref{fig:artifitial_ba_regrets_suboptimal}, where we plot the cumulative regret of \SpectralUCB for an empirically suboptimal value of $c = 1$ (close to the best theoretical value of $c$) to demonstrate the sublinear trend of the regret. 
\begin{table}[h]
\begin{center}
\begin{tabular}{| l | l | l | l | l | l | l | l |}
\hline
\multicolumn{2}{ |c| }{\SpectralTS} & \multicolumn{2}{ |c| }{\SpectralUCB} & \multicolumn{2}{ |c| }{\LinearTS} & \multicolumn{2}{ |c| }{\LinUCB} \\[.1em]
\hline
$\lambda = 0.001$ & $v = 0.1$ & $\lambda = 0.001$ & $c = 0.01$ & $\lambda = 0.01$ & $v = 0.01$ & $\lambda = 0.1$ & $c = 0.1$ \\
\hline
\end{tabular}
\caption{The best-performing empirical parameters for the Barab\'asi-Albert graph model.}
\label{tab:ba_parameters}
\end{center}
\end{table}


\subsection{The effect of smoothness on the regret}
In this section, we study the  effect of the smoothness of the reward function on the performance of spectral algorithms. We use a BA graph on 500 nodes for the experiment with time horizon 100 and the parameters of the algorithms are set according to table \ref{tab:ba_parameters}. The value of effective dimension is close to 8. We control the smoothness by explicitly setting the number of eigenvectors used for constructing the reward function by letting 5, 25, 100, or 500 elements of $\balpha$ to be nonzero. Note that the value of the effective dimension is the same for every reward function we used, since the definition of the effective dimension is independent of the reward function. Table \ref{tab:smoothness_vs_regret} shows how the smoothness changes with the number of nonzero elements of $\balpha$ and Figures~\ref{fig:increased_smoothness_ts} and~\ref{fig:increased_smoothness_ucb} confirm that the spectral algorithms are able to leverage spectral properties of underlying graph better when the reward function is smoother. This is also supported by our analysis, since in our experiment, the smoothness of the reward function decreases with a smaller number of eigenvectors and the regret bounds of the spectral algorithms are decreasing with smoothness as well.
\begin{table}[h]
\begin{center}
\begin{tabular}{| l | r | r | r | r |}
\hline
\textbf{Number of nonzero components} &\textbf{5} & \textbf{25} & \textbf{100} & \textbf{500}\\
\hline
Smoothness of the rewards ($\balpha\transpose\bLambda\balpha$) & $1.56$ & $11.16$ & $58.12$ & $216.89$\\
\hline
Regret of \SpectralTS & $7.99$ & $32.80$ & $94.10$ & $123.79$\\
\hline
Regret of \SpectralUCB & $3.05$ & $22.84$ & $108.19$ & $130.54$\\
\hline
\end{tabular}
\caption{The effect of smoothness on the cumulative regret for $T=100$.}
\label{tab:smoothness_vs_regret}
\end{center}
\end{table}

\subsection{Computational complexity improvements}\label{ssec:computationImprovements}
In general, the computation of $N$ UCBs is computationally more expensive than sampling in \TS. In Section \ref{ssec:scalability}, we discuss several possibilities to speed up algorithms. The impact of lazy updates for computing UCBs and effect of Sherman-Morrison formula on matrix inversion is demonstrated in Figure \ref{fig:comp_time}. The plot clearly shows that the lazy updates can improve the computation of UCBs to the point where the running time of \SpectralUCB is comparable and in some cases even better than the running time of \SpectralTS.
\begin{figure}
\centering
\includegraphics[width = .5\textwidth]{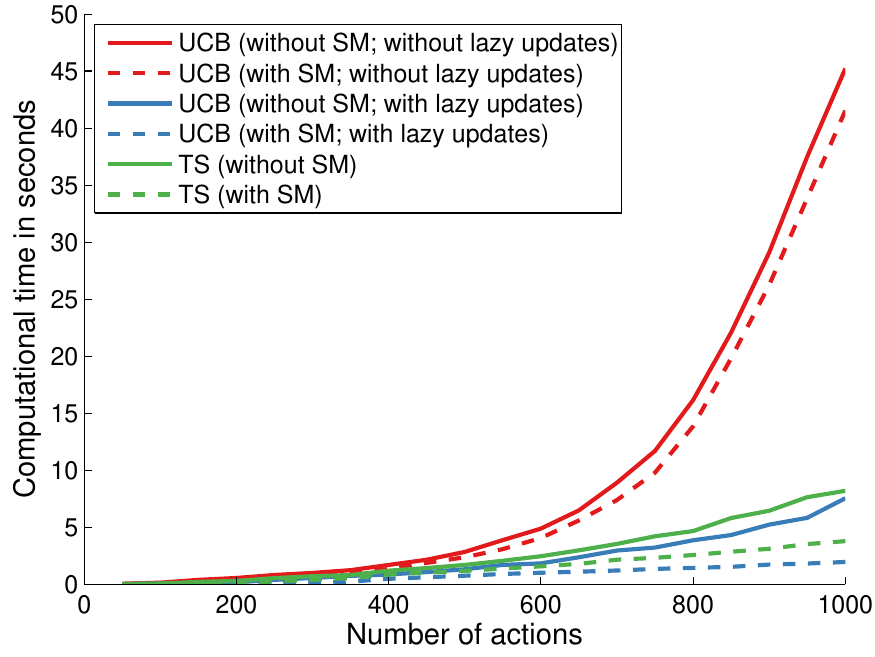}
\caption{The impact of lazy updates and Sherman-Morrison formula on running time.}
\label{fig:comp_time}
\end{figure}

\noindent
Another possible computational-time improvement, discussed in Section~\ref{ssec:scalability}, can be made by extracting only the first $L\ll N$ eigenvectors
of the graph Laplacian. First, the computational
complexity of such operation is $\cO(L m \log m)$
where $m$ is the number of edges. Second, the least-squares problem that we have to do in each round of the
algorithm is only $L$-dimensional. In Figure~\ref{fig:reduced} (right), we plot
the cumulative regret and the total
running time in seconds (log scale) of \SpectralUCB  for a single user from the MovieLens
dataset. We vary~$L$ as 20, 200, and 2000
 which corresponds to about $1\%$,
$10\%$, and $100\%$ of basis functions ($N=2019$).
The total running time also includes
the computational savings from
lazy updates and iterative matrix inversion.
We see that with $10\%$ of the eigenvectors, we  achieve a
similar performance as for the full set in the fraction of the running
time.

\begin{figure}
 \begin{center}
\belowbaseline[0pt]{\includegraphics[width=0.49\columnwidth]{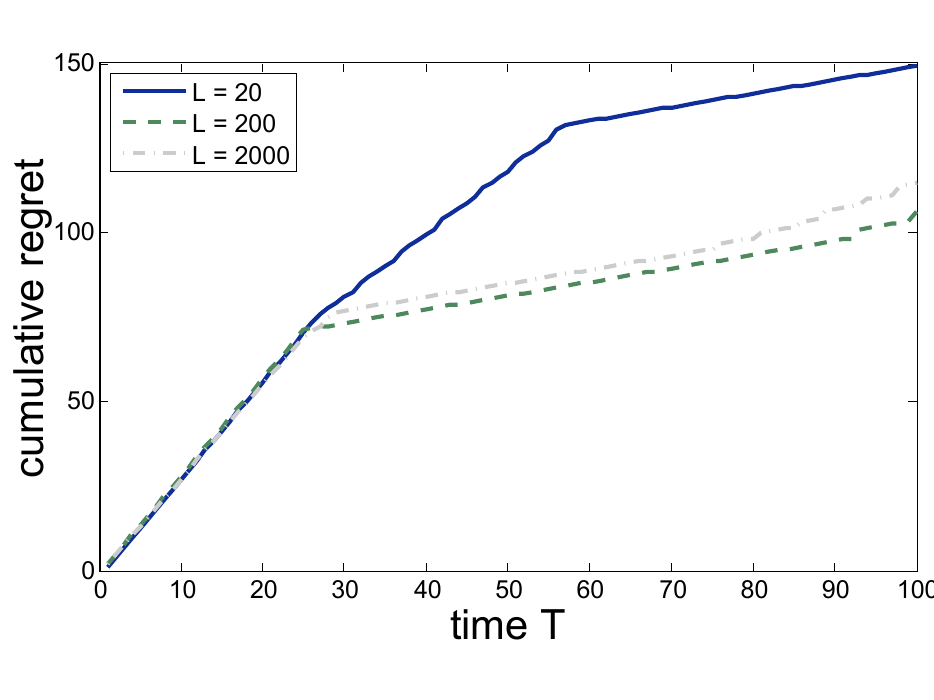}}
\belowbaseline[8pt]{\includegraphics[width=0.41\columnwidth]{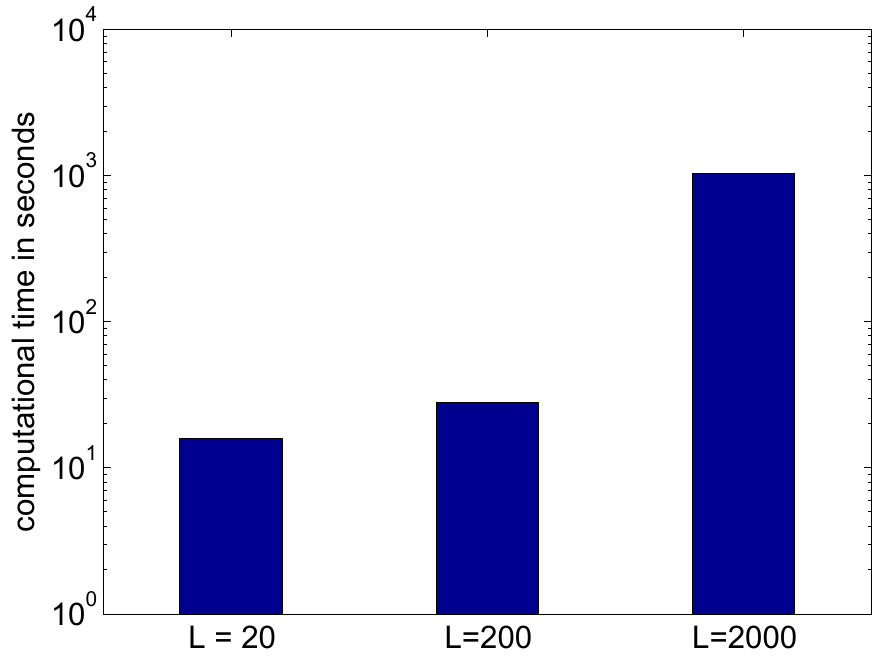}}
\caption{Cumulative regret and running time of \SpectralUCB with reduced basis.}
\label{fig:reduced}
 \end{center}
 \end{figure}


\subsection{MovieLens experiments}
\label{ssec:movielens}

In this experiment, we take user preferences and the similarity graph over
movies from the MovieLens dataset \citep{movielens}, a dataset of
6k users who rated one million movies. First, we extract a subset of 400 users 
and 618 movies with at least 500 ratings. Then we divide the dataset into three
 parts. The first is used to build our model of users, the
rating that user $i$ assigns to movie $j$. We factor the user-item matrix using
low-rank matrix factorization \citep{keshavan2009matrix} as $\bM \approx \bU
\bV'$,
a standard approach to collaborative filtering. The rating that the user $i$ assigns to
movie $j$ is estimated as $\widehat{r}_{i, j} = \langle\mathbf{u}_i, \mathbf{v}_j\rangle$,
where~$\mathbf{u}_i$ is the $i$-th row of $\bU$ and $\mathbf{v}_j$ is the $j$-th row of $\bV$. The rating $\widehat{r}_{i, j}$ is the payoff of pulling arm $j$ when
recommending to user $i$. 

The second part of the dataset is used  for parameter estimation. Similarly as for the 
 first part, we complete ratings using low-rank factorization. The reason for using a different part of the dataset is to avoid dependencies.

The last part of the dataset is used to build our similarity
graph over movies. We factor the dataset in the same way as the first two parts of the dataset.
The graph contains an edge between movies $i$ and $i'$ if the movie $i'$ is
among 5 nearest neighbors of the
movie $i$ in the latent space of items $\bV$. The weights on all edges are set to one.
Notice that if two items are close in the item space, then their expected
rating is expected to be similar. However, the opposite is not true. If two
items have a similar expected rating, they do not have to be close in the item
space. In other words, we take advantage of ratings  but do
not hardwire the two similarly-rated items to be similar.

Table~\ref{tab:movielens_parameters} summarizes the best parameters learned on training part of the dataset.
We use the parameters to run the algorithms on test part. Figure~\ref{fig:movielensdataset} shows
20 random users sampled from the testing part of the MovieLens dataset. We evaluate the regret of all four 
algorithms for $T=500$ and compared the running time. We make a few 
observations. First, spectral algorithms are consistently outperforming linear algorithms. Second, as we mention in 
Section \ref{ssec:scalability}, we use lazy updates for \UCB algorithms which can improve the running time significantly.
We  see that in our experiment, the running time of \UCB algorithms is better than the running time of \TS
algorithms even though in general, \TS algorithms are computationally more efficient than \UCB algorithms without lazy updates.

\begin{table}[h]
\begin{center}
\begin{tabular}{| l | l | l | l | l | l | l | l |}
\hline
\multicolumn{2}{ |c| }{\SpectralTS} & \multicolumn{2}{ |c| }{\SpectralUCB} & \multicolumn{2}{ |c| }{\LinearTS} & \multicolumn{2}{ |c| }{\LinUCB} \\[.1em]
\hline
$\lambda = 0.001$ & $v = 0.1$ & $\lambda = 0.1$ & $c = 1$ & $\lambda = 100$ & $v = 1$ & $\lambda = 0.001$ & $c = 0.001$ \\
\hline
\end{tabular}
\caption{The best-performing empirical parameters for Movielens.}
\label{tab:movielens_parameters}
\end{center}
\end{table}

\subsection{Flixster experiments}
\label{sec:flixster}

We also perform experiments on users preferences from the movie
recommendation
website Flixster. The social network of the users was crawled by
\citet{jamali2010matrix} and then clustered by \citet{graclus} to obtain a
strongly connected subgraph. Similarly as for Movielens, we extract a subset of users and movies, where
each movie has at least 500 ratings. This
results in a dataset of 972 movies and 1070 users.
As with MovieLens, we complete the missing
ratings by a low-rank matrix factorization and used it to construct a 5-NN
similarity graph. For Figure~\ref{fig:flixsterdataset}, we
sample 20 random users and evaluate the regret of all four algorithms for $T=50$.
Similarly as for MovieLens, we set parameter $\lambda$ to $0.01$ 
while setting the parameter $v$ of \SpectralTS to be ten times smaller than the theoretical value.

\begin{table}[h]
\begin{center}
\begin{tabular}{| l | l | l | l | l | l | l | l |}
\hline
\multicolumn{2}{ |c| }{\SpectralTS} & \multicolumn{2}{ |c| }{\SpectralUCB} & \multicolumn{2}{ |c| }{\LinearTS} & \multicolumn{2}{ |c| }{\LinUCB} \\[.1em]
\hline
$\lambda = 0.01$ & $v = 0.1$ & $\lambda = 0.01$ & $c = 0.11$ & $\lambda = 1$ & $v = 0.1$ & $\lambda = 1$ & $c = 1$ \\
\hline
\end{tabular}
\caption{The best-performing empirical parameters for Flixster.}
\label{tab:flixster_parameters}
\end{center}
\end{table}

\begin{figure}[t]
\centering
\begin{subfigure}{0.45\textwidth}
\includegraphics[width=\columnwidth]{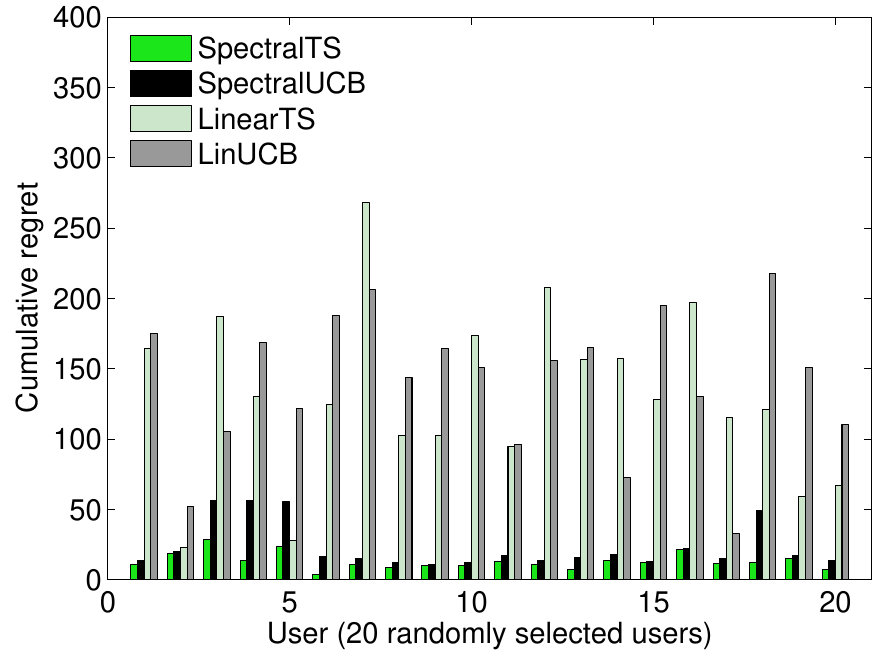}
\caption{Movielens dataset, cumulative regret for 20 randomly selected users}
\label{fig:movielensdataset}
\end{subfigure}
\quad
\begin{subfigure}{0.45\textwidth}
\includegraphics[width=\columnwidth]{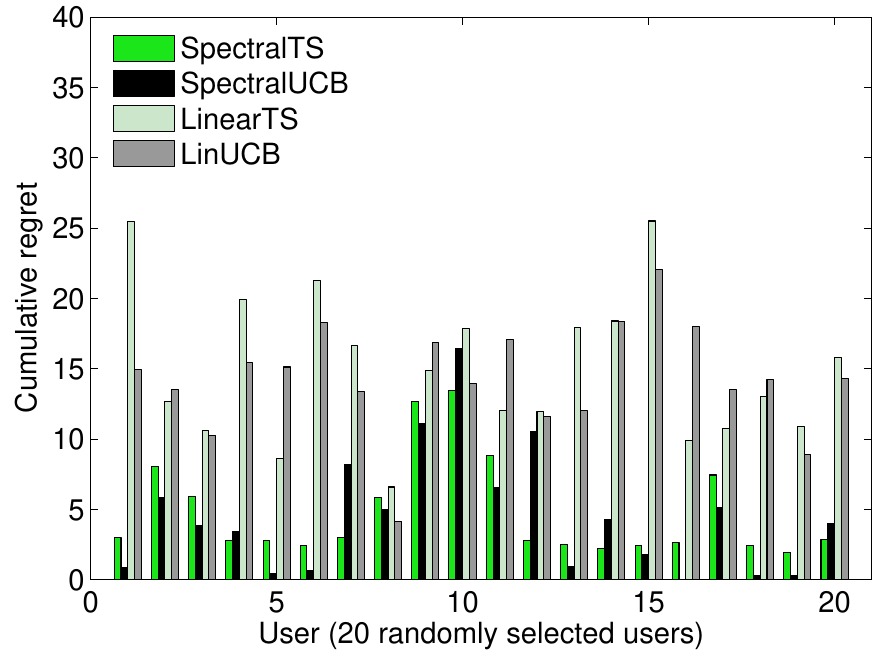}
\caption{Flixster dataset, cumulative regret for 20 randomly selected users}
\label{fig:flixsterdataset}
\end{subfigure}
\caption{Comparison of spectral and linear bandit algorithms for a subset of users.}
\end{figure}

\begin{figure}[t]
\centering
\begin{subfigure}{0.45\textwidth}
\includegraphics[width=\columnwidth]{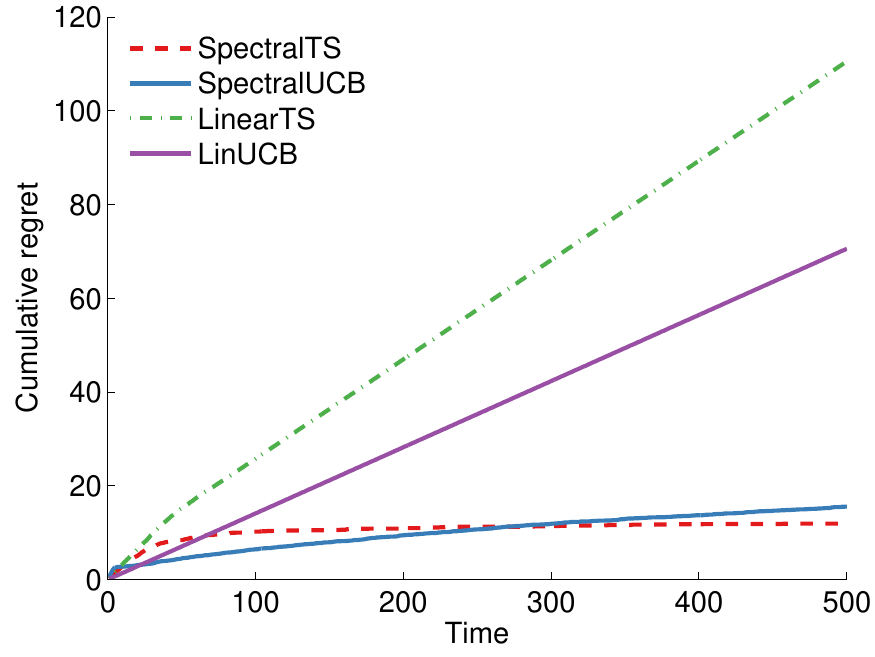}
\caption{Movielens dataset, cumulative regret for one random user}
\label{fig:movielensdatasetsingle}
\end{subfigure}
\quad
\begin{subfigure}{0.45\textwidth}
\includegraphics[width=\columnwidth]{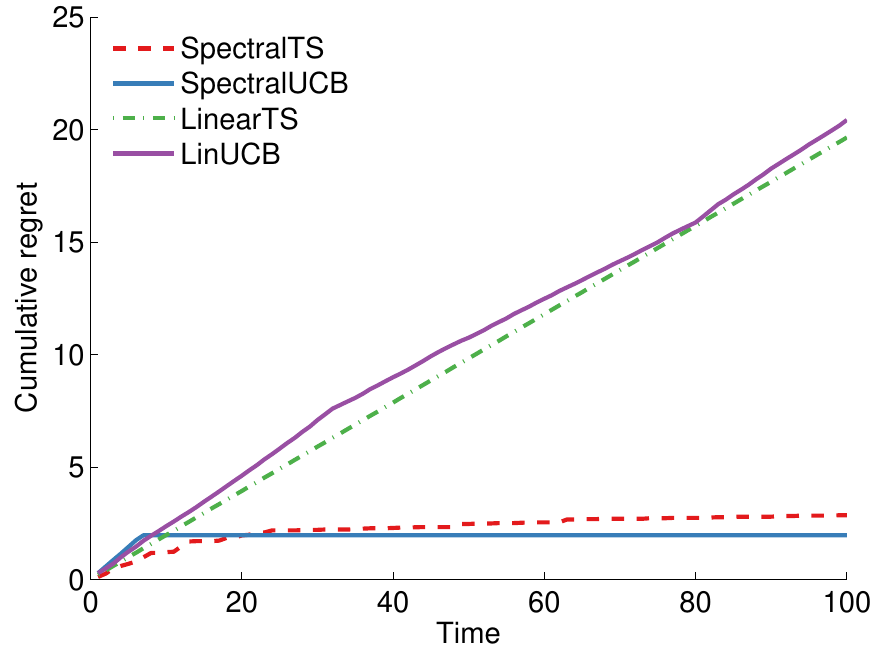}
\caption{Flixster dataset, cumulative regret for one random user}
\label{fig:flixsterdatasetsingle}
\end{subfigure}
\caption{Comparison of spectral and linear bandit algorithms.}
\end{figure}

\subsection{Additional observations for improving the empirical performance}
We give additional indications on how to improve the performance of the algorithms. This can be useful for the deployment of spectral bandits in practice.
\begin{itemize}
\item \textbf{Adjusting the number of edges in the graph.}  Typically, the real-world datasets do not come with a graph structure. Therefore, we  construct a nearest-neighbor graph which connects only the most similar actions. By reducing the number of neighbors, we are increasing the effective dimension (worsening the regret bound) and decreasing smoothness of the function (improving the regret bound). Finding a good trade-off and adjusting the number of the edges can improve the performance of the algorithms significantly.
\item \textbf{Scaling the confidence ellipsoid}, i.e., parameter $c$ in \SpectralUCB and parameter~$v$ in \SpectralTS\textbf{.} Typically, the algorithms are too conservative and the bounds are too loose in order to include the worst case. Therefore, reducing the size of the confidence ellipsoid can sometimes improve the empirical performance of the algorithm at the price that some bounds might not hold anymore. In our experiments, we used the values for which the algorithms had good empirical performance.
\item \textbf{The magnitude of regularization parameter $\lambda$.} By setting $\lambda$ to a large value, all regularized eigenvalues become similar and therefore the algorithms take the graph structure into account less. On the other hand, if the regularization parameter $\lambda$ is small, the algorithms depend on the graph structure more. Therefore, in order to leverage the graphs' structure, we have to find a good compromise while setting $\lambda$. In our experiments, we found that setting $\lambda$ well was important and we tried several values of $\lambda$ to pick the value with the best empirical performance.
\item \textbf{Scaling the graph weights.} By scaling all the weights of the graph by some constant we scale the gap between the eigenvalues and therefore change the value of the effective dimension. Moreover, by scaling the weights we are also changing the smoothness of the reward function. Therefore, by simply scaling the weights we can make the graphs more useful for spectral bandits.
\end{itemize}


\section{Conclusion}
\label{sec:conclusion}

We presented spectral bandits
inspired mostly by the applications in
 recommender systems and targeted advertisement in social networks.
In this setting, we are asked to repeatedly
maximize an unknown graph function,
 assumed to be smooth
on a given similarity graph.
While standard linear bandits
 can be applied but their regret
scales with the ambient
dimension~$D$, either linearly or as a square root,
which can be very large.

Therefore, we introduced
three algorithms, \SpectralUCB, \SpectralTS, and theoretically interesting \SpectralEliminator. For all three 
algorithms, the regret bound only scales with the 
effective dimension $d$
which is typically much smaller than $D$
for real-world graphs. We also performed experiments and showed that 
spectral algorithms are able to leverage the structure of the problem when 
the reward function is smooth on the graph much better than their linear counterparts.

As two side results of independent interest, we provide the regret analysis of \LinUCB
with the upper bound of $\tcO(D\sqrt{T})$ and define \LinearEliminator
for which we prove minimax-optimal regret bound of $\tcO(\sqrt{DT}).$ 
With adaptive confidence bounds and simpler analysis, \LinearEliminator
becomes a state-of-the-art algorithm among the ones with $\tcO(\sqrt{DT})$ regret.

\section*{Acknowledgments}\label{sec:Acknowledgements}
The authors wish to thank anonymous reviewers and Claudio Gentile for helpful suggestions.
We are also grateful to Andr\'as Gy\"orgy for pointing 
out the connection between the effective dimension 
and the capacity formula of parallel Gaussian channels, and 
also for his suggestions on fixing the asymptotics of our lower bound.

The research presented in this paper was supported by French Ministry of
Higher Education and Research, by European Community's
Seventh Framework Programme (FP7/2007-2013) under grant agreement n$^{\rm
o}$270327 (project CompLACS), European CHIST-ERA project DELTA, French Ministry of
Higher Education and Research, 
Inria and Otto-von-Guericke-Universit\"at Magdeburg associated-team North-European project Allocate, Nord-Pas-de-Calais Regional Council, 
CPER Nord-Pas de Calais/FEDER DATA Advanced data science and technologies 2015-2020,
French National Research Agency projects ExTra-Learn (n.ANR-14-CE24-0010-01) and BoB (n.ANR-16-CE23-0003),  Intel Corporation, FMJH Program PGMO with the support to this program from CRITEO.

\bibliography{perfect_biblio,library}

\begin{thebibliography}{63}
\providecommand{\natexlab}[1]{#1}
\providecommand{\url}[1]{\texttt{#1}}
\expandafter\ifx\csname urlstyle\endcsname\relax
  \providecommand{\doi}[1]{doi: #1}\else
  \providecommand{\doi}{doi: \begingroup \urlstyle{rm}\Url}\fi

\bibitem[Abbasi-Yadkori et~al.(2011)Abbasi-Yadkori, P{\'{a}}l, and
  Szepesv{\'{a}}ri]{abbasi2011improved}
Yasin Abbasi-Yadkori, D{\'{a}}vid P{\'{a}}l, and Csaba Szepesv{\'{a}}ri.
\newblock
  \href{https://papers.nips.cc/paper/4417-improved-algorithms-for-linear-stochastic-bandits.pdf}{{Improved
  algorithms for linear stochastic bandits}}.
\newblock In \emph{Neural Information Processing Systems (NeurIPS)}, 2011.

\bibitem[Abeille and Lazaric(2017)]{abeille2017linear}
Marc Abeille and Alessandro Lazaric.
\newblock
  \href{http://proceedings.mlr.press/v54/abeille17a/abeille17a.pdf}{{Linear
  Thompson sampling revisited}}.
\newblock In \emph{International Conference on Artificial Intelligence and
  Statistics (AISTATS)}, 2017.

\bibitem[Abernethy et~al.(2008)Abernethy, Hazan, and
  Rakhlin]{abernethy2008competing}
Jacob~D Abernethy, Elad Hazan, and Alexander Rakhlin.
\newblock
  \href{http://web.eecs.umich.edu/{~}jabernet/123-Abernethy.pdf}{{Competing in
  the dark: An efficient algorithm for bandit linear optimization.}}
\newblock In \emph{Conference on Learning Theory (COLT)}, 2008.

\bibitem[Agrawal and Goyal(2012)]{agrawal2011analysis}
Shipra Agrawal and Navin Goyal.
\newblock
  \href{http://proceedings.mlr.press/v23/agrawal12/agrawal12.pdf}{{Analysis of
  Thompson sampling for the multi-armed bandit problem}}.
\newblock In \emph{Conference on Learning Theory (COLT)}, 2012.

\bibitem[Agrawal and Goyal(2013{\natexlab{a}})]{agrawal2013further}
Shipra Agrawal and Navin Goyal.
\newblock \href{http://proceedings.mlr.press/v31/agrawal13a.pdf}{{Further
  optimal regret bounds for Thompson sampling}}.
\newblock In \emph{International Conference on Artificial Intelligence and
  Statistics (AISTATS)}, 2013{\natexlab{a}}.

\bibitem[Agrawal and Goyal(2013{\natexlab{b}})]{agrawal2013thomson}
Shipra Agrawal and Navin Goyal.
\newblock \href{http://proceedings.mlr.press/v28/agrawal13.pdf}{{Thompson
  sampling for contextual bandits with linear payoffs}}.
\newblock In \emph{International Conference on Machine Learning (ICML)},
  2013{\natexlab{b}}.

\bibitem[Alon et~al.(2013)Alon, Cesa-Bianchi, Gentile, and
  Mansour]{alon2013from}
Noga Alon, Nicol{\`{o}} Cesa-Bianchi, Claudio Gentile, and Yishay Mansour.
\newblock
  \href{https://papers.nips.cc/paper/4908-from-bandits-to-experts-a-tale-of-domination-and-independence.pdf}{{From
  bandits to experts: A tale of domination and independence}}.
\newblock In \emph{Neural Information Processing Systems (NeurIPS)}, 2013.

\bibitem[Alon et~al.(2015)Alon, Cesa-Bianchi, Dekel, and Koren]{alon2015online}
Noga Alon, Nicol{\`{o}} Cesa-Bianchi, Ofer Dekel, and Tomer Koren.
\newblock \href{http://proceedings.mlr.press/v40/Alon15.pdf}{{Online learning
  with feedback graphs: Beyond bandits}}.
\newblock In \emph{Conference on Learning Theory (COLT)}, 2015.

\bibitem[Alon et~al.(2017)Alon, Cesa-Bianchi, Gentile, Mannor, Mansour, and
  Shamir]{Alon2014nonstochastic}
Noga Alon, Nicol{\`{o}} Cesa-Bianchi, Claudio Gentile, Shie Mannor, Yishay
  Mansour, and Ohad Shamir.
\newblock \href{https://arxiv.org/abs/1409.8428}{{Nonstochastic multi-armed
  bandits with graph-structured feedback}}.
\newblock \emph{SIAM Journal on Computing}, 46\penalty0 (6):\penalty0
  1785--1826, 2017.

\bibitem[Auer(2002)]{auer2002using}
Peter Auer.
\newblock \href{http://www.jmlr.org/papers/volume3/auer02a/auer02a.pdf}{{Using
  confidence bounds for exploitation-exploration trade-offs}}.
\newblock \emph{Journal of Machine Learning Research}, 3:\penalty0 397--422,
  2002.

\bibitem[Auer and Ortner(2010)]{auer2010ucb}
Peter Auer and Ronald Ortner.
\newblock \href{http://personal.unileoben.ac.at/rortner/Pubs/UCBRev.pdf}{{\UCB
  revisited: Improved regret bounds for the stochastic multi-armed bandit
  problem}}.
\newblock \emph{Periodica Mathematica Hungarica}, 2010.

\bibitem[Auer et~al.(2002)Auer, Cesa-Bianchi, Freund, and
  Schapire]{auer2002nonstochastic}
Peter Auer, Nicol{\`{o}} Cesa-Bianchi, Yoav Freund, and Robert~E. Schapire.
\newblock \href{https://epubs.siam.org/doi/pdf/10.1137/S0097539701398375}{{The
  nonstochastic multi-armed bandit problem}}.
\newblock \emph{Journal on Computing}, 32\penalty0 (1):\penalty0 48--77, 2002.

\bibitem[Belkin et~al.(2004)Belkin, Matveeva, and
  Niyogi]{belkin2004regularization}
Mikhail Belkin, Irina Matveeva, and Partha Niyogi.
\newblock
  \href{http://people.cs.uchicago.edu/{~}niyogi/papersps/reg{\_}colt.pdf}{{Regularization
  and semi-supervised learning on large graphs}}.
\newblock In \emph{Conference on Learning Theory (COLT)}, 2004.

\bibitem[Belkin et~al.(2006)Belkin, Niyogi, and Sindhwani]{belkin2006manifold}
Mikhail Belkin, Partha Niyogi, and Vikas Sindhwani.
\newblock
  \href{http://www.jmlr.org/papers/volume7/belkin06a/belkin06a.pdf}{{Manifold
  regularization: A geometric framework for learning from labeled and unlabeled
  examples}}.
\newblock \emph{Journal of Machine Learning Research}, 7:\penalty0 2399--2434,
  2006.

\bibitem[Billsus et~al.(2000)Billsus, Pazzani, and Chen]{billsus2000learning}
Daniel Billsus, Michael~J. Pazzani, and James Chen.
\newblock
  \href{https://www.ics.uci.edu/{~}pazzani/Publications/billsuspazzanichen.pdf}{{A
  learning agent for wireless news access}}.
\newblock In \emph{International Conference on Intelligent User Interfaces},
  2000.

\bibitem[Bubeck et~al.(2011)Bubeck, Munos, Stoltz, and
  Szepesv{\'{a}}ri]{bubeck2011x}
S{\'{e}}bastien Bubeck, R{\'{e}}mi Munos, Gilles Stoltz, and Csaba
  Szepesv{\'{a}}ri.
\newblock
  \href{http://www.jmlr.org/papers/volume12/bubeck11a/bubeck11a.pdf}{{$\cX$-armed
  bandits}}.
\newblock \emph{Journal of Machine Learning Research}, 12:\penalty0 1587--1627,
  2011.

\bibitem[Bubeck et~al.(2012)Bubeck, Cesa-Bianchi, and
  Kakade]{bubeck2012towards}
S{\'{e}}bastien Bubeck, Nicol{\`{o}} Cesa-Bianchi, and Sham~M. Kakade.
\newblock
  \href{http://proceedings.mlr.press/v23/bubeck12a/bubeck12a.pdf}{{Towards
  minimax policies for online linear optimization with bandit feedback}}.
\newblock In \emph{Conference on Learning Theory (COLT)}, 2012.

\bibitem[Buccapatnam et~al.(2014)Buccapatnam, Eryilmaz, and
  Shroff]{buccapatnam2014stochastic}
Swapna Buccapatnam, Atilla Eryilmaz, and Ness~B. Shroff.
\newblock
  \href{https://www.orie.cornell.edu/orie/research/groups/multheavytail/upload/mabSigfinal.pdf}{{Stochastic
  bandits with side observations on networks}}.
\newblock In \emph{International Conference on Measurement and Modeling of
  Computer Systems}, 2014.

\bibitem[Caron et~al.(2012)Caron, Kveton, Lelarge, and
  Bhagat]{caron2012leveraging}
St{\'{e}}phane Caron, Branislav Kveton, Marc Lelarge, and Smriti Bhagat.
\newblock \href{https://arxiv.org/pdf/1210.4839.pdf}{{Leveraging side
  observations in stochastic bandits.}}
\newblock In \emph{Uncertainty in Artificial Intelligence (UAI)}, 2012.

\bibitem[Cesa-Bianchi and Lugosi(2006)]{cesa-bianchi2006prediction}
Nicol{\`{o}} Cesa-Bianchi and G{\'{a}}bor Lugosi.
\newblock
  \href{http://www.ii.uni.wroc.pl/{~}lukstafi/pmwiki/uploads/AGT/Prediction{\_}Learning{\_}and{\_}Games.pdf}{\emph{{Prediction,
  learning, and games}}}.
\newblock Cambridge University Press, 2006.

\bibitem[Cesa-Bianchi et~al.(2013)Cesa-Bianchi, Gentile, and
  Zappella]{cesa-bianchi2013gang}
Nicol{\`{o}} Cesa-Bianchi, Claudio Gentile, and Giovanni Zappella.
\newblock \href{https://papers.nips.cc/paper/5006-a-gang-of-bandits.pdf}{{A
  gang of bandits}}.
\newblock In \emph{Neural Information Processing Systems (NeurIPS)}, 2013.

\bibitem[Chapelle and Li(2011)]{chapelle2011empirical}
Olivier Chapelle and Lihong Li.
\newblock \href{https://arxiv.org/pdf/1605.08722.pdf}{{An empirical evaluation
  of Thompson sampling}}.
\newblock In \emph{Neural Information Processing Systems (NeurIPS)}. 2011.

\bibitem[Chau et~al.(2011)Chau, Kittur, Hong, and Faloutsos]{chau2011apolo}
Duen~Horng Chau, Aniket Kittur, Jason~I. Hong, and Christos Faloutsos.
\newblock
  \href{https://www.cc.gatech.edu/{~}dchau/papers/11-chi-apolo.pdf}{{Apolo:
  Making sense of large network data by combining rich user interaction and
  machine learning}}.
\newblock In \emph{Conference on Human Factors in Computing Systems}, 2011.

\bibitem[Chu et~al.(2011)Chu, Li, Reyzin, and Schapire]{chu2011contextual}
Lei Chu, Lihong Li, Lev Reyzin, and Robert~E Schapire.
\newblock \href{http://proceedings.mlr.press/v15/chu11a/chu11a.pdf}{{Contextual
  bandits with linear payoff functions}}.
\newblock In \emph{International Conference on Artificial Intelligence and
  Statistics (AISTATS)}, 2011.

\bibitem[Combes and Prouti{\`{e}}re(2014)]{combes2014unimodal}
Richard Combes and Alexandre Prouti{\`{e}}re.
\newblock \href{http://proceedings.mlr.press/v32/combes14.pdf}{{Unimodal
  bandits: Regret lower bounds and optimal algorithms}}.
\newblock In \emph{International Conference on Machine Learning (ICML)}, 2014.

\bibitem[Dani et~al.(2008)Dani, Hayes, and Kakade]{dani2008stochastic}
Varsha Dani, Thomas~P Hayes, and Sham~M Kakade.
\newblock
  \href{https://repository.upenn.edu/cgi/viewcontent.cgi?article=1501{\&}context=statistics{\_}papers}{{Stochastic
  linear optimization under bandit feedback}}.
\newblock In \emph{Conference on Learning Theory (COLT)}, 2008.

\bibitem[Desautels et~al.(2012)Desautels, Krause, and
  Burdick]{desautels12parallelizing}
Thomas Desautels, Andreas Krause, and Joel Burdick.
\newblock \href{https://icml.cc/2012//papers/602.pdf}{{Parallelizing
  exploration-exploitation tradeoffs in Gaussian process bandit optimization}}.
\newblock In \emph{International Conference on Machine Learning (ICML)}, 2012.

\bibitem[Fang and Tao(2014)]{fang2014networked}
Meng Fang and Dacheng Tao.
\newblock
  \href{http://delivery.acm.org/10.1145/2630000/2623672/p1106-fang.pdf?ip=193.49.212.233{\&}id=2623672{\&}acc=ACTIVE
  SERVICE{\&}key=7EBF6E77E86B478F.5C2A4B72BE2A7DDF.4D4702B0C3E38B35.4D4702B0C3E38B35{\&}{\_}{\_}acm{\_}{\_}=1528847270{\_}23b4902004322713593e1389d42ae48c}{{Networked
  bandits with disjoint linear payoffs}}.
\newblock In \emph{International Conference on Knowledge Discovery and Data
  Mining}, 2014.

\bibitem[Gentile et~al.(2014)Gentile, Li, and Zappella]{gentile2014online}
Claudio Gentile, Shuai Li, and Giovanni Zappella.
\newblock \href{http://proceedings.mlr.press/v32/gentile14.pdf}{{Online
  clustering of bandits}}.
\newblock In \emph{International Conference on Machine Learning (ICML)}, 2014.

\bibitem[Gentile et~al.(2017)Gentile, Li, Kar, Karatzoglou, Zappella, and
  Etrue]{gentile2017context}
Claudio Gentile, Shuai Li, Purushottam Kar, Alexandros Karatzoglou, Giovanni
  Zappella, and Evans Etrue.
\newblock \href{http://proceedings.mlr.press/v70/gentile17a/gentile17a.pdf}{{On
  context-dependent clustering of bandits}}.
\newblock In \emph{International Conference on Machine Learning (ICML)},  2017.

\bibitem[Graclus(2013)]{graclus}
Graclus.
\newblock \href{http://www.cs.utexas.edu/users/dml/Software/graclus.html}{\tt
  http://www.cs.utexas.edu/users/dml/software/graclus.html}.
\newblock \emph{{University of Texas}}, 2013.

\bibitem[Gu and Han(2014)]{gu2014online}
Quanquan Gu and Jiawei Han.
\newblock
  \href{https://ieeexplore.ieee.org/stamp/stamp.jsp?arnumber=7023409}{{Online
  spectral learning on a graph with bandit feedback}}.
\newblock In \emph{International Conference on Data Mining}, 2014.

\bibitem[Hanawal et~al.(2015)Hanawal, Saligrama, Valko, and
  Munos]{hanawal2015cheap}
Manjesh Hanawal, Venkatesh Saligrama, Michal Valko, and R{\'{e}}mi Munos.
\newblock \href{http://proceedings.mlr.press/v37/hanawal15.pdf}{{Cheap
  bandits}}.
\newblock In \emph{International Conference on Machine Learning (ICML)}, 2015.

\bibitem[Jamali and Ester(2010)]{jamali2010matrix}
Mohsen Jamali and Martin Ester.
\newblock
  \href{http://citeseerx.ist.psu.edu/viewdoc/download?doi=10.1.1.459.691{\&}rep=rep1{\&}type=pdf}{{A
  matrix factorization technique with trust propagation for recommendation in
  social networks}}.
\newblock In \emph{Conference on Recommender systems}, 2010.

\bibitem[Jannach et~al.(2010)Jannach, Zanker, Felfernig, and
  Friedrich]{jannach2010recommender}
Dietmar Jannach, Markus Zanker, Alexander Felfernig, and Gerhard Friedrich.
\newblock
  \href{www.scholat.com/teamwork/teamworkdownloadscholar.html?id=542{\&}teamId=316}{\emph{{Recommender
  systems: An introduction}}}.
\newblock Cambridge University Press, 2010.

\bibitem[Kaufmann et~al.(2012)Kaufmann, Korda, and Munos]{kaufmann2012thompson}
Emilie Kaufmann, Nathaniel Korda, and R{\'{e}}mi Munos.
\newblock \href{https://arxiv.org/pdf/1205.4217.pdf}{{Thompson sampling: An
  asymptotically optimal finite-time analysis}}.
\newblock \emph{Algorithmic Learning Theory}, 2012.

\bibitem[Keshavan et~al.(2009)Keshavan, Oh, and Montanari]{keshavan2009matrix}
Raghunandan Keshavan, Sewoong Oh, and Andrea Montanari.
\newblock \href{https://arxiv.org/pdf/0901.3150.pdf}{{Matrix completion from a
  few entries}}.
\newblock In \emph{International Symposium on Information Theory}, 2009.

\bibitem[Kleinberg et~al.(2008)Kleinberg, Slivkins, and
  Upfal]{kleinberg2008multi}
Robert Kleinberg, Aleksandrs Slivkins, and Eli Upfal.
\newblock \href{https://arxiv.org/pdf/0809.4882.pdf}{{Multi-armed bandit
  problems in metric spaces}}.
\newblock In \emph{Symposium on Theory of Computing (STOC)}, 2008.

\bibitem[Koc{\'{a}}k et~al.(2014{\natexlab{a}})Koc{\'{a}}k, Neu, Valko, and
  Munos]{kocak2014efficient}
Tom{\'{a}}{\v{s}} Koc{\'{a}}k, Gergely Neu, Michal Valko, and R{\'{e}}mi Munos.
\newblock
  \href{https://papers.nips.cc/paper/5462-efficient-learning-by-implicit-exploration-in-bandit-problems-with-side-observations.pdf}{{Efficient
  learning by implicit exploration in bandit problems with side observations}}.
\newblock In \emph{Neural Information Processing Systems (NeurIPS)}, 2014{\natexlab{a}}.

\bibitem[Koc{\'{a}}k et~al.(2014{\natexlab{b}})Koc{\'{a}}k, Valko, Munos, and
  Agrawal]{kocak2014spectral}
Tom{\'{a}}{\v{s}} Koc{\'{a}}k, Michal Valko, R{\'{e}}mi Munos, and Shipra
  Agrawal.
\newblock \href{https://hal.inria.fr/hal-00981575v2/document}{{Spectral
  Thompson sampling}}.
\newblock In \emph{AAAI Conference on Artificial Intelligence (AAAI)},
  2014{\natexlab{b}}.

\bibitem[Koc{\'{a}}k et~al.(2016{\natexlab{a}})Koc{\'{a}}k, Neu, and
  Valko]{kocak2016online}
Tom{\'{a}}{\v{s}} Koc{\'{a}}k, Gergely Neu, and Michal Valko.
\newblock \href{http://proceedings.mlr.press/v51/kocak16-supp.pdf}{{Online
  learning with noisy side observations}}.
\newblock In \emph{International Conference on Artificial Intelligence and
  Statistics (AISTATS)}, 2016{\natexlab{a}}.

\bibitem[Koc{\'{a}}k et~al.(2016{\natexlab{b}})Koc{\'{a}}k, Neu, and
  Valko]{kocak2016onlinea}
Tom{\'{a}}{\v{s}} Koc{\'{a}}k, Gergely Neu, and Michal Valko.
\newblock \href{https://hal.inria.fr/hal-01320588/document}{{Online learning
  with Erd\H os-R{\'{e}}nyi side-observation graphs}}.
\newblock In \emph{Uncertainty in Artificial Intelligence (UAI)}, 2016{\natexlab{b}}.

\bibitem[Korda et~al.(2016)Korda, Sz{\"{o}}r{\'{e}}nyi, and
  Li]{korda2016distributed}
Nathan Korda, Bal{\'{a}}zs Sz{\"{o}}r{\'{e}}nyi, and Shuai Li.
\newblock \href{http://proceedings.mlr.press/v48/korda16.pdf}{{Distributed
  clustering of linear bandits in peer to peer networks}}.
\newblock In \emph{International Conference on Machine Learning (ICML)},  2016.

\bibitem[Koutis et~al.(2011)Koutis, Miller, and
  Tolliver]{koutis2011combinatorial}
Ioannis Koutis, Gary~L. Miller, and David Tolliver.
\newblock
  \href{http://www.cs.cmu.edu/{~}./jkoutis/papers/cviu{\_}preprint.pdf}{{Combinatorial
  preconditioners and multilevel solvers for problems in computer vision and
  image processing}}.
\newblock \emph{Computer Vision and Image Understanding}, 115\penalty0
  (12):\penalty0 1638--1646, 2011.

\bibitem[Lam and Herlocker(2012)]{movielens}
Shyong Lam and Jon Herlocker.
\newblock \href{http://www.grouplens.org/node/12}{\tt
  http://www.grouplens.org/node/12}.
\newblock \emph{{MovieLens 1M dataset}}, 2012.

\bibitem[Li et~al.(2010)Li, Chu, Langford, and Schapire]{li2010contextual}
Lihong Li, Wei Chu, John Langford, and Robert~E. Schapire.
\newblock \href{http://rob.schapire.net/papers/www10.pdf}{{A contextual-bandit
  approach to personalized news article recommendation}}.
\newblock \emph{International World Wide Web Conference}, 2010.

\bibitem[Li et~al.(2016)Li, Karatzoglou, and Gentile]{li2015online}
Shuai Li, Alexandros Karatzoglou, and Claudio Gentile.
\newblock \href{https://arxiv.org/pdf/1502.03473.pdf}{{Collaborative filtering
  bandits}}.
\newblock In \emph{Conference on Research and Development in Information
  Retrieval}, 2016.

\bibitem[Ma et~al.(2015)Ma, Huang, and Schneider]{ma2015active}
Yifei Ma, Tzu-Kuo Huang, and Jeff Schneider.
\newblock
  \href{https://pdfs.semanticscholar.org/f72b/71c747d2f487e8c0ade09f4d31e4ad2c0185.pdf}{{Active
  search and bandits on graphs using sigma-optimality}}.
\newblock In \emph{Uncertainty in Artificial Intelligence (UAI)}, 2015.

\bibitem[Mannor and Shamir(2011)]{mannor2011from}
Shie Mannor and Ohad Shamir.
\newblock
  \href{https://papers.nips.cc/paper/4366-from-bandits-to-experts-on-the-value-of-side-observations.pdf}{{From
  bandits to experts: On the value of side-observations}}.
\newblock In \emph{Neural Information Processing Systems (NeurIPS)}, 2011.

\bibitem[May et~al.(2012)May, Korda, Lee, and Leslie]{may2012optimistic}
Benedict~C. May, Nathaniel Korda, Anthony Lee, and David~S. Leslie.
\newblock
  \href{http://www.jmlr.org/papers/volume13/may12a/may12a.pdf}{{Optimistic
  Bayesian sampling in contextual-bandit problems}}.
\newblock \emph{Journal of Machine Learning Research}, 13\penalty0
  (1):\penalty0 2069--2106, 2012.

\bibitem[McPherson et~al.(2001)McPherson, Smith-Lovin, and
  Cook]{mcpherson2001birds}
Miller McPherson, Lynn Smith-Lovin, and James Cook.
\newblock \href{http://aris.ss.uci.edu/{~}lin/52.pdf}{{Birds of a feather:
  Homophily in social networks}}.
\newblock \emph{Annual Review of Sociology}, 27:\penalty0 415--444, 2001.

\bibitem[Narang et~al.(2013)Narang, Gadde, and Ortega]{narang2013signal}
Sunil~K. Narang, Akshay Gadde, and Antonio Ortega.
\newblock
  \href{http://citeseerx.ist.psu.edu/viewdoc/download?doi=10.1.1.650.2525{\&}rep=rep1{\&}type=pdf}{{Signal
  processing techniques for interpolation in graph structured data}}.
\newblock In \emph{International Conference on Acoustics, Speech and Signal
  Processing}, 2013.

\bibitem[Slivkins(2009)]{slivkins2009contextual}
Aleksandrs Slivkins.
\newblock
  \href{http://proceedings.mlr.press/v19/slivkins11a/slivkins11a.pdf}{{Contextual
  bandits with similarity information}}.
\newblock In \emph{Conference on Learning Theory (COLT)}, 2009.

\bibitem[Srinivas et~al.(2010)Srinivas, Krause, Kakade, and
  Seeger]{srinivas2009gaussian}
Niranjan Srinivas, Andreas Krause, Sham~M. Kakade, and Matthias Seeger.
\newblock \href{https://arxiv.org/pdf/0912.3995.pdf}{{Gaussian process
  optimization in the bandit setting: No regret and experimental design}}.
\newblock \emph{International Conference on Machine Learning (ICML)}, 2010.

\bibitem[Thompson(1933)]{thompson1933likelihood}
William~R. Thompson.
\newblock \href{https://www.jstor.org/stable/pdf/2332286.pdf}{{On the
  likelihood that one unknown probability exceeds another in view of the
  evidence of two samples}}.
\newblock \emph{Biometrika}, 25:\penalty0 285--294, 1933.

\bibitem[Valko(2016)]{valko2016bandits}
Michal Valko.
\newblock \href{https://hal.inria.fr/tel-01359757/document}{\emph{{Bandits on
  graphs and structures}}}.
\newblock habilitation, {\'{E}}cole normale sup{\'{e}}rieure de Cachan, 2016.

\bibitem[Valko et~al.(2010)Valko, Kveton, Huang, and Ting]{valko2010online}
Michal Valko, Branislav Kveton, Ling Huang, and Daniel Ting.
\newblock
  \href{http://researchers.lille.inria.fr/{~}valko/hp/serve.php?what=publications/valko2010online.pdf}{{Online
  semi-supervised learning on quantized graphs}}.
\newblock In \emph{Uncertainty in Artificial Intelligence (UAI)}, 2010.

\bibitem[Valko et~al.(2013)Valko, Korda, Munos, Flaounas, and
  Cristianini]{valko2013finite}
Michal Valko, Nathan Korda, R{\'{e}}mi Munos, Ilias Flaounas, and Nelo
  Cristianini.
\newblock \href{https://hal.inria.fr/hal-00826946/document}{{Finite-time
  analysis of kernelised contextual bandits}}.
\newblock In \emph{Uncertainty in Artificial Intelligence (UAI)}, 2013.

\bibitem[Valko et~al.(2014)Valko, Munos, Kveton, and
  Koc{\'{a}}k]{valko2014spectral}
Michal Valko, R{\'{e}}mi Munos, Branislav Kveton, and Tom{\'{a}}{\v{s}}
  Koc{\'{a}}k.
\newblock \href{http://proceedings.mlr.press/v32/valko14.pdf}{{Spectral bandits
  for smooth graph functions}}.
\newblock In \emph{International Conference on Machine Learning (ICML)}, 2014.

\bibitem[von Luxburg(2007)]{luxburg2007tutorial}
Ulrike von Luxburg.
\newblock
  \href{http://www.kyb.mpg.de/fileadmin/user{\_}upload/files/publications/attachments/Luxburg07{\_}tutorial{\_}4488{\%}5B0{\%}5D.pdf}{{A
  tutorial on spectral clustering}}.
\newblock \emph{Statistics and Computing}, 17\penalty0 (4):\penalty0 395--416,
  2007.

\bibitem[Wainwright(2015)]{wainwright2015stat210b}
Martin Wainwright.
\newblock \href{https://www.stat.berkeley.edu/~mjwain/stat210b/}{{STAT 210B}
  {A}dvanced mathematical statistics}.
\newblock Lecture notes, University of California at Berkeley, 2015.

\bibitem[Yu and Mannor(2011)]{yu2011unimodal}
Jia~Yuan Yu and Shie Mannor.
\newblock \href{http://www.icml-2011.org/papers/50{\_}icmlpaper.pdf}{{Unimodal
  bandits}}.
\newblock In \emph{International Conference on Machine Learning (ICML)}, 2011.

\bibitem[Zhu(2008)]{zhu2008semi-supervised}
Xiaojin Zhu.
\newblock
  \href{http://pages.cs.wisc.edu/{~}jerryzhu/pub/ssl{\_}survey.pdf}{{Semi-supervised
  learning literature survey}}.
\newblock Technical Report 1530, University of Wisconsin-Madison, 2008.

\end{thebibliography}
\end{document}